%% file: main.tex
\documentclass[10pt,leqno]{article}

\usepackage[
    letterpaper,
    margin=1in
]{geometry}

\usepackage{setspace}
\usepackage{indentfirst}

\usepackage{amsmath}
\usepackage{amssymb}
\usepackage{amsfonts}
\usepackage{amsthm}
\usepackage{mathrsfs}
\usepackage{nicefrac}

\newtheorem{lemma}{Lemma}
\newtheorem{proposition}{Proposition}

\theoremstyle{definition}
\newtheorem{definition}{Definition}

\theoremstyle{remark}
\newtheorem{remark}{Remark}

\usepackage{graphicx}
\usepackage{booktabs}
\usepackage{tabularx}
\usepackage{multirow}
\usepackage{makecell}
\usepackage{float}
\usepackage{subcaption}
\usepackage{algorithm}
\usepackage{algpseudocode}
\usepackage{listings}

\usepackage{enumitem}
\usepackage{csquotes}
\usepackage{textcomp}

\usepackage{xcolor}
\usepackage{tikz}

\usepackage[title]{appendix}

\usepackage[authoryear,longnamesfirst]{natbib}

\usepackage{authblk}

\usepackage{hyperref}

\hypersetup{
    colorlinks=true,
    linkcolor=black,
    citecolor=black,
    urlcolor=black,
    filecolor=black
}

\usepackage{lineno}

\begin{document}

\title{Joint Spatiotemporal Spectral Neural Operators for Learning PDEs on Irregular Domains}

\author[1]{Abdolmehdi Behroozi}
\author[1]{Chaopeng Shen\thanks{Corresponding author: \texttt{cshen@engr.psu.edu}}}

\affil[1]{Department of Civil and Environmental Engineering,
Penn State University,
University Park, PA 16802, USA}

\date{\today}

\maketitle

\begin{center}
\small
Abdolmehdi Behroozi: ORCID 0000-0002-7663-8727

Chaopeng Shen: ORCID 0000-0002-0685-1901
\end{center}

\input{sections/abstract}

\medskip

\noindent\textbf{Keywords:}
graph spectral neural operators,
spatiotemporal learning,
irregular domains,
partial differential equations,
scientific machine learning

\input{sections/intro}

\input{sections/methods}

\input{sections/Experiments}

\input{sections/discussion}

\bibliographystyle{plainnat}
\bibliography{ref}

\clearpage
\appendix

\input{sections/appendix1}

\section*{CRediT authorship contribution statement}

\noindent
\textbf{Abdolmehdi Behroozi:}
Conceptualization, Methodology, Software, Formal analysis,
Investigation, Data curation, Visualization,
Writing -- original draft.

\medskip

\noindent
\textbf{Chaopeng Shen:}
Conceptualization, Supervision, Project administration,
Funding acquisition, Writing -- review and editing.

\end{document}

%% file: sections/abstract.tex
\begin{abstract}
Learning solution operators for partial differential equations (PDEs) on irregular and geometry-dependent domains remains a central challenge in scientific machine learning. While spectral methods provide strong inductive biases for modeling global interactions, they are typically limited to regular domains, and existing neural approaches often require domain warping, interpolation, or costly geometric embeddings. We introduce the \textbf{Graph Spectral Neural Operator (GSNO)}, a neural operator that combines spatial graph spectral decompositions with temporal Fourier transforms through a unified space--time spectral kernel. This formulation enables globally coherent operator learning on non-Cartesian discretizations without domain warping or autoregressive rollouts. By replacing learned geometric embeddings with a graph Laplacian spectral basis, GSNO provides geometry-aware spectral learning with low parameter complexity. Across steady and unsteady PDE benchmarks on irregular and geometry-dependent domains, GSNO achieves strong accuracy with reduced runtime and parameter counts, while demonstrating robust zero-shot generalization across mesh resolutions and geometry families.
\end{abstract}

%% file: sections/intro.tex
\section{Introduction}

Many problems in science and engineering involve solving complex partial differential equations (PDEs) repeatedly for varying parameters, as seen in fluid dynamics, structural analysis, and geophysical modeling. Capturing multiscale dynamics often requires fine spatial and temporal resolutions, making classical solvers computationally prohibitive; for instance, simulating pollutant transport in irregular terrain or flow in fractured porous media can require thousands of expensive forward solves \citep{palais2009differential}.

\textbf{Conventional solvers vs. data-driven approaches.} 
Traditional numerical techniques such as the finite difference method (FDM), finite volume method (FVM), and finite element method (FEM) \citep{leveque2007finite} offer high precision but suffer from poor scalability as resolution increases, forcing a difficult trade-off between accuracy and computational cost \citep{blechschmidt2021three}. Conversely, data-driven models learn direct mappings from input parameters to solutions \citep{rudy2017data, xiao2024fourier}, providing speedups of several orders of magnitude over traditional solvers once trained \citep{raissi2019physicsinformed, kovachki2023neural}. Recent progress has introduced \textit{neural operators} (NOs), models designed to learn mappings between infinite-dimensional function spaces. Unlike standard neural networks tied to fixed grids, neural operators are mesh-invariant and can generalize across different resolutions. Leading frameworks such as Fourier Neural Operators (FNO) \citep{li2021fourier} have achieved state-of-the-art results via efficient spectral learning, though they are primarily optimized for structured, rectangular domains.

\textbf{Limitations on irregular domains.} 
Most neural operator architectures struggle with the irregular geometries---such as propagating cracks, complex airfoil contours, or patient-specific anatomies---that are fundamental to real-world engineering. In these cases, regular grid approximations fail to capture critical boundary physics and fine-scale geometric details, necessitating models that operate directly on unstructured meshes.

\textbf{Related work for irregular geometries.} 
Several recent methods extend operators to unstructured domains, yet each faces distinct limitations. \textbf{MGKN} \citep{li2020multipole} lacks explicit temporal modeling; \textbf{Geo-FNO} \citep{li2023fourier} relies on diffeomorphic mappings that fail in the presence of holes, sharp boundaries, or complex topologies; and \textbf{CORAL} \citep{serrano2023operator} lacks geometric priors like spectral bases, limiting its ability to capture long-range dependencies. While \textbf{GNOT} \citep{hao2023gnot} and \textbf{Transolver} \citep{wu2024transolver} use attention or token-clustering to handle irregularity, they lack explicit spectral formulations and may obscure fine-scale structures when learned groupings misalign. Notably, \textbf{Sp$^2$GNO} \citep{sarkar2025spatio} handles time autoregressively, which restricts long-range temporal correlations and leads to error accumulation. Finally, \textbf{AMG} \citep{li2025harnessing} introduces significant multi-graph complexity without the physics-grounded guarantees provided by a principled spectral operator. (See Appendix~\ref{appendix:baseline_comparison} for a detailed comparison).

\textbf{Our contributions.} 
We propose the \textbf{Graph Spectral Neural Operator (GSNO)}, a unified architecture for explicit spectral learning in both space and time on arbitrary domains. GSNO utilizes a parameter-free graph Laplacian basis derived from triangulated point clouds, providing a geometry-adaptive representation without added complexity. By combining Laplacian eigenvectors for space with a real-valued FFT for time, GSNO learns a single complex-valued kernel that captures global spatiotemporal dependencies. Our contributions are three-fold: (1) \textbf{Efficiency and Scalability:} By replacing recurrent modules and learned spatial embeddings with a unified spectral formulation, we achieve higher accuracy with fewer parameters and lower training costs. (2) \textbf{Superior Performance:} GSNO reaches state-of-the-art results across diverse PDE benchmarks, maintaining robustness against irregular geometries and multiscale dynamics. (3) \textbf{Zero-Shot Generalization:} The architecture naturally transfers to unseen meshes, resolutions, and discretizations without the need for retraining.
GSNO represents the first operator framework to unify space--time spectral learning on arbitrary geometries, establishing a principled standard for efficient, mesh-invariant PDE modeling.

%% file: sections/methods.tex
\section{Methodology}

We present the Graph Spectral Neural Operator (GSNO), a neural operator architecture for learning solution operators of parametric partial differential equations (PDEs) on irregular domains. The core novelty of GSNO is not only the use of graph spectral representations for irregular spatial domains, but the construction of a joint graph--temporal spectral operator that learns directly over coupled spatial and temporal frequencies. By combining graph Fourier transforms in space with classical Fourier transforms in time, GSNO captures global spatiotemporal interactions while remaining applicable to nonuniform and geometry-dependent discretizations.

\subsection{Neural Operator Framework}

Let \(D \subset \mathbb{R}^d\) be a bounded spatial domain, and let
\(\mathcal{A}=\mathcal{A}(D;\mathbb{R}^{d_a})\) and \(\mathcal{U}=\mathcal{U}(D;\mathbb{R}^{d_u})\)
be the input and output function spaces. Given training pairs
\(\{(a_j,u_j)\}_{j=1}^{N}\) with \(u_j=G^\dagger(a_j)\), we approximate
the solution operator \(G^\dagger:\mathcal{A}\rightarrow\mathcal{U}\) by a
parametric model \(G_\theta:\mathcal{A}\rightarrow\mathcal{U}\),
\(\theta\in\Theta\), trained via empirical loss minimization.
A neural operator lifts the input via \(v_0(x)=P(a(x))\), applies \(L\)
iterative layers,
\begin{equation}
    v_{\ell+1}(x)
    =
    \sigma\!\left(Wv_\ell(x)+(\mathcal{K}_\phi v_\ell)(x)\right),
    \qquad \ell=0,\ldots,L-1,
\end{equation}
and projects to the output via \(u(x)=Q(v_L(x))\), where \(P\) and \(Q\)
are lifting and projection networks, \(W\) is a pointwise linear map,
\(\sigma\) is a nonlinear activation, and \(\mathcal{K}_\phi\) is a
learnable global operator \cite{kovachki2023neural, behroozi2025sensitivity}.

\subsection{Graph Spectral Neural Operator}

GSNO instantiates \(\mathcal{K}_\phi\) as a joint graph--temporal spectral convolution. Unlike approaches that only apply spectral learning over space or rely on local graph message passing, GSNO projects latent fields into a coupled spectral domain defined by graph Laplacian eigenvectors in space and Fourier modes in time. This allows the model to learn nonlocal spatial structure and temporal dynamics within a single global operator. Let the spatial domain be discretized by nodes \(\{x_i\}_{i=1}^{N_s}\). For time-dependent problems, the latent feature tensor is \(v_\ell \in \mathbb{R}^{N_s \times T \times d_v}\), where \(N_s\) is the number of spatial nodes, \(T\) is the number of time steps, and \(d_v\) is the latent channel dimension. For time-independent problems, the same formulation reduces to a purely spatial operator by removing the temporal Fourier transform.

\paragraph{Graph construction and spectral basis.}
We represent the spatial discretization as an undirected geometric graph \(G=(V,E)\), where \(V=\{x_i\}_{i=1}^{N_s}\). The edge set \(E\) is obtained from the available mesh connectivity or constructed using triangulation-based connectivity in 2D and tetrahedralization-based connectivity in 3D. Compared with \(k\)-nearest-neighbor graphs, this construction follows the geometric adjacency of the discretization, provides more isotropic local connectivity,avoids arbitrary neighborhood choices, and better preserves the mesh-induced structure used to define the graph Laplacian, as confirmed by the ablation study in Section~\ref{sec:component-analysis}.

The weighted adjacency matrix is defined as
\begin{equation}
    A_{ij} =
    \begin{cases}
    \exp\left(-\frac{\|x_i-x_j\|^2}{\sigma^2}\right),
    & (x_i,x_j)\in E,\\
    0, & \text{otherwise},
    \end{cases}
\end{equation}
and the normalized graph Laplacian is
\begin{equation}
    \tilde{L}=I-D^{-1/2}AD^{-1/2},
    \qquad
    D_{ii}=\sum_j A_{ij}.
\end{equation}
Its eigendecomposition is
\begin{equation}
    \tilde{L}=\Phi\Lambda\Phi^\top,
\end{equation}
where \(\Phi\in\mathbb{R}^{N_s\times N_s}\) contains the graph eigenvectors and
\(\Lambda=\mathrm{diag}(\lambda_1,\ldots,\lambda_{N_s})\) contains the eigenvalues. We retain the first \(k_s\) eigenvectors, \(\Phi_{k_s}\in\mathbb{R}^{N_s\times k_s}\), as the truncated graph Fourier basis. For a spatial feature field \(f\in\mathbb{R}^{N_s\times d_v}\), the graph Fourier transform projects the field onto the retained spectral basis, and the corresponding reconstruction maps it back to the physical mesh:
\begin{equation}
    \hat{f}=\Phi_{k_s}^{\top}f,
    \qquad
    f\approx \Phi_{k_s}\hat{f}.
\end{equation}

\paragraph{Joint graph--temporal spectral operator.}
The main operator in GSNO acts in a joint spectral domain rather than treating space and time as separate components. Given \(v_\ell\in\mathbb{R}^{N_s\times T\times d_v}\), GSNO first projects the spatial dimension into the graph spectral domain and then applies a temporal Fourier transform:
\begin{equation}
    \hat{v}_{st}
    =
    \mathcal{F}_t\left(\Phi_{k_s}^{\top}v_\ell\right)
    \in
    \mathbb{C}^{k_s\times k_t\times d_v},
\end{equation}
where \(\mathcal{F}_t\) denotes the temporal Fourier transform and
\(k_t=\lfloor T/2\rfloor+1\) is the number of retained temporal modes. A learnable complex-valued kernel \(R_\phi \in \mathbb{C}^{k_s\times k_t\times d_v\times d_v}\) mixes channels at each joint graph--temporal frequency:
\begin{equation}
    \hat{v}'_{st}(p,q,l)
    =
    \sum_{j=1}^{d_v}
    R_\phi(p,q,j,l)\,
    \hat{v}_{st}(p,q,j),
\end{equation}
where \(p\) and \(q\) index graph and temporal modes, respectively. The result is then mapped back to physical space and time:
\begin{equation}
    \mathcal{K}_\phi v_\ell
    =
    \Phi_{k_s}
    \mathcal{F}_t^{-1}
    \left(
    R_\phi \cdot
    \mathcal{F}_t(\Phi_{k_s}^{\top}v_\ell)
    \right).
\end{equation}
Equivalently, this operator applies the sequence:
graph Fourier projection \(\rightarrow\) temporal Fourier transform \(\rightarrow\) learnable joint spectral mixing \(\rightarrow\) inverse temporal Fourier transform \(\rightarrow\) inverse graph Fourier reconstruction. Thus, each GSNO layer combines a local residual mapping with a joint graph--temporal spectral operator acting in the compressed spectral domain:
\begin{equation}
    v_{\ell+1}
    =
    \sigma\!\left(
    \underbrace{W(v_\ell)}_{\text{Local Residual}}
    +
    \underbrace{
    \Phi_{k_s}
    \mathcal{F}_t^{-1}
    \left(
    R_\phi \cdot
    \mathcal{F}_t(\Phi_{k_s}^{\top}v_\ell)
    \right)
    }_{\text{Joint Graph--Temporal Spectral Operator } \mathcal{K}_\phi v_\ell}
    \right),
\end{equation}

where \(W\) is a learnable \(1\times1\) convolution acting locally on the latent channels. The spectral branch captures global coupled graph--temporal interactions, while the pointwise residual branch preserves local adaptivity on irregular meshes.

\begin{figure}
    \centering
    \includegraphics[width=0.7\textwidth]{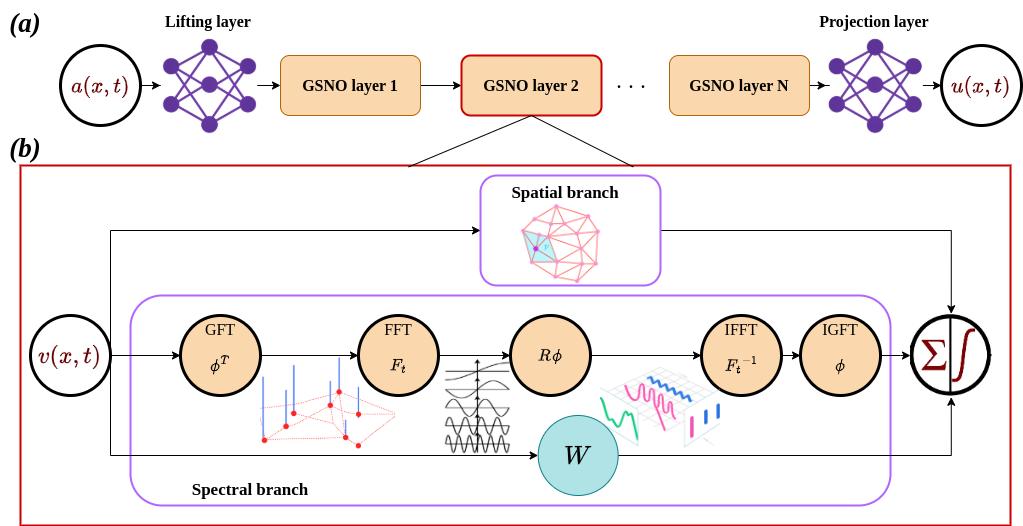}
    \captionsetup{font=footnotesize}
    \caption{
    Overview of the GSNO architecture. \textbf{(a)} The input coefficient \(a(x,t)\) is lifted to a latent space and processed by multiple GSNO layers, each combining graph Fourier transforms and temporal FFTs, before projection to the output \(u(x,t)\). \textbf{(b)} Each GSNO layer maps the latent field to a joint graph--temporal spectral domain, applies a learnable kernel \(R_\phi\), reconstructs the result in physical space and time, and combines it with a local residual branch.
    }
    \label{fig:gsno_architecture}
\end{figure}

\paragraph{Fixed- and varying-geometry settings.}
GSNO is designed to support both fixed-geometry and varying-geometry problems. For a mesh or point set $\mathcal{G}$, we construct the graph Laplacian $L_{\mathcal{G}}$ and compute its truncated eigendecomposition:
\begin{equation}
L_{\mathcal{G}}\Phi_{\mathcal{G},k}
=
\Phi_{\mathcal{G},k}\Lambda_{\mathcal{G},k}.
\end{equation}
where $\Phi_{\mathcal{G},k}\in\mathbb{R}^{N_s\times k}$ denotes the first $k$ graph-spectral modes. In fixed-geometry settings, $L_{\mathcal{G}}$ and $\Phi_{\mathcal{G},k}$ are computed once for the given mesh and reused throughout training and inference. In varying-geometry settings, each geometry $\mathcal{G}^{(i)}$ has its own Laplacian $L_{\mathcal{G}^{(i)}}$ and basis $\Phi_{\mathcal{G}^{(i)},k}$, while the learnable GSNO parameters are shared across all samples. Thus, the eigenbasis construction is a one-time preprocessing cost per geometry, not a trainable component of the model. This allows GSNO to operate across both fixed and varying geometries without changing the architecture.

\paragraph{Computational and spectral interpretation.} The graph spectral decomposition is a \emph{one-time offline preprocessing step}: the first $k_s$ eigenvectors of the normalized graph Laplacian are computed via LOBPCG~\cite{knyazev2001toward} once per geometry and reused across all training epochs and at inference. For fixed-geometry benchmarks a single decomposition suffices for the entire dataset; for varying-geometry benchmarks one decomposition is required per unique geometry. In both settings the amortized preprocessing cost represents a small fraction of per-epoch training time across all experiments, as confirmed by the timing results reported in Appendix~\ref{appendix:laplacian_eigendecomposition_cost}. This stands in contrast to learnable graph methods that recompute edge weights at every forward pass. During training and inference, the spectral operation depends only on the retained modes $k_s$ and $k_t$, avoiding dense operations over the full graph spectrum.

GSNO can be interpreted as a learnable spectral method on irregular geometries, following a projection--transformation--reconstruction structure in which the spectral coefficient transformation is learned from data through \(R_\phi\). Formal properties of the truncated basis, including projector invariance and approximation behavior, are provided in Appendix~\ref{app:theory}.



%% file: sections/Experiments.tex
\section{Numerical Experiments}

We evaluate GSNO on nine PDE systems spanning steady-state and time-dependent regimes: (i) steady-state Darcy flow;
(ii) Euler equations over a 2D airfoil, formulated as a single-step temporal forecast;
(iii) Pipe Flow;
(iv) Hyper-Elasticity;
(v) airflow around the Shape-Net 3D car;
(vi) unsteady Navier--Stokes equations in vorticity form;
(vii) unsteady Shallow Water equations; and
(viii) unsteady Burgers' equations in 2D and 3D. Among these benchmarks, the airfoil, 3D car, pipe, and plasticity cases are treated as varying-geometry problems across samples. Detailed problem setups---including domain geometry, meshes, initial conditions, boundary conditions, and preprocessing---are provided in Appendices~\ref{appendix:darcy}--\ref{appendix:3d_Car}. GSNO is benchmarked against state-of-the-art neural operators, while classical numerical solvers and public datasets are used only to generate ground-truth data. Inputs and outputs are min--max normalized, with spectral settings, model hyperparameters, and sensitivity analyses reported in Appendix~\ref{appendix:hyperparams}. All experiments were conducted on a single NVIDIA V100 GPU (32 GB); reported runtimes reflect per-epoch training time, with the one-time graph Laplacian eigendecomposition treated as offline preprocessing whose cost is shown to be negligible across all benchmarks in Appendix~\ref{appendix:laplacian_eigendecomposition_cost}.

\begin{figure}
  \centering
  \includegraphics[width=0.95\linewidth]{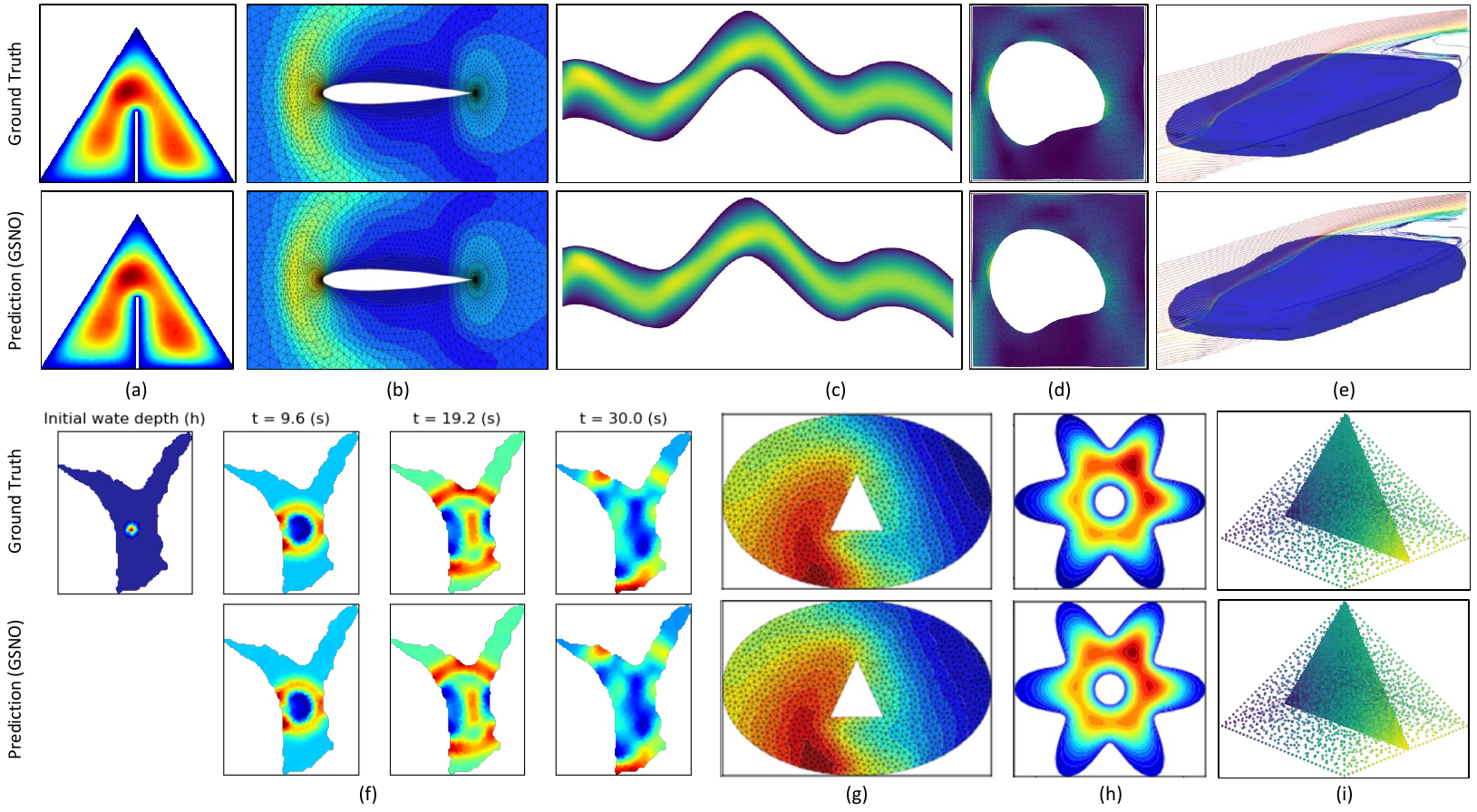}
  \captionsetup{font=footnotesize}

  \caption{Selected results from PDE benchmarks solved using GSNO. Ground truth (top) and GSNO predictions (bottom) are shown for each case. Additional samples and temporal generalization results are provided in Appendix~\ref{appendix:more_results}. 
  \textbf{(a) Darcy flow:} hydraulic head field on a mesh with \(N_s=1184\). 
  \textbf{(b) Airfoil:} pressure field on a mesh with \(N_s=5233\). 
  \textbf{(c) Pipe flow:} horizontal velocity field on a mesh with \(N_s=4225\). 
  \textbf{(d) Hyper-elasticity:} stress field on a mesh with \(N_s=972\). 
  \textbf{(e) Shape-Net 3D Car:} steady-state flow streamlines with \(N_s=32186\). 
  \textbf{(f) Shallow water:} water-height field \(h(x,y,t)\) over 30 seconds in a realistic lake basin; GSNO was trained on \(N_s=1832\) and evaluated zero-shot on \(N_s=3663\). 
  \textbf{(g) Navier--Stokes:} stream function \(\psi(x,y)\) at \(t=10.0\); GSNO was trained on a mesh with \(N_s=972\) and evaluated zero-shot on a finer mesh with \(N_s=1903\). 
  \textbf{(h) 2D unsteady Burgers' equation:} velocity magnitude field \(\|\mathbf{u}(x,y)\|\) at the final time step \(t=1.00\), evaluated on a mesh with \(N_s=1168\). 
  \textbf{(i) 3D unsteady Burgers' equation:} velocity magnitude field \(\|\mathbf{u}(x,y,z)\|\) at the final time step \(t=1.00\), evaluated on a mesh with \(N_s=1867\).}
  \label{fig:comparison}
\end{figure}

\textbf{Benchmarks.}
We compare GSNO against seven neural-operator baselines—\textbf{DeepONet}, \textbf{MGKN}, \textbf{CORAL}, \textbf{Geo-FNO}, \textbf{AMG}, \textbf{Sp$^2$GNO}, \textbf{GNOT}, and \textbf{Transolver}—covering a range of operator-learning approaches on irregular domains. All baselines are retrained on our datasets using identical splits, mesh configurations, and optimization settings. Hyperparameters follow the original implementations, with light tuning to ensure fair comparison, and full training details are provided in Appendix~\ref{appendix:baseline_hyper}.

\subsection{Forward PDE Benchmarks}
\label{sec:pde_forward}

Tables~\ref{tab:steady}, \ref{tab:geo_change}, and~\ref{tab:tin_errors_temporal} provide a comprehensive summary of the relative \(L_2\) errors across all benchmark settings. For clarity, the best result is shown in \textbf{bold} and the second best is \underline{underlined}. \emph{Promotion} denotes the relative error reduction with respect to the second-best model, \(1 - \frac{E_{\text{GSNO}}}{E_{\text{2nd-best}}}\) (reported as a percentage where indicated). Figure~\ref{fig:comparison} presents selected visual comparisons of GSNO predictions against ground truth to be discussed below, with a more comprehensive set of results provided in Appendix~\ref{appendix:more_results}.

\begin{table}
\centering
\footnotesize

\begin{minipage}[t]{0.45\textwidth}
\centering
\captionsetup{font=footnotesize}
\caption{Relative \( L_2 \) error of models on steady-state PDEs\\ at fixed resolution.}
\label{tab:steady}
\resizebox{\textwidth}{!}{%
\begin{tabular}{l c cccc}
\toprule
\multirow{2}{*}{\textbf{Model}}
& \makecell{\textbf{(a) Darcy Flow} \\ (\( N_s = 1184 \))} 
& \multicolumn{4}{c}{\makecell{\textbf{(b) 2D-Airfoil} \\ (\( N_s = 5233 \))}} \\
\cmidrule(lr){2-2}\cmidrule(lr){3-6}
& \multicolumn{1}{c}{Hydraulic head}
& Density & Pressure & Velocity\_x & Velocity\_y \\
\midrule
CORAL      & 0.0664 & 0.0650 & 0.0610 & 0.0365 & 0.0410 \\
Geo-FNO    & 0.0548 & 0.0580 & 0.0550 & 0.0320 & 0.0360 \\
MGKN       & 0.0242 & 0.0500 & 0.0480 & 0.0215 & 0.0260 \\
DeepONet   & 0.0312 & 0.0400 & 0.0370 & 0.0290 & 0.0310 \\
AMG        & 0.0172 & \underline{0.0021} & \underline{0.0020} & \underline{0.0014} & \underline{0.0018} \\
SP$^2$GNO  & 0.0150 & 0.0030 & 0.0028 & 0.0022 & 0.0025 \\
GNOT       & \underline{0.0118} & 0.0054 & 0.0049 & 0.0040 & 0.0040 \\
Transolver & 0.0142 & 0.0036 & 0.0032 & 0.0028 & 0.0035 \\
\textbf{GSNO} 
           & \textbf{0.0083} & \textbf{0.0012} & \textbf{0.0012} & \textbf{0.0009} & \textbf{0.0008} \\
\midrule
\textit{(vs 2nd-best)} 
           & ↓29.66\% & ↓42.86\% & ↓40.00\% & ↓35.71\% & ↓55.56\% \\
\bottomrule
\end{tabular}%
}
\end{minipage}
\hfill
\begin{minipage}[t]{0.48\textwidth}
\centering
\captionsetup{font=footnotesize}

\caption{Relative \(L_2\) error of models on  geometry-varying \\PDE benchmarks.}
\label{tab:geo_change}
\resizebox{\textwidth}{!}{%
\begin{tabular}{l c c cc}
\toprule
\multirow{2}{*}{\textbf{Model}}
& \makecell{\textbf{(a) Pipe FLow} \\ (\( N_s = 4225 \))}
& \makecell{\textbf{(b) Hyper-Elasticity} \\ (\( N_s = 972 \))}
& \multicolumn{2}{c}{\makecell{\textbf{(c) Shape-Net 3D Car} \\ (\( N_s = 32186 \))}} \\
\cmidrule(lr){2-2}\cmidrule(lr){3-3}\cmidrule(lr){4-5}
& \multicolumn{1}{c}{Velocity}
& \multicolumn{1}{c}{Stress field}
& Pressure & Velocity magnitude \\
\midrule
CORAL      & 0.0461 & 0.1091 & 0.1680 & 0.1750 \\
Geo-FNO    & 0.0501 & 0.0937 & 0.1560 & 0.1620 \\
MGKN       & 0.0300 & 0.0795 & 0.1350 & 0.1420 \\
DeepONet   & 0.0371 & 0.1335 & 0.1400 & 0.1480 \\
AMG        & \underline{0.0293} & \underline{0.0770} & \underline{0.0878} & \underline{0.0919} \\
SP$^2$GNO  & 0.0394 & 0.1115 & 0.1005 & 0.1102 \\
GNOT       & 0.0436 & 0.1097 & 0.1199 & 0.1206 \\
Transolver & 0.0310 & 0.0827 & 0.0993 & 0.1208 \\
\textbf{GSNO} 
           & \textbf{0.0142} & \textbf{0.0271} & \textbf{0.0712} & \textbf{0.0759} \\
\midrule
\textit{(vs 2nd-best)} 
           & ↓51.54\% & ↓64.81\% & ↓18.91\% & ↓17.41\% \\
\bottomrule
\end{tabular}%
}
\end{minipage}
\end{table}

\begin{table}
\centering
\captionsetup{font=footnotesize}
\caption{Relative \( L_2 \) error of models on time-dependent PDEs at fixed resolution. \\For these cases, the model takes \( T_{\text{in}} \) input steps to predict the next \( T_{\text{out}} \) steps. }
\label{tab:tin_errors_temporal}

\resizebox{0.98\textwidth}{!}{%
{\setlength{\tabcolsep}{2pt}\renewcommand{\arraystretch}{1.1}%
\begin{tabular}{lcccc cccc cccc cccc}
\toprule
\multirow{3}{*}{\textbf{Model}}
& \multicolumn{4}{c}{\makecell{\textbf{(a) Navier–Stokes Equation} \\ (\( N_s = 1244 \))}} 
& \multicolumn{4}{c}{\makecell{\textbf{(b) Shallow Water Equation} \\ (\( N_s = 1830 \))}}
& \multicolumn{4}{c}{\makecell{\textbf{(c) 2D Burgers’ Equation} \\ (\( N_s = 1168 \))}} 
& \multicolumn{4}{c}{\makecell{\textbf{(d) 3D Burgers’ Equation} \\ (\( N_s = 1867 \))}} \\
& \multicolumn{16}{c}{\textbf{Temporal Config:} \(T_{\text{in}} \rightarrow T_{\text{out}}\)} \\
\cmidrule(lr){2-5} \cmidrule(lr){6-9} \cmidrule(lr){10-13} \cmidrule(lr){14-17}
& 1$\rightarrow$50 & 3$\rightarrow$48 & 5$\rightarrow$46 & 10$\rightarrow$41 
& 1$\rightarrow$50 & 3$\rightarrow$48 & 5$\rightarrow$46 & 10$\rightarrow$41 
& 1$\rightarrow$50 & 3$\rightarrow$48 & 5$\rightarrow$46 & 10$\rightarrow$41 
&  1$\rightarrow$50 &  3$\rightarrow$48 &  5$\rightarrow$46 &  10$\rightarrow$41 
\\
\cmidrule(lr){2-5} \cmidrule(lr){6-9} \cmidrule(lr){10-13} \cmidrule(lr){14-17}
& \multicolumn{4}{c}{Vorticity (\(\omega\))}
& \multicolumn{4}{c}{Water height (\(h\))}
& \multicolumn{4}{c}{Velocity magnitude}
& \multicolumn{4}{c}{Velocity magnitude} \\
\midrule

CORAL      & 0.1654 & 0.1412 & 0.1148 & 0.1070 & 0.1784 & 0.1534 & 0.1238 & 0.1162 & 0.1542 & 0.1308 & 0.1052 & 0.0966 & 0.1975 & 0.1688 & 0.1340 & 0.1255 \\
Geo-FNO    & 0.2056 & 0.1790 & 0.1614 & 0.1512 & 0.2112 & 0.1848 & 0.1650 & 0.1526 & 0.1968 & 0.1718 & 0.1564 & 0.1410 & 0.2380 & 0.2070 & 0.1820 & 0.1695 \\
MGKN       & \underline{0.0964} & \underline{0.0770} & 0.0654 & 0.0606 & \underline{0.1128} & \underline{0.0876} & 0.0724 & 0.0668 & 0.0876 & \underline{0.0686} & 0.0612 & 0.0562 & \underline{0.1245} & \underline{0.0965} & 0.0798 & 0.0735 \\
DeepONet   & 0.1250 & 0.1096 & 0.0958 & 0.0916 & 0.1284 & 0.1128 & 0.0982 & 0.0904 & 0.1146 & 0.1004 & 0.0894 & 0.0846 & 0.1415 & 0.1235 & 0.1080 & 0.0995 \\
AMG        & 0.1060 & 0.0780 & 0.0580 & 0.0540 & 0.1210 & 0.0890 & 0.0692 & 0.0630 & 0.0812 & 0.0694 & 0.0570 & 0.0540 & 0.1325 & 0.0975 & 0.0755 & 0.0685 \\
GNOT       & 0.1876 & 0.1326 & 0.0912 & 0.0844 & 0.2150 & 0.1534 & 0.1093 & 0.0995 & 0.1301 & 0.1232 & 0.0907 & 0.0855 & 0.2425 & 0.1710 & 0.1215 & 0.1105 \\
SP$^2$GNO  & 0.1533 & 0.1081 & 0.0744 & 0.0689 & 0.1752 & 0.1250 & 0.0891 & 0.0810 & 0.1054 & 0.0998 & 0.0736 & 0.0695 & 0.1965 & 0.1400 & 0.1005 & 0.0915 \\
Transolver & 0.1189 & 0.0836 & \underline{0.0575} & \underline{0.0534} & 0.1354 & 0.0965 & \underline{0.0689} & \underline{0.0625} & \underline{0.0806} & 0.0763 & \underline{0.0564} & \underline{0.0534} & 0.1500 & 0.1065 & \underline{0.0765} & \underline{0.0695} \\
\textbf{GSNO} 
           & \textbf{0.0336} & \textbf{0.0237} & \textbf{0.0164} & \textbf{0.0152} 
           & \textbf{0.0375} & \textbf{0.0268} & \textbf{0.0193} & \textbf{0.0174} 
           & \textbf{0.0221} & \textbf{0.0213} & \textbf{0.0156} & \textbf{0.0148} 
           & \textbf{0.0425} & \textbf{0.0305} & \textbf{0.0220} & \textbf{0.0198} \\
\midrule
\textit{Promotion (vs 2nd-best)} 
           & ↓65.15\% & ↓69.22\% & ↓71.48\% & ↓71.54\% 
           & ↓66.76\% & ↓69.41\% & ↓71.99\% & ↓72.16\% 
           & ↓72.58\% & ↓68.95\% & ↓72.34\% & ↓72.28\% 
           & ↓65.86\% & ↓68.39\% & ↓71.24\% & ↓71.51\% \\

\bottomrule
\end{tabular}%
}%
}
\end{table}

\textbf{Darcy Flow.} We evaluate steady-state Darcy flow on a notched triangular domain, shown in Figure~\ref{fig:pde_geometries}a. The task is to learn the solution operator from the diffusion coefficient field to the hydraulic head field on a triangular mesh with \(N_s\) nodes, \(G_\theta: a\in\mathbb{R}^{N_s}\mapsto u\in\mathbb{R}^{N_s}\). As reported in Table~\ref{tab:steady}a, GSNO achieves the lowest relative \(L_2\) error of \textbf{0.0083}, improving over GNOT (0.0118) by 29.6\%. Figures~\ref{fig:comparison}a and~\ref{fig:sample_Darcy} show that GSNO captures the elevated hydraulic head near the notch and respects the no-flow behavior along irregular boundaries. The resolution study in Figure~\ref{fig:gsno-accuracy-gap_darcy} shows consistent gains across mesh resolutions, with up to \(8\times\) lower error than competing methods. Efficiency results in Figure~\ref{fig:runtime_comparison_darcy} and Table~\ref{tab:darcy_runtime_mem_compact} show that GSNO also provides the fastest per-epoch training and inference, with memory usage comparable to the most efficient baselines.

\textbf{2D Airfoil.} We evaluate GSNO on unsteady compressible Euler flow around a 2D airfoil, as shown in Figure~\ref{fig:pde_geometries}b. The irregular domain is discretized using an unstructured mesh with \(N_s=5233\) nodes. The task is one-step prediction: given the flow state at time \(t\), including density, velocity, and pressure fields, the operator predicts the next state at \(t+1\). Formally, \(G_\theta: \mathbb{R}^{N_s} \rightarrow \mathbb{R}^{N_s}\). Table~\ref{tab:steady}b shows that GSNO achieves the lowest relative \(L_2\) errors across all state variables, with \textbf{0.0012} for density, \textbf{0.0012} for pressure, \textbf{0.0009} for \(u_x\), and \textbf{0.0008} for \(u_y\). Compared with AMG, the second-best baseline, these results correspond to error reductions of 42.9\%, 40.0\%, 35.7\%, and 55.6\%, respectively. The qualitative comparisons in Figures~\ref{fig:comparison}b and~\ref{fig:sample_Airfoil} further show that GSNO accurately reconstructs near-field flow structures, including the pressure distribution along the airfoil surface and velocity separation in the wake. These results demonstrate GSNO’s ability to operate on highly irregular meshes while recovering multiple coupled flow variables with high fidelity. In terms of computational efficiency, GSNO attains the lowest per-epoch training time among all evaluated models, requiring only \textbf{9.6s} per epoch. This is faster than Transolver (10.8s) and GNOT (12.1s), and corresponds to approximately a \(7.0\times\) reduction in training time relative to AMG (67.2s), as reported in Table~\ref{tab:airfoil_relL2_mae_rmse}.

\textbf{Pipe Flow.}
We evaluate GSNO on the pipe flow benchmark from \cite{li2023fourier}, where the domain geometry changes across samples through different curved pipe configurations, as shown in Figure~\ref{fig:pde_geometries}h. The objective is to learn an operator that maps each sample-dependent geometry to its corresponding velocity field. In particular, for each pipe geometry \(\Omega_i\), the model learns
\(G_\theta: \mathbb{R}^{N_s \times d_{\text{in}}} \rightarrow \mathbb{R}^{N_s \times 1}\),
where the input consists of geometric features such as node coordinates and shape descriptors, and the output is the velocity field defined on the same geometry. Table~\ref{tab:geo_change}a shows that GSNO achieves the lowest relative \(L_2\) error of \textbf{0.0142}, compared with 0.0293 for AMG, the second-best model. This gives a 51.5\% reduction in error. The visual results in Figure~\ref{fig:comparison}c and the additional samples in Figure~\ref{fig:sample_pipe} show that GSNO closely matches the ground-truth velocity fields across different pipe shapes. The detailed metrics in Table~\ref{tab:pipe_relL2_rmse_mae} further support this result, with GSNO obtaining the lowest RMSE and MAE while also requiring the shortest per-epoch training time. These results indicate that GSNO can accurately and efficiently learn geometry-dependent pipe flow solutions across varying curved domains.

\textbf{Hyper-Elasticity.}
We evaluate GSNO on the hyper-elasticity benchmark from \cite{li2023fourier}, where the internal void geometry varies across samples, as shown in Figure~\ref{fig:pde_geometries}i. This benchmark tests whether the model can learn a geometry-dependent solution operator for solid mechanics problems. The objective is to map each sample-specific solid geometry to its corresponding stress field. In particular, for each geometry \(\Omega_i\), the model learns
\(G_\theta: \mathbb{R}^{N_s \times d_{\text{in}}} \rightarrow \mathbb{R}^{N_s \times 1}\),
where the input contains geometric features such as node coordinates and void-shape information, and the output is the stress field defined on the same solid domain. Table~\ref{tab:geo_change}b shows that GSNO achieves the lowest relative \(L_2\) error of \textbf{0.0271}, compared with 0.0770 for AMG, the second-best model. This corresponds to a 64.8\% reduction in error. The qualitative results in Figure~\ref{fig:comparison}d and the additional examples in Figure~\ref{fig:sample_hyperelasticity} show that GSNO accurately recovers stress concentrations around different void patterns. The detailed results in Table~\ref{tab:elasticity_relL2_rmse_mae} further show that GSNO obtains the lowest RMSE and MAE while also achieving the fastest per-epoch training time. These results indicate that GSNO can efficiently learn geometry-to-stress mappings across varying hyper-elastic solid configurations.

\textbf{Shape-Net 3D Car.}
We evaluate GSNO on the ShapeNet 3D car benchmark, a varying-geometry problem where each sample represents a distinct car shape with approximately 32k unstructured points, as shown in Figure~\ref{fig:pde_geometries}c. The task is to learn a geometry-to-flow operator from geometry-dependent inputs, including coordinates, signed distance values, and surface normals, to time-averaged pressure and velocity fields. Figures~\ref{fig:comparison}e and~\ref{fig:sample_car} show that GSNO predictions closely follow the reference streamlines and capture aerodynamic structures across different vehicle shapes. As reported in Table~\ref{tab:geo_change}c, GSNO achieves the lowest relative \(L_2\) errors, with \textbf{0.0712} for pressure and \textbf{0.0759} for velocity magnitude, improving over AMG by 18.9\% and 17.4\%, respectively. Table~\ref{tab:car3d_metrics} further shows that GSNO has the lowest per-epoch training time, requiring only \textbf{24s}, faster than Transolver (27s), GNOT (36s), and AMG (181s). These results show that GSNO remains accurate and efficient under substantial geometry variation.

\textbf{Navier--Stokes Equations.} For the incompressible 2D Navier--Stokes equations, we consider an elliptical domain with a triangular cutout, as shown in Figure~\ref{fig:pde_geometries}e. The task is to predict the vorticity field \(\omega \in \mathbb{R}^{N_s \times T \times 1}\). The operator maps the first \(T_{\text{in}}\) input snapshots to the following \(T_{\text{out}}\) forecast snapshots, namely \(G_\theta: \mathbb{R}^{N_s \times T_{\text{in}} \times 1} \rightarrow \mathbb{R}^{N_s \times T_{\text{out}} \times 1}\). Table~\ref{tab:tin_errors_temporal}a shows that GSNO achieves the most accurate predictions across all temporal splits. In the long-horizon setting \(T_{\text{in}}=10 \rightarrow T_{\text{out}}=41\), GSNO obtains an error of \textbf{0.0152}, while the next best method, Transolver, records 0.0534, corresponding to a 71\% error reduction. The qualitative comparisons in Figures~\ref{fig:comparison}g and~\ref{fig:sample_NSE} further show that GSNO reconstructs the roll-up of vortical structures and the onset of flow separation with high fidelity. Its resolution generalization is also strong: Figure~\ref{fig:nse_joint_summary}a shows errors up to an order of magnitude lower than competing baselines on finer meshes. In terms of efficiency, Figure~\ref{fig:nse_joint_summary}b shows that GSNO completes each training epoch up to \(5\times\) faster than competing operators. Table~\ref{tab:nse_runtime_mem_compact} further summarizes the computational profile of this setup, where GSNO achieves the fastest inference while maintaining a memory footprint comparable to the most efficient baselines.

\textbf{Shallow Water Equations.}
We test GSNO on the 2D Shallow Water Equations in conservative form, a standard flood inundation model, using an irregular Lake Union mesh with \(N_s=3663\) nodes over 30 seconds, as shown in Figure~\ref{fig:pde_geometries}f. The task is to learn the temporal solution operator from an input water-height window to future states, \(G_\theta: h\in\mathbb{R}^{N_s \times T_{\text{in}}}\mapsto h\in\mathbb{R}^{N_s \times T_{\text{out}}}\). Table~\ref{tab:tin_errors_temporal}b shows that GSNO is the most accurate across temporal settings; for \(T_{\text{in}}=10 \rightarrow T_{\text{out}}=41\), it achieves \textbf{0.0174} relative error, a 72\% reduction over Transolver (0.0625). Figure~\ref{fig:comparison}f shows that GSNO captures wavefront propagation and shoreline reflections, while Figure~\ref{fig:gsno-accuracy-gap_swe} shows up to \(9\times\) lower error under mesh refinement. Figure~\ref{fig:runtime_comparison_swe} further shows up to \(3.5\times\) per-epoch training speedup.

\textbf{2D and 3D Burgers' Equations.}
We evaluate GSNO on two- and three-dimensional unsteady Burgers' equations, shown in Figures~\ref{fig:pde_geometries}d and~\ref{fig:pde_geometries}g. Both simulations contain 51 time steps, and the task is to map the first \(T_{\text{in}}\) snapshots to the following \(T_{\text{out}}\) velocity states,
\(G_\theta: \mathbb{R}^{N_s \times T_{\text{in}} \times d} \rightarrow \mathbb{R}^{N_s \times T_{\text{out}} \times d}\),
where \(d\) is the velocity-state dimension. For 2D Burgers, the domain is flower-shaped with a central hole and \(N_s=1168\) nodes. The model predicts \([u,v] \in \mathbb{R}^{N_s \times T \times 2}\), so \(d=2\). GSNO achieves the lowest errors across all temporal splits in Table~\ref{tab:tin_errors_temporal}c. In the \(1\rightarrow50\) setting, it obtains 0.0221, compared with 0.0806 for Transolver, giving more than a 72\% error reduction. It further reaches \textbf{0.0156} for \(5\rightarrow46\) and remains below 0.022 across all splits. Figures~\ref{fig:comparison}h and~\ref{fig:sample_burgers_uv} show that GSNO captures nonlinear transport and dissipation within the hollowed-out domain, while Figures~\ref{fig:gsno-accuracy-gap_burger} and~\ref{fig:runtime_comparison_burger} show up to \(10\times\) lower error and \(3.5\times\) runtime speedup. For 3D Burgers, the equation is solved on a fixed square-pyramid domain \(\Omega\), represented by an irregular tetrahedral mesh with \(N_s=1867\) nodes. GSNO again obtains the lowest errors across all settings in Table~\ref{tab:tin_errors_temporal}d. In the \(1\rightarrow50\) split, it reaches \textbf{0.0425}, compared with 0.1245 for MGKN, corresponding to a 65.9\% reduction. With \(10\rightarrow41\), the error further decreases to \textbf{0.0198}. Figure~\ref{fig:comparison}i shows that GSNO captures the main 3D velocity structures, while Table~\ref{tab:swe_metrics} and Figure~\ref{fig:runtime_comparison_3d_burgers} confirm its lowest RMSE, MAE, and per-epoch runtime across all temporal configurations.

\begin{figure}
    \centering
    \begin{subfigure}[t]{\textwidth}
        \centering
        \captionsetup{font=footnotesize, justification=centering}
        \includegraphics[trim={0 7 0 0}, clip, width=0.9\textwidth]{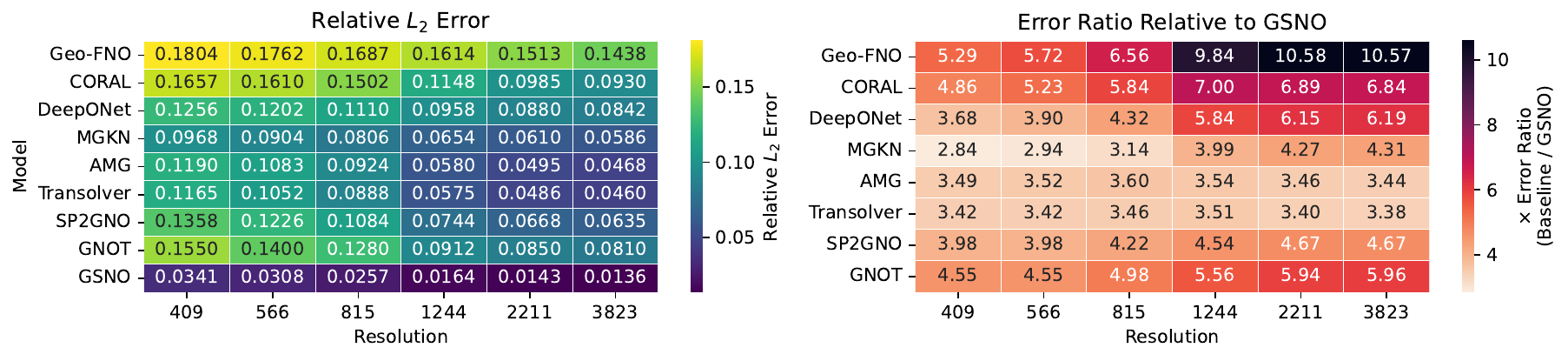}
        \caption{Accuracy vs. spatial resolution. All models are trained and evaluated on the same temporal setup \( T_{\text{in}} = 5 \rightarrow T_{\text{out}} = 46 \) and tested across increasing mesh resolutions. Left: relative \( L_2 \) error. Right: model-wise error ratio with respect to GSNO.}
    \end{subfigure}
    
    \vspace{1em}
    
    \begin{subfigure}[t]{\textwidth}
        \centering
        \captionsetup{font=footnotesize, justification=centering}
        \includegraphics[trim={0 7 0 0}, clip, width=0.9\textwidth]{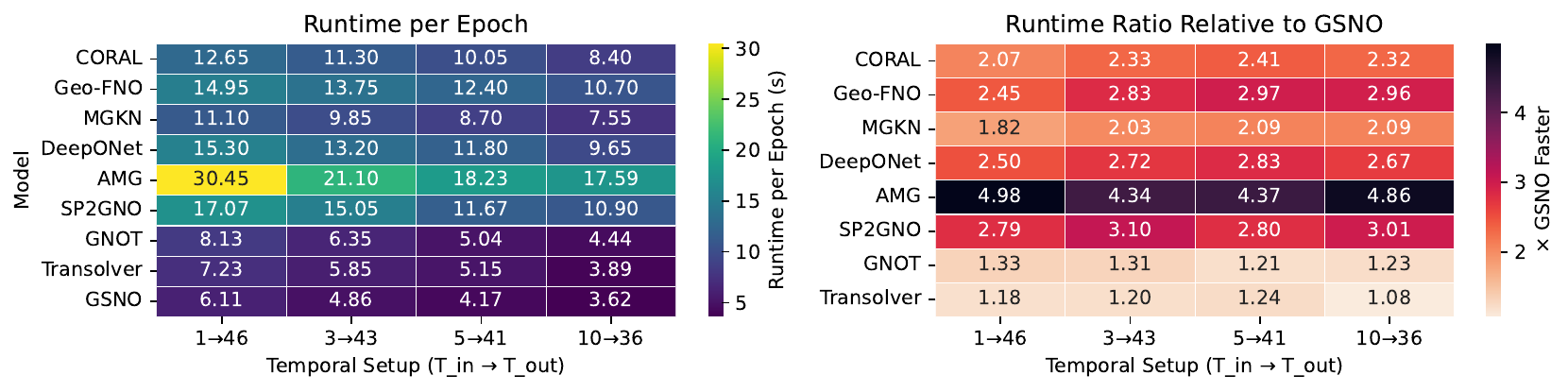}
        \caption{Training runtime per epoch. All models are trained and evaluated on the same mesh resolution (\( N_s = 1244 \)) and tested across different temporal configurations. Left: runtime in seconds. Right: slowdown factor relative to GSNO.}
    \end{subfigure}

    \captionsetup{font=footnotesize}
    \caption{Performance comparison of neural operator models on the 2D Navier–Stokes equation using GSNO and baselines.}
    \label{fig:nse_joint_summary}
\end{figure}

\subsection{Zero-Shot Super-Resolution}
\label{sec:zero_shot}

GSNO enables zero-shot super-resolution by applying the learned operator in a graph-spectral basis rather than on a fixed grid. Given latent features \(v_t \in \mathbb{R}^{N_s \times T \times d_v}\), GSNO projects them onto the Laplacian eigenbasis \(\Phi_{k_s}\) and applies the learned graph--temporal spectral operator:
\begin{equation}
(\mathcal{K}_\phi v_t)(x)
=
\Phi_{k_s}
\mathcal{F}_t^{-1}
\left(
R_\phi \cdot \mathcal{F}_t(\Phi_{k_s}^{\top} v_t)
\right),
\end{equation}
where \(R_\phi\) is the learned spectral kernel and \(\mathcal{F}_t\) is the temporal Fourier transform. At test time, the Laplacian basis is recomputed on the finer mesh, while the trained \(R_\phi\) is reused without retraining. Thus, GSNO transfers across resolutions through Laplacian recomputation instead of coordinate interpolation or grid mapping. Appendix~\ref{app:theory} discusses the related permutation-equivariance and refinement properties. Figures~\ref{fig:comparison}(e) and~\ref{fig:sample_NSE} show zero-shot transfer from \(N_s=1832\) to \(3663\) for shallow water equations and from \(N_s=972\) to \(1903\) for Navier--Stokes.

\subsection{Bayesian Inverse Problem for GSNO}
\label{sec:darcy_inverse}

We consider the Bayesian inverse problem of recovering the unknown Darcy coefficient field
\( a(x,y) \in \mathbb{R}^{1184 \times 1} \) from a single observed solution
\( u_{\mathrm{obs}} \in \mathbb{R}^{1184 \times 1} \). In this setting, the forward map
\( a \mapsto u \) is approximated using the trained GSNO model, which provides fast,
GPU-accelerated surrogate evaluations during posterior sampling. To characterize the uncertainty in the recovered coefficient field, we employ a
function-space Markov Chain Monte Carlo (MCMC) method
\cite{geyer1992practical,geyer1995annealing}, specifically the
Metropolis--Hastings algorithm \cite{chib2001marginal}, to sample from the posterior
distribution over admissible coefficient fields. We assume a zero-mean Gaussian prior,

\[
a \sim \mathcal{N}\!\left(0,\sigma_{\mathrm{prior}}^{2}I\right),
\]

where \( \sigma_{\mathrm{prior}}^{2} \) controls the prior variance. The data-misfit term
is defined using the squared error between the GSNO-predicted solution and the observed
solution. Because the observations are assumed to be noise-free, the resulting
unnormalized log-posterior is written as

\[
\log p\!\left(a \mid u_{\mathrm{obs}}\right)
\propto
-\frac{1}{2}
\left\|
\mathsf{GSNO}(a)-u_{\mathrm{obs}}
\right\|^{2}
-\frac{1}{2\sigma_{\mathrm{prior}}^{2}}
\left\|a\right\|^{2}.
\]

The first term penalizes disagreement between the surrogate prediction and the observed
solution, while the second term regularizes the inferred coefficient field according to
the Gaussian prior. At each Metropolis--Hastings iteration, a candidate coefficient field
is evaluated through a single forward pass of the trained GSNO model, and the candidate
is accepted or rejected according to the corresponding posterior probability. We perform \(5{,}000\) MCMC iterations and discard the first \(500\) iterations as
burn-in. This results in \(4{,}500\) retained posterior samples and requires a total of
\(5{,}000\) forward evaluations of GSNO. Owing to the mesh-invariant spectral formulation
of GSNO and its GPU-accelerated execution, the complete sampling procedure can be
performed within minutes. Figure~\ref{fig:gsno_bayesian_inverse_darcy} compares the true Darcy coefficient field
with the posterior mean estimated from the retained samples. The close agreement between
the reconstructed and true fields demonstrates that GSNO can serve as an efficient and
accurate surrogate forward model for Bayesian inversion under noise-free observation
conditions.

\begin{figure}[h]
    \centering
    \includegraphics[width=0.8\textwidth]
    {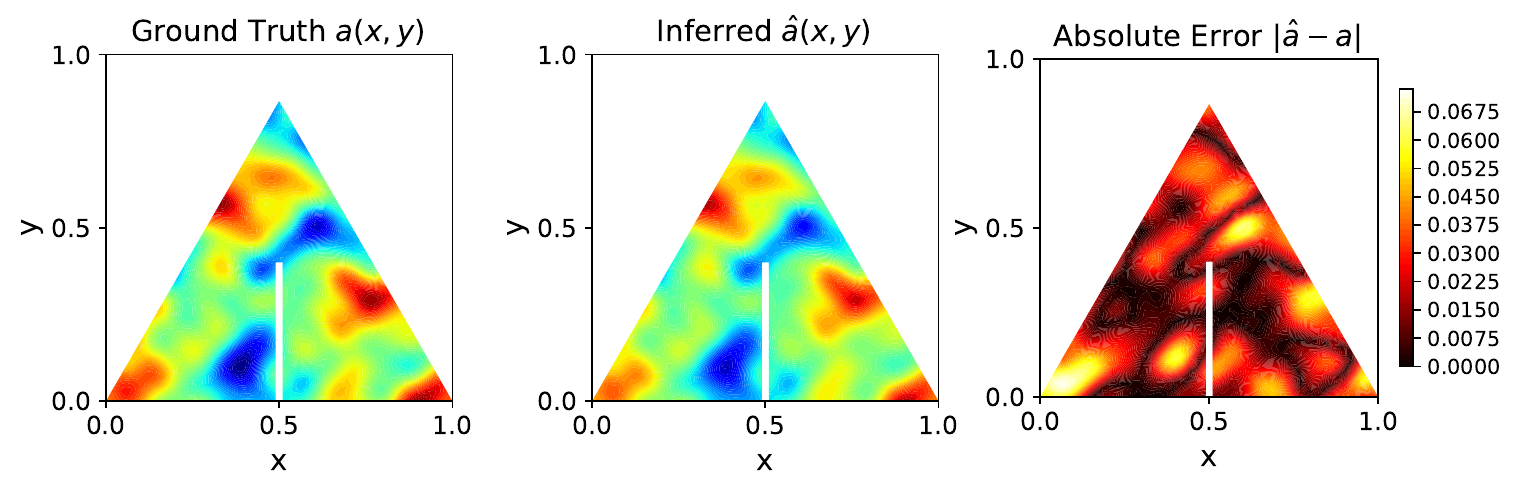}
    \captionsetup{font=footnotesize}
    \caption{
    Comparison of the true Darcy coefficient field and the posterior mean inferred from
    the noise-free observation
    \(u_{\mathrm{obs}} \in \mathbb{R}^{1184 \times 1}\).
    GSNO is used as the surrogate forward model for \(5{,}000\)
    Metropolis--Hastings iterations, of which the first \(500\) are discarded as burn-in,
    leaving \(4{,}500\) retained posterior samples. The reconstructed posterior-mean
    field closely agrees with the true coefficient field.
    }
    \label{fig:gsno_bayesian_inverse_darcy}
\end{figure}

\subsection{Component-wise Ablation Analysis}
\label{sec:component-analysis}

We evaluate the contribution of the main GSNO components through controlled ablation experiments on the steady-state Darcy flow problem and the time-dependent two-dimensional Burgers' equation. Five model variants are considered. First, we remove the learnable spectral kernel by setting \(R_\phi=I\). This variant retains the mesh-induced graph Laplacian eigenbasis but does not learn interactions or mixing among the retained spatial modes. Second, we disable the local residual path by setting \(W=0\), removing the pointwise correction branch and leaving the prediction primarily governed by the global spectral operator. Third, for the time-dependent Burgers' problem, we replace the temporal fast Fourier transform (FFT) with a multilayer perceptron (MLP) to examine whether explicit global temporal-frequency modeling is beneficial for long-horizon forecasting. Fourth, we replace the mesh-induced spatial graph with a K-nearest-neighbor (KNN) graph constructed from the node coordinates while retaining the remaining GSNO architecture and spectral-kernel structure. This variant tests whether approximate distance-based connectivity can substitute for the original mesh connectivity. Finally, we replace the graph Laplacian eigenbasis \(\Phi_{k_s}\) with a random orthonormal basis of the same dimension, thereby removing the geometry-aware spatial spectral structure while preserving the dimensionality of the projected representation. The experiments use the same training configuration, mesh resolution, number of retained spatial modes, and optimization settings as the corresponding main experiments. The Darcy problem is formulated as a static operator mapping the coefficient field \(a(x)\) to the solution field \(u(x)\). For Burgers', we use the long-horizon forecasting configuration
\(T_{\mathrm{in}}=5 \rightarrow T_{\mathrm{out}}=46\). For all graph-based variants, the graph structure and spatial basis are computed before training. Consequently, the runtimes reported in Table~\ref{tab:gsno_ablation} represent the training cost per epoch and exclude the one-time costs of graph construction and eigendecomposition.

\begin{table}[h]
\centering
\captionsetup{font=footnotesize}
\caption{
Component-wise ablation results for GSNO on steady-state Darcy flow and the time-dependent Burgers' equation with
\(T_{\mathrm{in}}=5 \rightarrow T_{\mathrm{out}}=46\).
The reported runtimes represent training time per epoch and exclude the one-time costs of graph construction and eigendecomposition.
}
\label{tab:gsno_ablation}
\scriptsize
\resizebox{0.85\textwidth}{!}{%
\begin{tabular}{l|cc|cc}
\toprule
\textbf{Model Variant}
& \multicolumn{2}{c|}{\makecell{\textbf{(a) Darcy Flow} \\ (\(N_s=1184\))}}
& \multicolumn{2}{c}{\makecell{\textbf{(b) Burgers' Equation} \\ (\(N_s=1168\))}} \\
\cmidrule(lr){2-3}
\cmidrule(lr){4-5}
& Relative \(L_2\) Error
& Runtime (s/epoch)
& Relative \(L_2\) Error
& Runtime (s/epoch) \\
\midrule
\textbf{Full GSNO (ours)}
& \textbf{0.0083}
& 2.24
& \textbf{0.0156}
& 3.12 \\

No Spectral Kernel (\(R_\phi=I\))
& 0.0138
& 2.02
& 0.0269
& 2.86 \\

No Local Path (\(W=0\))
& 0.0119
& 2.08
& 0.0237
& 2.91 \\

No Temporal FFT (MLP instead)
& --
& --
& 0.0348
& 3.74 \\

KNN Spatial Graph
& 0.0198
& 2.30
& 0.0445
& 3.24 \\

Random Spatial Basis
& 0.0465
& 2.26
& 0.0917
& 3.18 \\
\bottomrule
\end{tabular}%
}
\end{table}

As shown in Table~\ref{tab:gsno_ablation}, the full GSNO achieves the lowest relative \(L_2\) error on both benchmarks, demonstrating that the mesh-aware spatial basis, learnable spectral kernel, local residual path, and temporal Fourier representation each contribute to the final performance. Removing the learnable spectral kernel by setting \(R_\phi=I\) increases the relative \(L_2\) error from \(0.0083\) to \(0.0138\) for Darcy flow and from \(0.0156\) to \(0.0269\) for Burgers'. Because this variant retains the original mesh-induced Laplacian eigenbasis, the results show that spectral projection alone provides a useful geometry-aware representation. However, the increased error indicates that learning to mix and transform the retained spatial modes is necessary to capture interactions among frequencies and improve the global expressivity of the operator. Disabling the local residual path by setting \(W=0\) also degrades accuracy, increasing the error to \(0.0119\) for Darcy flow and \(0.0237\) for Burgers'. The global spectral representation efficiently captures long-range spatial interactions, but it is based on a truncated set of modes and may not fully represent localized or fine-scale features. The pointwise residual branch therefore complements the global operator by correcting and refining local features that are not completely captured in the truncated spectral representation. For the Burgers' equation, replacing the temporal FFT with an MLP increases the relative \(L_2\) error from \(0.0156\) to \(0.0348\). It also increases the training runtime from \(3.12\) to \(3.74\) seconds per epoch. Thus, the MLP replacement is both less accurate and more computationally expensive than the temporal Fourier formulation. This result supports the use of explicit temporal-frequency modeling for efficiently capturing global and long-range temporal dependencies, particularly in the challenging
\(T_{\mathrm{in}}=5 \rightarrow T_{\mathrm{out}}=46\) forecasting setting. Replacing the mesh-induced graph with a KNN graph further increases the error to \(0.0198\) for Darcy flow and \(0.0445\) for Burgers'. Although the KNN graph preserves approximate geometric proximity between nodes, it does not exactly retain the physical mesh connectivity, boundary structure, or element-induced neighborhood relationships encoded by the original computational mesh. The KNN variant performs worse than the full GSNO and also worse than the variant that retains the mesh Laplacian basis while removing the spectral kernel. This comparison indicates that constructing the spatial spectral basis from the correct mesh topology is more important than applying a learnable spectral kernel to an approximate distance-based graph. The largest degradation occurs when the mesh Laplacian eigenbasis is replaced by a random orthonormal basis. The relative \(L_2\) error increases to \(0.0465\) for Darcy flow and \(0.0917\) for Burgers', while the per-epoch runtimes remain comparable to those of the full model. Because the random basis has the same dimension as the Laplacian basis, this result confirms that GSNO's performance does not arise merely from projecting the solution onto an arbitrary low-dimensional subspace. Instead, the geometry-aware Laplacian modes provide essential information about the irregular domain, mesh connectivity, and spatial structure of the underlying PDE. Overall, the ablation results show that GSNO's accuracy arises from the combined effect of its mesh-induced graph spectral representation, learnable spectral-mode coupling, local residual correction, and global temporal Fourier modeling. The relatively similar per-epoch runtimes of most spatial variants further indicate that the accuracy improvements are primarily attributable to the quality of the learned representation rather than a substantial increase in computational cost. These findings support GSNO's central design of coupling a geometry-aware graph spectral basis in space with Fourier modeling in time for learning PDE solution operators on irregular domains.

%% file: sections/discussion.tex
\section{Discussion, Limitations, and Conclusion}
\label{sec:conclusion}

GSNO's performance across diverse PDE benchmarks confirms that compact space–time spectral representations—avoiding mesh-specific encodings and heavy parameterization—yield high accuracy and scalability across varying resolutions and geometries without model reconfiguration. We interpret these results through the lens of spectral learning, resolution-adaptive generalization, geometry-aware transfer, and compression-driven efficiency.

\textbf{Spectral Learning Enables Accurate Operator Approximation.}
GSNO projects spatial inputs onto a truncated graph Laplacian basis \(\Phi_{k_s}\) and applies a temporal Fourier transform, enabling the learnable kernel \(R_\phi\) to capture long-range spatial and temporal dependencies while preserving geometric and physical structure—unlike coordinate-based MLPs or message-passing architectures. As shown in Figures~\ref{fig:comparison},~\ref{fig:sample_Darcy},~\ref{fig:sample_burgers_uv}, and~\ref{fig:sample_NSE}, this yields high accuracy and smooth, physically consistent predictions on complex irregular domains for both steady and unsteady PDEs.

\textbf{Generalization Across Mesh Resolutions.}
GSNO supports resolution-independent inference by learning operators in a graph spectral basis induced by the normalized graph Laplacian. At inference time, the graph Laplacian and its spectral basis can be recomputed for a new discretization, while the learned spectral kernel \(R_\phi\) remains fixed. This enables zero-shot transfer across mesh resolutions without retraining. As shown in Section~\ref{sec:zero_shot} and Figures~\ref{fig:gsno-accuracy-gap_darcy},~\ref{fig:gsno-accuracy-gap_burger},~\ref{fig:gsno-accuracy-gap_nse}, and~\ref{fig:gsno-accuracy-gap_swe}, GSNO maintains strong accuracy across multiple mesh resolutions for both linear and nonlinear PDEs.

\textbf{Geometry-Aware Generalization Across Varying Domains.}
In addition to resolution transfer, GSNO can handle geometry-varying datasets by constructing the graph and spectral basis from each sample-specific domain. This allows the model to learn geometry-conditioned solution operators, where changes in the domain shape are reflected directly in the graph representation. As shown in Table~\ref{tab:geo_change}, GSNO achieves the lowest errors across geometry-varying PDE benchmarks. The results demonstrate this capability across both fluid and solid mechanics settings, where GSNO accurately maps varying geometries to their corresponding solution fields.

\textbf{Efficient Training via Low-Rank Spectral Compression.} GSNO's efficiency stems from its low-rank spectral design. A one-time offline step computes only the first $k_s$ low-frequency graph Laplacian eigenvectors, reused throughout training and inference, so per-epoch cost depends on $k_s$ and $k_t$ rather than full mesh resolution $N_s$. Operating in this compact subspace, GSNO filters a reduced set of spatial and temporal modes, avoiding high-resolution convolutions and deep message-passing stacks---substantially cutting memory and computation while preserving the dominant physical modes needed for accuracy. As shown in Figures~\ref{fig:runtime_comparison_darcy},~\ref{fig:runtime_comparison_burger}, and~\ref{fig:runtime_comparison_nse}, this yields up to $2$--$3\times$ faster per-epoch runtimes than strong baselines while matching or exceeding their accuracy, enabling scalable deployment on large or resource-constrained problems. The one-time eigendecomposition cost itself represents a small fraction of total training time across all benchmarks, as detailed in Appendix~\ref{appendix:laplacian_eigendecomposition_cost}.

\textbf{Limitations and Future Scope.}
While GSNO generalizes effectively across mesh resolutions and geometry-varying datasets, the current experiments focus on geometry families with shared structural characteristics, such as Pipe, Car, and Elasticity, where sample-specific domains vary but remain within a related class. Extending this capability to substantially different geometry families or topologies is a broader open challenge in neural operator learning and may require additional adaptation or problem-specific alignment. In addition, this study focuses on benchmark PDE systems under supervised settings; evaluating GSNO on large-scale real-world multiphysics problems and sparse observational regimes remains an important direction for future work.

%% file: sections/appendix1.tex
\section{Table of Notations}
\renewcommand{\thefigure}{A.\arabic{figure}}  
\renewcommand{\thetable}{A.\arabic{table}}  
\setcounter{figure}{0}
\setcounter{table}{0}

The definitions of key mathematical symbols used throughout this work are summarized in Table~\ref{tab:gsno_notation}.

\begin{table}[]
\centering
\footnotesize

\caption{\footnotesize Summary of notations used in the GSNO methodology.}
\label{tab:gsno_notation}
\resizebox{0.6\textwidth}{!}{%

\begin{tabular}{ll}
\toprule
\textbf{Notation} & \textbf{Meaning} \\
\midrule
\( D \subset \mathbb{R}^d \) & Spatial domain \\
\( \mathcal{A}, \mathcal{U} \) & Input/output function spaces \\
\( a \in \mathcal{A} \) & Input field (e.g., coefficients, initial conditions) \\
\( u \in \mathcal{U} \) & Output field (e.g., PDE solution) \\
\( G^\dagger \) & True PDE solution operator \\
\( G_\theta \) & Learnable neural operator with parameters \( \theta \) \\
\( v_t \in \mathbb{R}^{N_s \times T \times d_v} \) & Latent representation at layer \( t \) \\
\( W \) & Learnable pointwise (local) linear operator \\
\( \mathcal{K}_\phi \) & Learnable global spectral operator \\
\( \sigma \) & Nonlinear activation function (e.g., GELU) \\
\( N_s \) & Number of spatial nodes \\
\( T \) & Number of temporal steps \\
\( T_{\text{in}}, T_{\text{out}} \) & Number of input and output time steps \\
\( \Phi_{k_s} \in \mathbb{R}^{N_s \times k_s} \) & Truncated graph Laplacian eigenbasis \\
\( \hat{v}_s \in \mathbb{R}^{k_s \times T \times d_v} \) & Spatial graph Fourier transform of latent features \\
\( \hat{v}_{st} \in \mathbb{C}^{k_s \times k_t \times d_v} \) & Joint spatiotemporal Fourier representation \\
\( k_s \) & Number of retained spatial frequency modes \\
\( k_t \) & Number of retained temporal frequency modes \\
\( \mathcal{F}_t, \mathcal{F}_t^{-1} \) & Temporal Fourier transform and its inverse \\
\( R_\phi \in \mathbb{C}^{k_s \times k_t \times d_v \times d_v} \) & Learnable spectral convolution kernel \\
\( A \in \mathbb{R}^{N_s \times N_s} \) & Graph adjacency matrix (Gaussian-weighted) \\
\( D \in \mathbb{R}^{N_s \times N_s} \) & Degree matrix \\
\( \tilde{L} \in \mathbb{R}^{N_s \times N_s} \) & Normalized graph Laplacian \\
\( \Lambda \in \mathbb{R}^{N_s \times N_s} \) & Diagonal matrix of Laplacian eigenvalues \\
\( \hat{f} \) & Graph Fourier coefficients of a signal \( f \) \\
\bottomrule
\end{tabular}
}
\end{table}

\clearpage
\section{Approximation Guarantees, Truncation Bounds, and Equivariance of GSNO}
\label{app:theory}

\renewcommand{\thefigure}{B.\arabic{figure}}  
\renewcommand{\thetable}{B.\arabic{table}}  

\setcounter{figure}{0}
\setcounter{table}{0}

This appendix summarizes key theoretical properties of GSNO that support its design and empirical behavior. We first characterize the Laplacian spectral basis and quantify the approximation error induced by truncating to $k_s$ spatial modes via a graph-Sobolev tail bound. We then formalize expressivity within the class of joint space--time spectral-multiplier neural operators realized by GSNO, and establish permutation equivariance under node relabeling. Finally, we briefly discuss mesh-refinement consistency under standard spectral convergence assumptions, connecting these results to the zero-shot super-resolution experiments in the main text.

\subsection{Preliminaries and notation}
Let $G=(V,E)$ be an undirected graph with $|V|=N_s$ nodes and normalized graph Laplacian
$\tilde L\in\mathbb{R}^{N_s\times N_s}$. Let
\begin{equation}
\tilde L=\Phi\Lambda\Phi^\top,\qquad
\Lambda=\mathrm{diag}(\lambda_1\le\cdots\le\lambda_{N_s}),\quad \Phi^\top\Phi=I,
\end{equation}
where $\Phi=[\phi_1,\dots,\phi_{N_s}]$ is an orthonormal eigenbasis. For $k_s\le N_s$, define the truncated
eigenbasis $\Phi_{k_s}\in\mathbb{R}^{N_s\times k_s}$ and the orthogonal projector onto the first $k_s$ eigenspaces
\begin{equation}
P_{k_s}:=\Phi_{k_s}\Phi_{k_s}^\top.
\end{equation}
For time-dependent problems with $T$ time steps, let $\mathcal{F}_t$ and $\mathcal{F}_t^{-1}$ denote the
(unitary) discrete Fourier transform and its inverse along the temporal axis. For real-valued sequences we use
the real FFT convention; the statements below are unchanged up to standard rFFT conjugate-symmetry bookkeeping,
which is implicitly enforced by the inverse transform.

\subsection{Completeness of the Laplacian basis and truncation projector}
\begin{lemma}[Completeness on a given discretization]
\label{lem:complete}
For the fixed graph Laplacian $\tilde L$, the eigenvectors $\{\phi_i\}_{i=1}^{N_s}$ form an orthonormal basis
of $\mathbb{R}^{N_s}$. Consequently, $P_{k_s} f \to f$ for all $f\in\mathbb{R}^{N_s}$ as $k_s\to N_s$.
\end{lemma}
\begin{proof}
Since $\tilde L$ is real symmetric, it admits an orthonormal eigenbasis spanning $\mathbb{R}^{N_s}$.
Writing $f=\sum_{i=1}^{N_s}\langle f,\phi_i\rangle \phi_i$ gives
$P_{k_s}f=\sum_{i=1}^{k_s}\langle f,\phi_i\rangle \phi_i$, hence
$f-P_{k_s}f=\sum_{i>k_s}\langle f,\phi_i\rangle \phi_i\to 0$ as $k_s\to N_s$.
\end{proof}

\subsection{Spectral truncation error bounds via graph-Sobolev regularity}
For $\beta>0$, define the graph-Sobolev (spectral) norm
\begin{equation}
\|f\|_{H^\beta(\tilde L)}^2 \;:=\; \sum_{i=1}^{N_s}(1+\lambda_i)^\beta\,|\langle f,\phi_i\rangle|^2.
\end{equation}

\begin{lemma}[Graph spectral truncation bound]
\label{lem:trunc}
Let $\beta>0$ and $f\in\mathbb{R}^{N_s}$. Then
\begin{equation}
\|f-P_{k_s}f\|_2^2
\;=\;\sum_{i>k_s}|\langle f,\phi_i\rangle|^2
\;\le\;(1+\lambda_{k_s+1})^{-\beta}\,\|f\|_{H^\beta(\tilde L)}^2.
\end{equation}
\end{lemma}
\begin{proof}
For $i>k_s$, we have $(1+\lambda_i)^\beta\ge (1+\lambda_{k_s+1})^\beta$. Hence
\[
\sum_{i>k_s}|\langle f,\phi_i\rangle|^2
=\sum_{i>k_s}(1+\lambda_i)^{-\beta}\,(1+\lambda_i)^\beta|\langle f,\phi_i\rangle|^2
\le (1+\lambda_{k_s+1})^{-\beta}\sum_{i>k_s}(1+\lambda_i)^\beta|\langle f,\phi_i\rangle|^2,
\]
which is bounded by $(1+\lambda_{k_s+1})^{-\beta}\|f\|_{H^\beta(\tilde L)}^2$.
\end{proof}

\paragraph{Interpretation.}
Lemma~\ref{lem:trunc} shows that truncation error decays as $k_s$ increases, with faster decay when the signal
has higher graph-Sobolev regularity (i.e., energy concentrated in low graph frequencies). This formalizes the
accuracy--efficiency tradeoff induced by retaining only $k_s$ Laplacian modes.

\subsection{Universality within the class of spectral-multiplier neural operators}
Fix $(N_s,T)$ and define the \emph{joint spectral transform}
\begin{equation}
\mathcal{T}(v)\;:=\;\mathcal{F}_t\!\left(\Phi^\top v\right),
\qquad
\mathcal{T}^{-1}(\hat v)\;:=\;\Phi\,\mathcal{F}_t^{-1}(\hat v),
\end{equation}
acting channelwise on tensors $v\in\mathbb{R}^{N_s\times T\times d_v}$ (and correspondingly in the complex domain).

A \emph{joint spectral multiplier} is a tensor-valued function $M$ that assigns, for each spatial eigen-index $i$
and temporal frequency index $m$, a matrix $M(i,m)\in\mathbb{C}^{d_v\times d_v}$. It defines a linear operator
$\mathcal{K}_M$ by
\begin{equation}
\label{eq:mult}
\mathcal{K}_M(v)
:= \mathcal{T}^{-1}\!\big( M \cdot \mathcal{T}(v)\big),
\end{equation}
where $(M\cdot \hat v)(i,m,:) = M(i,m)\hat v(i,m,:)$.

\begin{definition}[Spectral-multiplier neural operator class]
\label{def:class}
Let $\mathfrak{G}$ denote the class of operators obtained by composing a finite number of layers of the form
\begin{equation}
v_{\ell+1}=\sigma\!\left(W_\ell v_\ell + \mathcal{K}_{M_\ell}(v_\ell)\right),
\qquad \ell=0,\dots,L-1,
\end{equation}
with pointwise linear maps $W_\ell$ (e.g., $1\times1$ convolutions), joint multipliers $M_\ell$,
and pointwise nonlinearity $\sigma$, together with input/output lifts $P,Q$.
\end{definition}

\begin{proposition}[Realizability by GSNO]
\label{prop:realize}
For any given discretization $(N_s,T)$, a GSNO layer with $k_s=N_s$ and full temporal FFT modes implements
\eqref{eq:mult} exactly by setting $R_\phi=M$ (up to rFFT conventions). Consequently, any operator
in the class $\mathfrak{G}$ can be realized by a GSNO network of the same depth and channel width.
\end{proposition}
\begin{proof}
With $k_s=N_s$, we have $\Phi_{k_s}=\Phi$ and $\Phi\Phi^\top=I$. With full FFT modes, $\mathcal{F}_t$ is invertible.
Comparing \eqref{eq:mult} to the GSNO spectral branch,
$\Phi\,\mathcal{F}_t^{-1}(R_\phi \cdot \mathcal{F}_t(\Phi^\top v))$, we obtain exact equality by choosing $R_\phi=M$.
Composition with pointwise maps $W_\ell$ and nonlinearities $\sigma$ is identical in Definition~\ref{def:class}.
\end{proof}

\begin{proposition}[Approximation by truncation]
\label{prop:trunc_operator}
Let $\mathcal{K}_M$ be a spectral multiplier operator \eqref{eq:mult}. Let $\Pi_{k_s}$ denote truncation of the
spatial spectrum to the first $k_s$ modes (i.e., replacing $\Phi$ with $\Phi_{k_s}$ and zeroing higher modes).
Then for any $v$,
\begin{equation}
\|\mathcal{K}_M(v)-\mathcal{K}_M(\Pi_{k_s}v)\|_2
\;\le\;
\|M\|_{\mathrm{op}}\;\|v-P_{k_s}v\|_2,
\end{equation}
where $\|M\|_{\mathrm{op}}:=\max_{i,m}\|M(i,m)\|_{2\to 2}$.
In particular, if $v$ has bounded graph-Sobolev norm, the RHS decays with $k_s$ by Lemma~\ref{lem:trunc}.
\end{proposition}
\begin{proof}
The joint spectral transform is unitary (up to FFT convention constants), hence it preserves $\ell_2$ norms.
The multiplier acts pointwise in $(i,m)$, so
$\|M\cdot \hat v\|_2\le \|M\|_{\mathrm{op}}\|\hat v\|_2$.
Applying this to $\hat v=\mathcal{T}(v-P_{k_s}v)$ yields the claim.
\end{proof}

\paragraph{Takeaway.}
Propositions~\ref{prop:realize}--\ref{prop:trunc_operator} establish a precise universality statement:
GSNO is universal \emph{within} the class of neural operators whose global mixing is realized by joint
graph/time spectral multipliers, and truncation introduces a controlled, regularity-dependent approximation error.

\subsection{Permutation equivariance (formal mesh invariance under node relabeling)}
\begin{proposition}[Permutation equivariance]
\label{prop:perm}
Let $P$ be any permutation matrix (node relabeling). Define the permuted signal $(Pv)(i,:)=v(\pi(i),:)$ and
the permuted Laplacian $\tilde L':=P\tilde L P^\top$. Let $\Phi'_{k_s}$ be any orthonormal eigenbasis spanning
the same $k_s$-dimensional invariant subspace of $\tilde L'$ as $P\Phi_{k_s}$.
Then the truncated projector satisfies $P'_{k_s}=P P_{k_s} P^\top$, and the GSNO spectral operator obeys
\begin{equation}
\mathcal{K}_\phi'(Pv)=P\,\mathcal{K}_\phi(v),
\end{equation}
where $\mathcal{K}_\phi'$ is computed using $\tilde L'$ and $\Phi'_{k_s}$.
Consequently, each GSNO layer (including the pointwise branch $W$) is permutation equivariant.
\end{proposition}
\begin{proof}
Since $\tilde L'=P\tilde L P^\top$, the invariant subspace spanned by the first $k_s$ eigenvectors is mapped by $P$.
While individual eigenvectors may change sign or rotate within degenerate eigenspaces, the projector onto the
subspace is unique: $P'_{k_s}=\Phi'_{k_s}(\Phi'_{k_s})^\top = P\Phi_{k_s}\Phi_{k_s}^\top P^\top = P P_{k_s}P^\top$.
Because the GSNO spectral branch is linear and expressed via projection, FFT, frequency-wise multiplication,
inverse FFT, and reconstruction, these operations commute with permutation when written at the projector/subspace
level, yielding $\mathcal{K}_\phi'(Pv)=P\,\mathcal{K}_\phi(v)$. The pointwise map $W$ satisfies $W(Pv)=P(Wv)$, and
pointwise nonlinearities preserve equivariance.
\end{proof}

\begin{remark}[Eigenvector sign and degeneracy in practice]
\label{rem:eig_ambiguity}
Proposition~\ref{prop:perm} is most naturally understood at the level of the invariant subspace (or projector)
$P_{k_s}$, since eigenvectors are only defined up to sign and may rotate within degenerate eigenspaces. Numerical
implementations may enforce a deterministic eigenvector convention (e.g., fixing the sign by requiring the
largest-magnitude entry of each eigenvector to be positive) or otherwise use subspace-invariant constructions
to avoid coefficient-orientation ambiguity.
\end{remark}

\subsection{Mesh refinement consistency (conditional remark)}
\begin{remark}[Refinement limit and zero-shot super-resolution]
\label{rem:refine}
A full proof of $G_\theta^{(h)}\to G^\dagger$ as $h \to 0$ requires additional assumptions from spectral graph
convergence (e.g., convergence of discrete Laplacian spectral projectors to the Laplace--Beltrami projector on
the underlying domain) and regularity/stability of the target operator under refinement. In particular, one typically
assumes uniform boundedness of an appropriate (continuous) Sobolev norm of the underlying fields across the refinement
sequence, together with spectral convergence of the projectors. Under such standard conditions, the truncation bounds
in Lemma~\ref{lem:trunc} and stability of multiplier operators imply a discretization-consistent low-frequency
representation whose approximation error decreases as $(k_s,k_t)$ increase. In the main paper we validate this refinement
robustness empirically via zero-shot super-resolution experiments.
\end{remark}

\clearpage
\section{PDE Setup and Data Generation}
\label{appendix:datasets}
\renewcommand{\thefigure}{C.\arabic{figure}}  
\renewcommand{\thetable}{C.\arabic{table}}  

\setcounter{figure}{0}
\setcounter{table}{0}

\begin{figure}
    \centering
    \includegraphics[trim={0 0 0 0}, clip, width=0.98\textwidth]{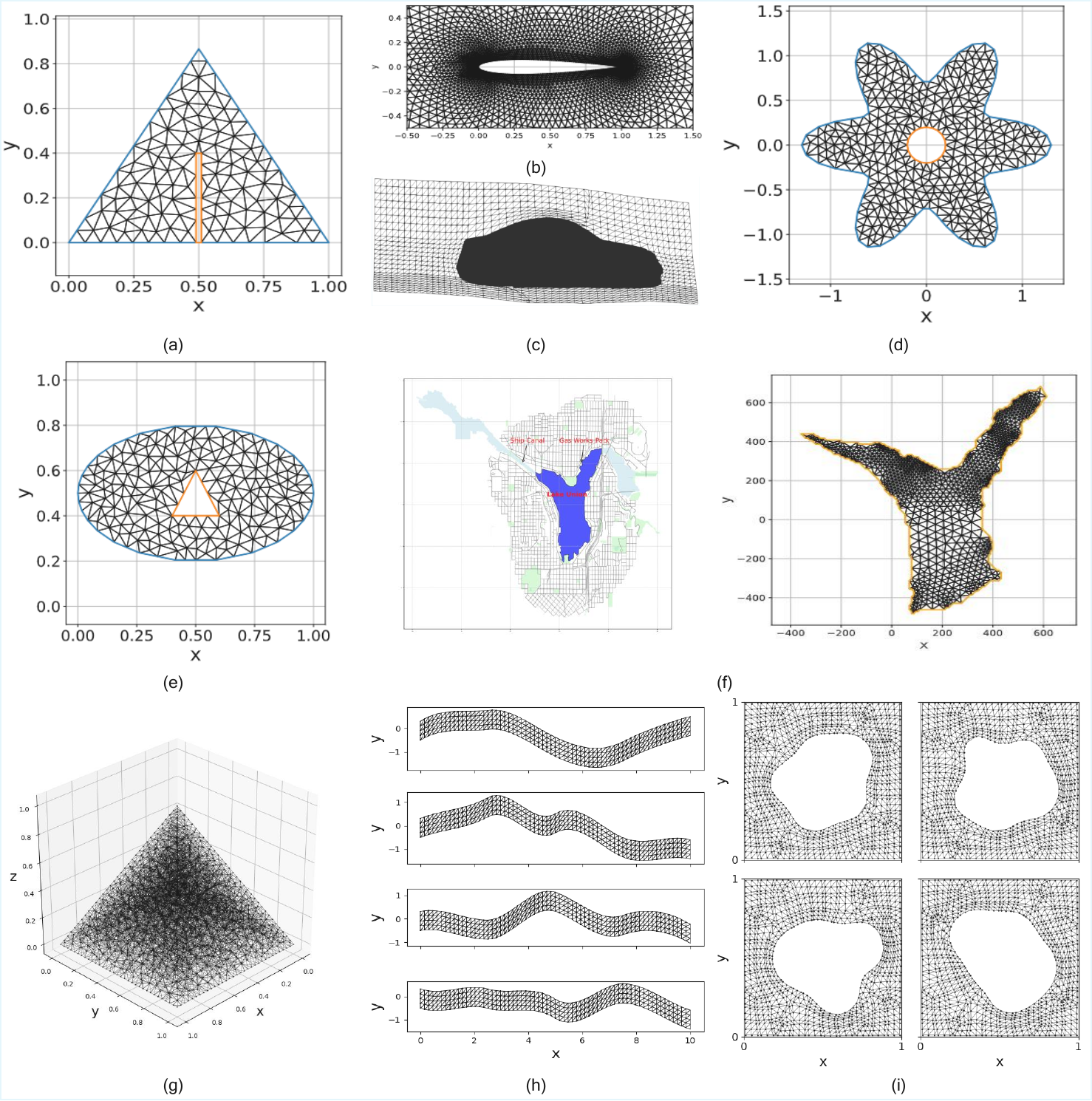}
    \captionsetup{font=footnotesize}
    \caption{Geometries used in the PDE benchmarks. Layout: (a) Darcy---triangle with a notch; (b) 2D Airfoil; (c) Shape-Net 3D Car, shown as a representative 2D cross-sectional slice for visualization; (d) Burgers---six-petal flower with a circular hole; (e) Navier--Stokes---ellipse with a triangular hole; (f) Shallow Water---Lake Union domain with a zoomed view of the southern inlet; (g) 3D Burgers---simple tetrahedral domain; (h) Pipe flow---geometry-dependent curved pipe domains; and (i) Hyper-elasticity---unit-cell domains with sample-dependent internal voids. All cases use irregular or geometry-dependent domains discretized with unstructured meshes.}

    \label{fig:pde_geometries}
\end{figure}

\subsection{\quad 2D Darcy Flow}
\label{appendix:darcy}
We consider the steady-state Darcy flow equation defined on an irregular domain shaped like a triangle with a single notch (see Figure~\ref{fig:pde_geometries}a). The PDE is given by:
\begin{align*}
    -\nabla \cdot \left(a(\mathbf{x}) \nabla u(\mathbf{x})\right) &= f(\mathbf{x}), && \mathbf{x} \in \Omega, \\
    u(\mathbf{x}) &= 0, && \mathbf{x} \in \partial \Omega,
\end{align*}
where $u(\mathbf{x})$ denotes the hydraulic head (solution field), and the forcing term is fixed as \( f(\mathbf{x}) = 1 \).

The spatially-varying diffusion coefficient \( a(\mathbf{x}) \) is drawn from a pushforward Gaussian random field distribution, \( a \sim \psi_\#\mathcal{N}(0,\ (-\Delta + 9I)^{-2}) \), where \( \psi \) is a nonlinear transformation to ensure positivity, and the Laplacian is defined with zero Neumann boundary conditions. The GRF is discretized on unstructured triangular meshes generated using PyMesh. We solve the equation using a generalized finite difference method (GFDM) on multiple mesh resolutions. This setup enables us to assess the model’s performance under complex geometries and varying spatial discretization scales. We curated 1{,}000 samples, allocated 600/200/200 to train/validation/test, and trained for 1{,}000 epochs.

\subsection{\quad Euler equations over a 2D Airfoil}
\label{appendix:Airfoil}
We consider subsonic/transonic flow past a two-dimensional airfoil (see Figure~\ref{fig:pde_geometries}b) using the compressible Euler equations on an unstructured mesh. This case is taken from the MeshGraphNets dataset \citep{pfaff2020learning}. The governing system is
\begin{align*}
    \partial_t \rho + \nabla\!\cdot(\rho\,\mathbf{u}) &= 0, \\
    \partial_t(\rho \mathbf{u}) + \nabla\!\cdot(\rho\,\mathbf{u}\!\otimes\!\mathbf{u} + p\,\mathbf{I}) &= 0,
\end{align*}
where \(\rho\) is density, \(\mathbf{u}\in\mathbb{R}^2\) is velocity, and \(p\) is pressure.

The domain is discretized using an unstructured triangular mesh with \(5{,}233\) vertices and \(10{,}216\) triangles. The dataset is evolved for 10 time steps and contains \(10{,}000\) training samples and \(1{,}000\) samples each for validation and testing. For each sample, the targets are the solution fields \(\{\rho,\,p,\,u_x,\,u_y\}\) defined on the mesh.

\subsection{\quad 2D unsteady Burgers’ Equation}
\label{appendix:burgers}
We study the two-dimensional vector-valued Burgers’ equation defined on an irregular domain shaped like a six-petal flower with a circular hole at its center (see Figure~\ref{fig:pde_geometries}d). The domain is embedded in the unit square and subject to no-slip boundary conditions, enforcing both velocity components to vanish along the boundary:
\begin{align*}
    \partial_t \mathbf{u}(\mathbf{x}, t) + \mathbf{u}(\mathbf{x}, t) \cdot \nabla \mathbf{u}(\mathbf{x}, t) &= \nu \Delta \mathbf{u}(\mathbf{x}, t), && \mathbf{x} \in \Omega,\ t \in (0, T], \\
    \mathbf{u}(\mathbf{x}, 0) &= \mathbf{u}_0(\mathbf{x}), && \mathbf{x} \in \Omega, \\
    \mathbf{u}(\mathbf{x}, t) &= \mathbf{0}, && \mathbf{x} \in \partial \Omega,
\end{align*}
where \( \mathbf{x} = (x, y) \) denotes spatial coordinates and \( \mathbf{u}(\mathbf{x}, t) = (u(\mathbf{x}, t), v(\mathbf{x}, t)) \) is the velocity vector field. The viscosity is set to \( \nu = 0.2 \). Spatial discretization is carried out using a generalized finite difference method (GFDM) over unstructured triangular meshes generated with PyMesh. Temporal integration is performed using a fourth-order adaptive Runge–Kutta scheme implemented through the \texttt{torchdiffeq} package. The initial condition \( \mathbf{u}_0(\mathbf{x}) \) is sampled componentwise from a Gaussian random field with distribution \( \mu = \mathcal{N}(0,\ 625(-\Delta + 25I)^{-2}) \), where the Laplacian is defined with zero Neumann boundary conditions. The PDE is solved across multiple mesh resolutions to evaluate the model’s generalization performance under varying discretization levels. 

The Burgers’ system was simulated for \(T=1.0s\) physical minutes, generating 51 temporal snapshots to capture nonlinear transport and dissipation. We prepared a dataset of 1{,}000 samples, divided into 600 for training, 200 for validation, and 200 for testing, and trained all models for 1{,}000 epochs.

\subsection{\quad 3D unsteady Burgers' equation}
\label{appendix:burgers3d}
We consider the three-dimensional viscous Burgers' equation on a fixed irregular domain, discretized using an unstructured tetrahedral mesh consisting of 1,867 vertices and 7,514 tetrahedra.The dynamics are governed by
\begin{align}
    \partial_t \mathbf{v} + (\mathbf{v}\cdot\nabla)\mathbf{v}
    &= \nu \Delta \mathbf{v},
    \qquad \mathbf{x}\in\Omega,\ t\in(0,T],
\end{align}
where \(\mathbf{v}=(u,v,w)\) denotes the velocity field and \(\nu=0.1\) is the
viscosity. Because the system contains nonlinear advection and diffusion but no
pressure term, it provides a clean and widely used benchmark for learning
nonlinear spatiotemporal operators in three dimensions.

As illustrated in Figure~\ref{fig:pde_geometries}g, the spatial domain \(\Omega\) is a fixed square-pyramid geometry represented by an irregular tetrahedral mesh. We consider different mesh resolutions, resulting in different numbers of vertices and tetrahedral elements. In the reported setup, the mesh contains \(533\) vertices and \(1763\) tetrahedral elements. This setup introduces nontrivial geometry and irregular mesh connectivity, making it well suited for evaluating operator learning methods on non-Cartesian discretizations.

For each sample, the initial condition is defined directly on the mesh nodes by
sampling smooth random fields as finite superpositions of sinusoidal modes:
\begin{equation}
    q_0(\mathbf{x})
    =
    \sum_{m=1}^{M}
    a_m \sin\left(\omega_m \mathbf{k}_m\cdot\mathbf{x}+\phi_m\right),
\end{equation}
where \(q_0\in\{u_0,v_0,w_0\}\), \(\mathbf{k}_m\) is a random unit direction,
\(\omega_m\) is a randomly sampled frequency, \(\phi_m\) is a random phase, and
\(a_m\) is a random amplitude. After sampling, the fields are normalized to a
prescribed magnitude, yielding smooth yet diverse initial velocity profiles
across different realizations.

Homogeneous Dirichlet boundary conditions are imposed on all boundary nodes, so
that \(\mathbf{v}=0\) on \(\partial\Omega\). The system is then advanced over a
finite time horizon \(T\) using an explicit time-stepping scheme.

Reference solutions are generated numerically on the same tetrahedral mesh,
resulting in spatiotemporal velocity fields \((u,v,w)\) defined at the mesh
vertices. Each trajectory corresponds to a different realization of the initial
condition while sharing the same geometry and governing equation. We use this
benchmark to assess the ability of GSNO to model nonlinear transport and
diffusion on three-dimensional unstructured domains.

\subsection{\quad 2D Navier–Stokes Equation}
\label{appendix:nse}
We study the two-dimensional incompressible Navier–Stokes equations in vorticity–stream function formulation, defined on an irregular domain shaped like an ellipse with a triangular hole (see Figure~\ref{fig:pde_geometries}e). The governing equations are:
\begin{align*}
    \partial_t w(\mathbf{x}, t) + \mathbf{u}(\mathbf{x}, t) \cdot \nabla w(\mathbf{x}, t) &= \nu \Delta w(\mathbf{x}, t) + f(\mathbf{x}), && \mathbf{x} \in \Omega,\ t \in (0, T], \\
    -\Delta \psi(\mathbf{x}, t) &= w(\mathbf{x}, t), && \mathbf{x} \in \Omega, \\
    \mathbf{u}(\mathbf{x}, t) &= \nabla^\perp \psi(\mathbf{x}, t) = \left( \frac{\partial \psi}{\partial y},\ -\frac{\partial \psi}{\partial x} \right), && \mathbf{x} \in \Omega, \\
    w(\mathbf{x}, 0) &= w_0(\mathbf{x}), && \mathbf{x} \in \Omega.
\end{align*}
$\nu$ here is the viscosity. No-slip boundary conditions are imposed by enforcing \( \mathbf{u} = 0 \) and \( \frac{\partial w}{\partial n} = 0 \) on \( \partial \Omega \). The initial vorticity \( w_0(\mathbf{x}) \) is sampled from a Gaussian random field with law \( \mathcal{N}(0,\ (-\Delta + 49I)^{-2.5}) \), where the Laplacian is equipped with zero Neumann boundary conditions. Spatial discretization is performed using a generalized finite difference method (GFDM) on unstructured triangular meshes generated via PyMesh. Time integration is carried out using a fourth-order adaptive Runge–Kutta scheme via the \texttt{torchdiffeq} package, consistent with the Burgers experiment.

The external forcing term is defined as:
\[
f(\mathbf{x}) = 0.1 \left( \sin(2\pi(x + y)) + \cos(2\pi(x - y)) \right).
\]

The Navier–Stokes solver was run for \(T=10s\) minutes of physical time, with 51 solution snapshots saved to resolve vortical dynamics and flow separation. From this, we generated a dataset of 1{,}000 samples, partitioned into 600 for training, 200 for validation, and 200 for testing, and trained all models for 1{,}000 epochs.

\subsection{\quad 2D Shallow Water Equations}
\label{appendix:swe}

We consider the two-dimensional nonlinear Shallow Water Equations (SWE) in conservative form, defined over an irregular domain shaped like the Lake Union (see Figure~\ref{fig:pde_geometries}f). The SWE system models the evolution of water surface height and horizontal momentum under gravity and is widely used for simulating wave propagation, including tsunami and flood inundation scenarios. The governing equations are:

\begin{align*}
    \partial_t h(\mathbf{x}, t) + \nabla \cdot (h \mathbf{v})(\mathbf{x}, t) &= 0, && \mathbf{x} \in \Omega,\ t \in (0, T], \\
    \partial_t (h v_x)(\mathbf{x}, t) + \nabla \cdot \left( h v_x^2 + \frac{1}{2} g h^2,\ h v_x v_y \right)(\mathbf{x}, t) &= 0, && \mathbf{x} \in \Omega,\ t \in (0, T], \\
    \partial_t (h v_y)(\mathbf{x}, t) + \nabla \cdot \left( h v_x v_y,\ h v_y^2 + \frac{1}{2} g h^2 \right)(\mathbf{x}, t) &= 0, && \mathbf{x} \in \Omega,\ t \in (0, T],
\end{align*}

where \( h \) denotes the fluid height, \( \mathbf{v} = (v_x, v_y) \) is the velocity field, and \( g = 8.81\) is the gravitational acceleration. The state variables are collectively represented as \( \mathbf{u} = [h,\ hu,\ hv]^\top \), and the system is solved using a finite volume method with Rusanov flux.

The shallow water simulation setup mimics the propagation of surface gravity waves initiated by a localized disturbance—an abstraction often used to model real-world scenarios such as tsunami generation from undersea earthquakes or landslides. The simulation begins with a quiescent water column and introduces a spatially localized Gaussian perturbation in the height field:

\begin{align*}
    h(\mathbf{x}, 0) &= h_{\text{base}} + \text{max\_field} \cdot \exp\left( -\frac{(x - x_c)^2 + (y - y_c)^2}{2 \sigma^2} \right), \\
    hu(\mathbf{x}, 0) &= 0, \quad hv(\mathbf{x}, 0) = 0, \quad \mathbf{x} \in \Omega,
\end{align*}

where \( (x_c, y_c) \in \Omega \) denotes the center of the perturbation—randomly sampled for each instance—and \( \sigma \) controls the spread of the Gaussian. We use fixed values \( h_{\text{base}} = 5.0 \), \( \text{max\_field} = 1.0 \), and \( \sigma = 20 \). This setup leads to outward-propagating circular wavefronts, reminiscent of the early stages of tsunami evolution in enclosed basins or coastal zones.

Reflective (wall) boundary conditions are applied on \( \partial \Omega \), enforcing zero normal velocity at the domain boundary:

\[
\mathbf{v}_{\text{inv}} = \mathbf{v} - 2(\mathbf{v} \cdot \mathbf{n})\mathbf{n},
\]

where \( \mathbf{n} \) is the outward unit normal vector. This condition ensures zero penetration and free-slip behavior, making it appropriate for modeling wave reflection against rigid coastal boundaries or natural terrain features. The SWE system is discretized on an unstructured triangular mesh containing 3,663 nodes, representing the Lake Union geometry (see Figure~\ref{fig:pde_geometries}f). A finite volume method is used to solve the conservative form of the equations, with Rusanov flux applied at each face. Reflective wall boundary conditions are enforced on all domain boundaries. Temporal integration is performed using a fourth-order adaptive Runge–Kutta scheme implemented via the \texttt{torchdiffeq} package. The simulation spans over 30 physical minutes, during which 51 solution snapshots are saved to capture the wave propagation dynamics. This configuration allows us to simulate wave propagation and reflection in a closed, irregular lake. The dataset includes multiple Gaussian-perturbed initial conditions sampled on varying mesh resolutions, enabling mesh-invariant surrogate modeling and resolution-aware prediction benchmarks. For this problem, we generated a dataset of 1{,}000 samples, divided into 600 for training, 200 for validation, and 200 for testing, and trained all models for 1{,}000 epochs.

\subsection{\quad Pipe Flow with Navier--Stokes Equation}
\label{appendix:pipe}

We consider incompressible flow through a two-dimensional pipe with geometry-dependent
centerline deformation. The problem is governed by the incompressible Navier--Stokes equations:
\begin{align}
    \partial_t \mathbf{v}
    + (\mathbf{v}\cdot\nabla)\mathbf{v}
    &=
    -\nabla p + \nu \Delta \mathbf{v},
    \qquad \mathbf{x}\in\Omega,\ t\in(0,T], \\
    \nabla\cdot\mathbf{v} &= 0,
    \qquad \mathbf{x}\in\Omega,\ t\in(0,T],
\end{align}
where \(\mathbf{v}\) is the velocity field, \(p\) is the pressure, and
\(\nu=0.005\) is the viscosity. A parabolic inlet velocity profile with maximum velocity
\(\mathbf{v}=[1,0]\) is imposed at the inlet, a free boundary condition is imposed at the outlet,
and no-slip boundary conditions are imposed along the pipe walls.

We use the pipe-flow benchmark from \cite{li2023fourier}. Reference solutions are generated using an implicit finite element solver with approximately \(4{,}000\) Taylor--Hood \(Q_2\)--\(Q_1\) mixed elements. The original data are provided on a \(129\times129\) structured representation of each deformed pipe, where the input consists of mesh point locations and the target output is the horizontal velocity field at the corresponding points.

As illustrated in Figure~\ref{fig:pde_geometries}h, the pipe has length \(10\) and width \(1\). Its centerline is parameterized by four piecewise cubic polynomials determined by the vertical positions and slopes of five uniformly spaced control nodes. The vertical positions are sampled from \(U[-2,2]\), while the slopes are sampled from \(U[-1,1]\). Therefore, each sample defines a different curved pipe geometry, making this a varying-geometry Navier--Stokes problem. The domain is discretized using an unstructured triangular mesh with \(4{,}225\) vertices and \(8{,}192\) triangles.

To evaluate GSNO across mesh resolutions, we construct additional dataset versions by resampling the solution fields for each pipe geometry onto point clouds with different densities. This produces different node locations and graph connectivity across resolutions while preserving the sample-dependent pipe geometries. The dataset contains \(1{,}000\) training samples and \(200\) test samples. This benchmark evaluates GSNO’s ability to model flow fields on geometry-dependent domains while preserving boundary-layer and global transport behavior.

\subsection{\quad Hyper-Elasticity Equation}
\label{appendix:hyperelasticity}

We consider a two-dimensional hyper-elastic solid mechanics benchmark from \cite{li2023fourier}, defined on a unit cell \(\Omega=[0,1]\times[0,1]\) with a sample-dependent void at the center. The governing equation for the solid body is
\begin{equation}
    \rho_s \frac{\partial^2 \mathbf{u}}{\partial t^2}
    + \nabla \cdot \boldsymbol{\sigma} = 0,
    \qquad \mathbf{x}\in\Omega,
\end{equation}
where \(\rho_s\) is the material density, \(\mathbf{u}\) is the displacement field, and
\(\boldsymbol{\sigma}\) is the stress tensor. The stress is related to strain through an incompressible Rivlin--Saunders hyper-elastic material model:
\begin{equation}
    \boldsymbol{\sigma}
    =
    \frac{\partial W(\boldsymbol{\epsilon})}{\partial \boldsymbol{\epsilon}},
    \qquad
    W(\boldsymbol{\epsilon})
    =
    C_1(I_1-3)+C_2(I_2-3),
\end{equation}
where \(I_1=\mathrm{tr}(\mathbf{C})\),
\(I_2=\frac{1}{2}\left[(\mathrm{tr}(\mathbf{C}))^2-\mathrm{tr}(\mathbf{C}^2)\right]\), and
\(\mathbf{C}=2\boldsymbol{\epsilon}+\mathbf{I}\) is the right Cauchy--Green stretch tensor.

The bottom boundary of the unit cell is clamped, while a tensile traction
\(\mathbf{t}=[0,100]\) is applied on the top boundary. As illustrated in Figure~\ref{fig:pde_geometries}i, the central void shape varies across samples and is parameterized by a random radius field, so each sample corresponds to a different internal geometry. The computational domain is represented by an unstructured triangular discretization consisting of \(N_s=972\) vertices and \(1908\) triangular elements. Reference solutions are generated using a finite element solver on geometry-dependent meshes with approximately \(10^3\) spatial points. The original input is given as a point cloud representing the sample-dependent computational domain, and the target output is the stress field. The dataset contains \(1{,}000\) training samples and \(200\) test samples. This benchmark evaluates GSNO's ability to learn stress fields on geometry-dependent domains with localized stress concentration near internal boundaries.

\subsection{\quad Shape-Net 3D Car}
\label{appendix:3d_Car}

We utilize the Car dataset (see Figure~\ref{fig:pde_geometries}c) introduced by \cite{umetani2018learning}, which sources its base geometries from the ShapeNet Car category \citep{chang2015shapenet}. Consistent with the procedure in \cite{umetani2018learning}, these geometries were manually modified to remove tires, spoilers, and side mirrors. The dataset was generated by obtaining time-averaged pressure and velocity fields from simulations of the Reynolds-Averaged Navier-Stokes (RANS) equations. These simulations incorporated a $k$--$\epsilon$ turbulence model, were stabilized using SUPG, and solved with a finite element method. A fixed inlet velocity of 20\,m/s (72\,km/h) was used, resulting in an approximate Reynolds number of $5 \times 10^6$. Each individual simulation required roughly 50 minutes to complete, with car surfaces discretized into 3.7k mesh points. From an initial pool of 889 instances, we selected the 611 water-tight shapes. This final dataset was then divided into 500 instances for training and 111 for validation.

\clearpage
\section{Amortized Cost of Graph Spectral Preprocessing}
\label{appendix:laplacian_eigendecomposition_cost}
\renewcommand{\thefigure}{D.\arabic{figure}}  
\renewcommand{\thetable}{D.\arabic{table}}  
\setcounter{figure}{0}
\setcounter{table}{0}

GSNO requires a truncated graph spectral basis for each computational geometry. Specifically, before training, we compute the \(k_s\) eigenpairs associated with the smallest eigenvalues of the normalized graph Laplacian using the \emph{locally optimal block preconditioned conjugate gradient} (LOBPCG) method~\cite{knyazev2001toward}. This eigendecomposition is performed only once for each geometry and constitutes an offline preprocessing step. Once computed, the same basis is reused across all mini-batches and training epochs and, for a given geometry, during inference. Therefore, eigendecomposition is not part of GSNO's recurrent per-epoch training cost. For fixed-geometry benchmarks, all samples are defined on the same mesh, and a single graph Laplacian eigendecomposition is sufficient for the entire dataset. For varying-geometry benchmarks, one basis is required for each distinct mesh. However, these decompositions are independent and can therefore be processed in batches and parallelized on the GPU.
LOBPCG avoids computing the full graph Laplacian spectrum and instead iteratively extracts only the required \(k_s\) low-frequency modes. This is particularly suitable for GSNO because the graph Laplacian is sparse: edges connect only neighboring mesh nodes, and each matrix--vector product therefore scales with
\(\mathcal{O}(|\mathcal{E}|)\), rather than the
\(\mathcal{O}(N_s^2)\) cost associated with a dense matrix. In addition, GSNO retains only a small truncated basis, with \(k_s=8\) in the experiments reported here. Because this preprocessing cost is incurred only once, its practical contribution should be evaluated after amortization over the full training procedure. For \(N_{\mathrm{epoch}}\) training epochs, the amortized preprocessing cost per epoch is defined as \(t_{\mathrm{amort}} = t_{\mathrm{total}} / N_{\mathrm{epoch}}\), where \(t_{\mathrm{total}}\) is the total wall-clock time required to compute the graph spectral bases for all geometries in the dataset. The corresponding overhead relative to the regular GSNO training cost is calculated as \(\mathrm{Overhead} = \left(t_{\mathrm{amort}} / t_{\mathrm{train}}\right) \times 100\%\), where \(t_{\mathrm{train}}\) denotes the measured training time per epoch. Table~\ref{tab:preprocess_cost} reports these quantities for \(N_{\mathrm{epoch}} = 1{,}000\), corresponding to the training duration used in our experiments.

\begin{table}[h!]
\centering
\caption{%
    Wall-clock cost of graph Laplacian eigendecomposition (LOBPCG,
    $k_s{=}8$) as a one-time offline preprocessing step,
    amortized over 1{,}000 training epochs.
}
\label{tab:preprocess_cost}
\resizebox{\textwidth}{!}{%
\begin{tabular}{lrrcrrrrrr}
\toprule
\multirow{2}{*}{\textbf{Benchmark}}
  & \multirow{2}{*}{$N_s$}
  & \multirow{2}{*}{$N_{\text{geom}}$}
  & \multirow{2}{*}{\textbf{Geometry}}
  & \multicolumn{6}{c}{\textbf{Preprocessing vs.\ training cost}} \\
\cmidrule(lr){5-10}
  &  &  &
  & \makecell{Per batch\\$t_{\text{batch}}$ (s)}
  & \makecell{Batch\\size}
  & \makecell{Total prep\\$t_{\text{total}}$ (s)}
  & \makecell{Amortized\\(s/epoch)}
  & \makecell{Train\\(s/epoch)}
  & \makecell{Overhead\\(\%)} \\
\midrule
Darcy Flow           & 1{,}184  & 1       & Fixed
    & 1.31  & 1  & 1.31   & 0.0013 & 2.24  & 0.058 \\
2D Airfoil           & 5{,}233  & 1       & Fixed
    & 5.74  & 1  & 5.74   & 0.0057 & 9.60  & 0.060 \\
2D Burgers'          & 1{,}168  & 1       & Fixed
    & 1.08  & 1  & 1.08   & 0.0011 & 4.05  & 0.027 \\
3D Burgers'          & 1{,}867  & 1       & Fixed
    & 2.14  & 1  & 2.14   & 0.0021 & 4.55  & 0.047 \\
Navier--Stokes       & 1{,}244  & 1       & Fixed
    & 1.46  & 1  & 1.46   & 0.0015 & 4.17  & 0.035 \\
Shallow Water        & 1{,}830  & 1       & Fixed
    & 1.92  & 1  & 1.92   & 0.0019 & 4.38  & 0.044 \\
\midrule
Hyper-Elasticity     & 972      & 1{,}000 & Varying
    & 1.07  & 50 & 21.40  & 0.0214 & 2.15  & 1.00 \\
Pipe Flow            & 4{,}225  & 1{,}000 & Varying
    & 5.82  & 50 & 116.40 & 0.1164 & 7.20  & 1.62 \\
Shape-Net 3D Car     & 32{,}186 & 500     & Varying
    & 39.70 & 50 & 397.00 & 0.3970 & 24.00 & 1.65 \\
\bottomrule
\end{tabular}%
}
\end{table}

As shown in Table~\ref{tab:preprocess_cost}, the one-time eigendecomposition cost becomes negligible when amortized over training. For the fixed-geometry benchmarks, only one basis is required for the entire dataset. The total preprocessing time ranges from \(1.08\) to \(5.74\) seconds, resulting in an amortized overhead below \(0.07\%\) for every fixed-geometry case. Thus, in these benchmarks, graph spectral preprocessing has effectively no impact on the total training cost. The varying-geometry datasets require substantially more decompositions because every distinct mesh has its own graph Laplacian. Nevertheless, batching and parallel execution keep the total preprocessing cost small relative to the complete training procedure. For Hyper-Elasticity, computing the bases for \(1{,}000\) geometries requires \(21.40\) seconds in total, corresponding to \(0.0214\) seconds per epoch after amortization and an overhead of \(1.00\%\). For Pipe Flow, the total preprocessing time is \(116.40\) seconds, producing an amortized overhead of \(1.62\%\). The most demanding case is Shape-Net 3D Car, which contains \(500\) varying geometries with \(N_s=32{,}186\) nodes per mesh. Computing all graph spectral bases requires \(397\) seconds, or approximately \(6.6\) minutes, as a one-time preprocessing operation. When distributed over \(1{,}000\) training epochs, this corresponds to \(0.397\) seconds per epoch, compared with a regular training cost of \(24.0\) seconds per epoch. The resulting amortized overhead is therefore only \(1.65\%\). The preprocessing time is not determined solely by the number of spatial nodes. LOBPCG convergence also depends on factors such as graph sparsity, connectivity, mesh topology, and spectral separation. This explains why the measured preprocessing times do not vary perfectly linearly with \(N_s\). Nevertheless, the results consistently show that the graph Laplacian eigendecomposition is a modest one-time cost rather than a recurring computational bottleneck. Across all benchmarks, its amortized contribution remains below \(1.7\%\) of the per-epoch training cost, including datasets containing hundreds or thousands of distinct geometries.

\clearpage
\section{Additional results}
\label{appendix:more_results}
\renewcommand{\thefigure}{E.\arabic{figure}}  
\renewcommand{\thetable}{E.\arabic{table}}  
\setcounter{figure}{0}
\setcounter{table}{0}

\subsection{Additional results for Darcy Flow}
Figure~\ref{fig:sample_Darcy} shows a sample prediction for the steady-state Darcy flow problem, comparing the input field, ground truth, and GSNO output. Figure~\ref{fig:gsno-accuracy-gap_darcy} reports the resolution-based generalization performance, highlighting GSNO's accuracy gains over baselines. Figure~\ref{fig:runtime_comparison_darcy} presents the training efficiency of all models across varying mesh resolutions.

\begin{table}
\centering
\captionsetup{font=footnotesize}
\caption{Comparison of neural operator models on Darcy Flow (\( N_s = 1184 \)).}
\label{tab:darcy_metrics}
\scriptsize
\resizebox{0.45\textwidth}{!}{%
\begin{tabular}{l|ccc}
\toprule
\textbf{Model} & \textbf{Relative $L_2$ Error} & \textbf{RMSE} & \textbf{MAE} \\
\midrule
CORAL        & 0.0664 & \(4.82 \times 10^{-4}\) & \(3.38 \times 10^{-4}\) \\
Geo-FNO      & 0.0548 & \(3.97 \times 10^{-4}\) & \(2.91 \times 10^{-4}\) \\
MGKN         & 0.0242 & \(1.74 \times 10^{-4}\) & \(1.43 \times 10^{-4}\) \\
DeepONet     & 0.0312 & \(2.60 \times 10^{-4}\) & \(2.04 \times 10^{-4}\) \\
AMG          & 0.0172 & \(1.29 \times 10^{-4}\) & \(9.83 \times 10^{-5}\) \\
SP$^2$GNO    & 0.0150 & \(1.12 \times 10^{-4}\) & \(8.92 \times 10^{-5}\) \\
GNOT         & 0.0118 & \(8.86 \times 10^{-5}\) & \(6.74 \times 10^{-5}\) \\
Transolver   & 0.0142 & \(1.07 \times 10^{-4}\) & \(8.11 \times 10^{-5}\) \\
\textbf{GSNO (Ours)} & \textbf{0.0083} & \(\mathbf{1.84 \times 10^{-5}}\) & \(\mathbf{1.62 \times 10^{-5}}\) \\
\bottomrule
\end{tabular}
}
\end{table}

\begin{figure}[h!]
    \centering
    \includegraphics[trim={0 0 0 0}, clip, width=0.7\textwidth]{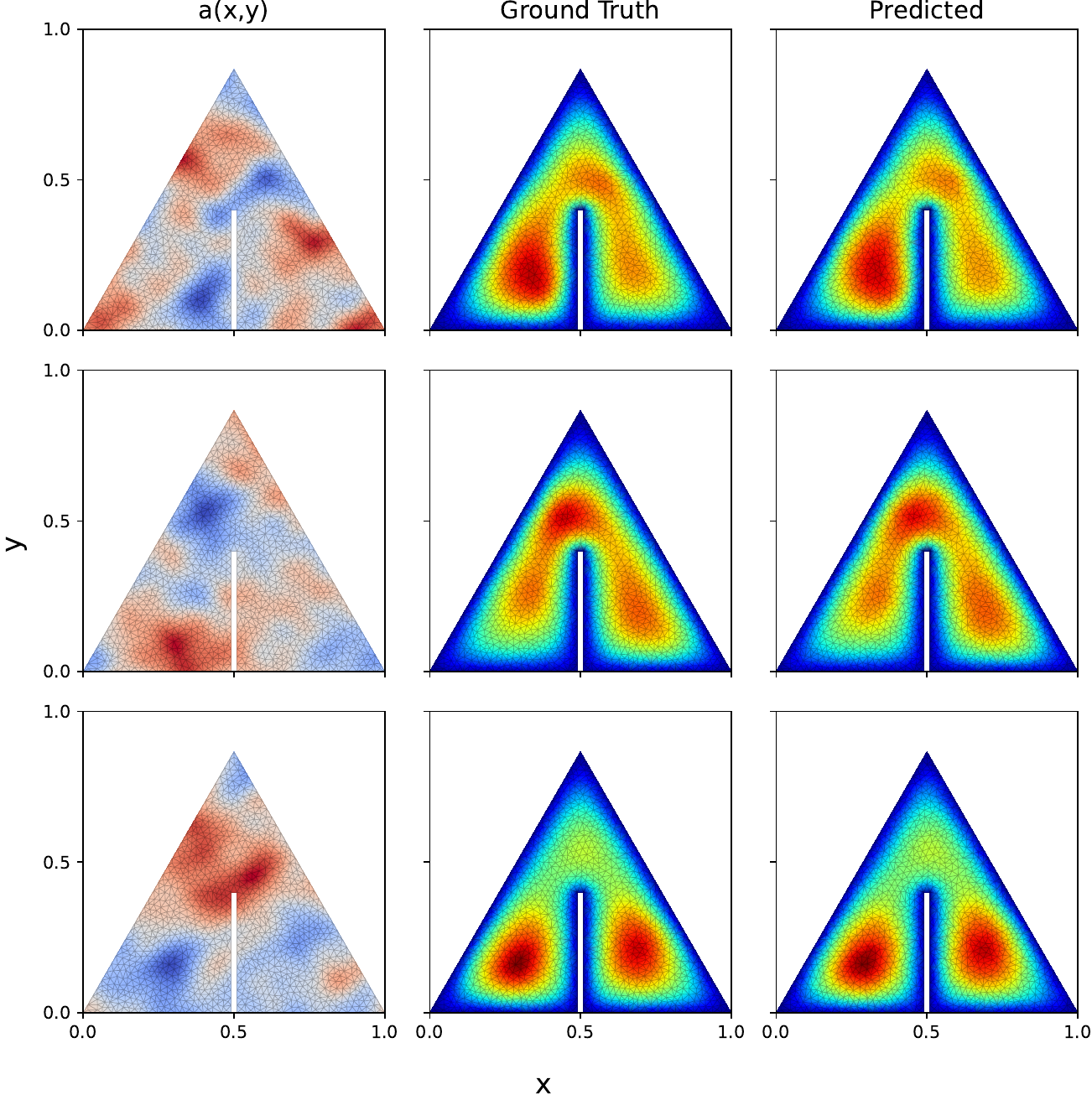}
    \captionsetup{font=footnotesize}
    \caption{Steady-state Darcy flow simulation on an irregular triangular domain with a notch. The first column shows the input diffusion field \( a(x, y) \), the second column shows the ground truth hydraulic head \( u \), and the third column presents GSNO predictions. The model is trained and tested on a mesh with \( N_s = 1184 \) points, highlighting GSNO’s ability to accurately recover the solution from heterogeneous input fields.}
    \label{fig:sample_Darcy}
\end{figure}

\begin{figure}
    \centering
    \includegraphics[width=\textwidth]{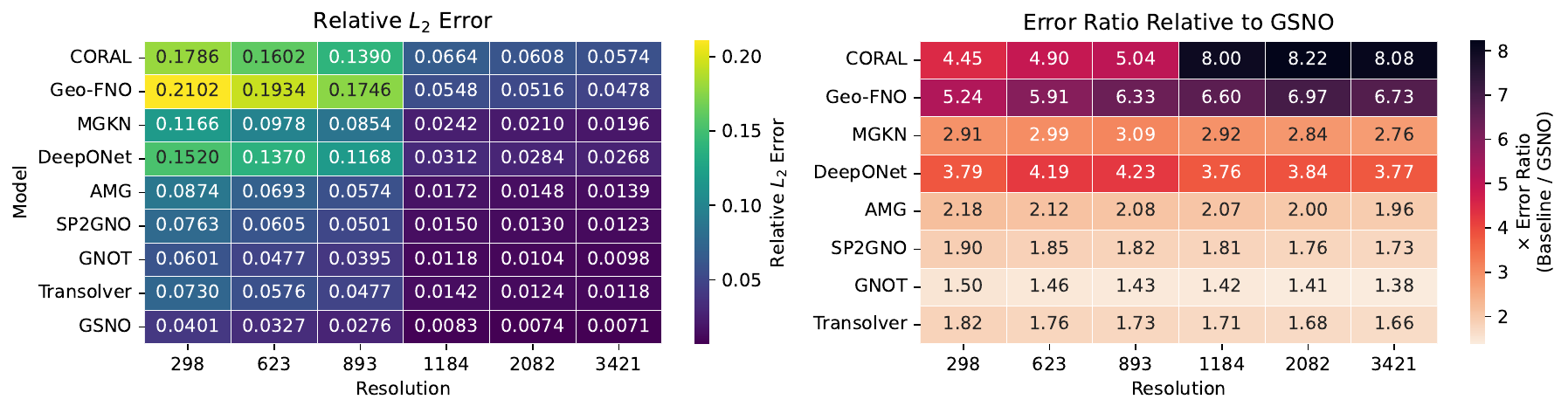}
    \captionsetup{font=footnotesize}
    \caption{
    Resolution-based generalization comparison for the steady-state Darcy flow problem.  
    \textbf{Left:} Relative $L_2$ error across increasing mesh resolutions for all models.  
    \textbf{Right:} Performance gap with respect to GSNO, shown as the ratio of each model’s error to GSNO’s at the same resolution.  
    All models are trained and evaluated on identical unstructured meshes.  
    GSNO demonstrates superior predictive accuracy across all resolutions, with error reductions of up to $8\times$ over competing methods.
    }
    \label{fig:gsno-accuracy-gap_darcy}
\end{figure}

\begin{figure}
    \centering
    \includegraphics[width=\textwidth]{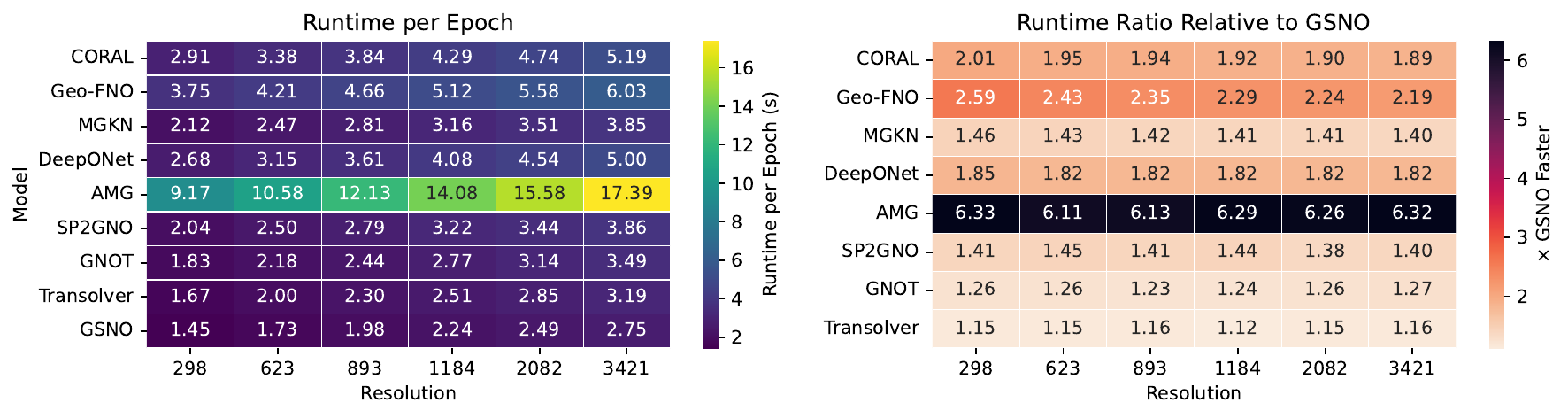}
    \captionsetup{font=footnotesize}
    \caption{
    Training efficiency of neural operator models on Darcy Flow across increasing spatial resolutions.  
    \textbf{Left:} Average runtime per epoch (in seconds) across six mesh resolutions from 298 to 3421 nodes.  
    \textbf{Right:} Slowdown relative to GSNO, computed as the ratio of each model's runtime to GSNO's at the same resolution.  
    GSNO consistently achieves the fastest per-epoch training, showcasing its efficiency on steady-state PDEs.
    }
    \label{fig:runtime_comparison_darcy}
\end{figure}

\newpage

\subsection{Additional results for Euler equations over a 2D Airfoil}

Figure~\ref{fig:sample_Airfoil} shows a qualitative one-step example for the 2D Airfoil. Table~\ref{tab:airfoil_relL2_mae_rmse} reports the corresponding quantitative metrics—training time (s/epoch) and errors (Relative \(L_2\), RMSE, MAE) for fluid quantities across all models.

\begin{figure}
    \centering
    \includegraphics[trim={0 10 0 0}, clip, width=0.7\textwidth]{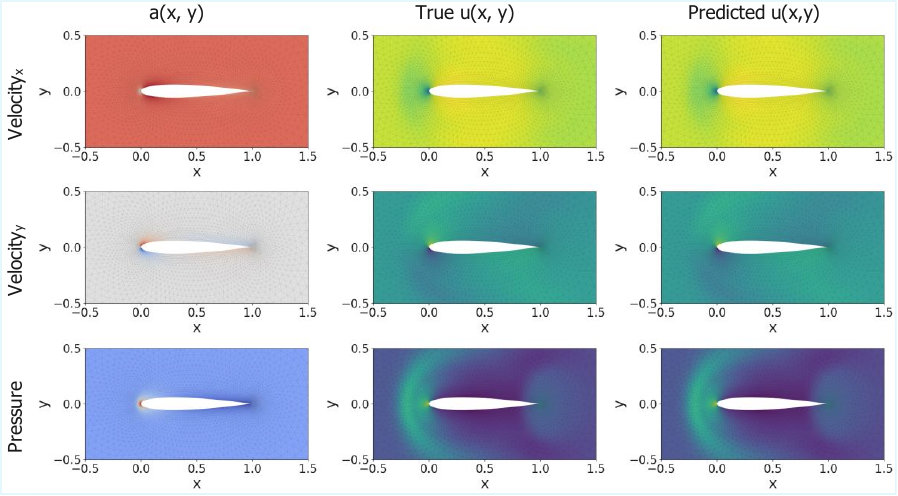}
    \captionsetup{font=footnotesize}
    \caption{2D Airfoil (compressible Euler) — one-step prediction. Columns: (left) input state at time \(t\), \(a(x,y)\); (middle) ground truth at \(t{+}1\), \(u(x,y)\); (right) GSNO prediction at \(t{+}1\). Rows: velocity \(u_x\) (top), velocity \(u_y\) (middle), and pressure \(p\) (bottom). All fields are on the same unstructured mesh with \(N_s = 5233\); the white region indicates the airfoil.}
    \label{fig:sample_Airfoil}
\end{figure}

\begin{table}
\centering
\captionsetup{font=footnotesize}
\caption{Comparison of neural operator models on \textbf{2D-Airfoil} (\(N_s=5233\))}
\label{tab:airfoil_relL2_mae_rmse}
\scriptsize
\resizebox{0.98\textwidth}{!}{%
{\setlength{\tabcolsep}{2.8pt}\renewcommand{\arraystretch}{1.05}%
\begin{tabular}{@{}l|c|ccc|ccc|ccc|ccc@{}}
\toprule
\multirow{2}{*}{\textbf{Model}}
& \multirow{2}{*}{\makecell{\textbf{Train}\\(s/epoch)}}
& \multicolumn{3}{c|}{\textbf{Density}}
& \multicolumn{3}{c|}{\textbf{Pressure}}
& \multicolumn{3}{c|}{\textbf{Velocity\_x}}
& \multicolumn{3}{c}{\textbf{Velocity\_y}} \\
& & \textbf{Rel.\(L_2\)} & \textbf{RMSE} & \textbf{MAE}
  & \textbf{Rel.\(L_2\)} & \textbf{RMSE} & \textbf{MAE}
  & \textbf{Rel.\(L_2\)} & \textbf{RMSE} & \textbf{MAE}
  & \textbf{Rel.\(L_2\)} & \textbf{RMSE} & \textbf{MAE} \\
\midrule
CORAL      & 25.2& 0.0650 & \(1.49 \times 10^{-2}\) & \(1.27 \times 10^{-2}\) & 0.0610 & \(1.34 \times 10^{-2}\) & \(1.13 \times 10^{-2}\) & 0.0365 & \(4.93 \times 10^{-3}\) & \(4.20 \times 10^{-3}\) & 0.0410 & \(6.15 \times 10^{-3}\) & \(5.12 \times 10^{-3}\) \\
Geo\mbox{-}FNO  & 29.8& 0.0580 & \(1.33 \times 10^{-2}\) & \(1.13 \times 10^{-2}\) & 0.0550 & \(1.21 \times 10^{-2}\) & \(1.02 \times 10^{-2}\) & 0.0320 & \(4.32 \times 10^{-3}\) & \(3.68 \times 10^{-3}\) & 0.0360 & \(5.40 \times 10^{-3}\) & \(4.50 \times 10^{-3}\) \\
MGKN       & 26.4& 0.0500 & \(1.15 \times 10^{-2}\) & \(0.975 \times 10^{-2}\) & 0.0480 & \(1.06 \times 10^{-2}\) & \(8.88 \times 10^{-3}\) & 0.0215 & \(2.90 \times 10^{-3}\) & \(2.47 \times 10^{-3}\) & 0.0260 & \(3.90 \times 10^{-3}\) & \(3.25 \times 10^{-3}\) \\
DeepONet   & 31.2& 0.0400 & \(9.20 \times 10^{-3}\) & \(7.80 \times 10^{-3}\) & 0.0370 & \(8.14 \times 10^{-3}\) & \(6.85 \times 10^{-3}\) & 0.0290 & \(3.92 \times 10^{-3}\) & \(3.34 \times 10^{-3}\) & 0.0310 & \(4.65 \times 10^{-3}\) & \(3.88 \times 10^{-3}\) \\
AMG        & 67.2& \underline{0.0021} & \(4.83 \times 10^{-4}\) & \(4.09 \times 10^{-4}\) & \underline{0.0020} & \(4.40 \times 10^{-4}\) & \(3.70 \times 10^{-4}\) & \underline{0.0014} & \(1.89 \times 10^{-4}\) & \(1.61 \times 10^{-4}\) & \underline{0.0018} & \(2.70 \times 10^{-4}\) & \(2.25 \times 10^{-4}\) \\
SP$^2$GNO  & 18.9& 0.0030 & \(6.90 \times 10^{-4}\) & \(5.80 \times 10^{-4}\) & 0.0028 & \(6.20 \times 10^{-4}\) & \(5.20 \times 10^{-4}\) & 0.0022 & \(3.00 \times 10^{-4}\) & \(2.50 \times 10^{-4}\) & 0.0025 & \(3.80 \times 10^{-4}\) & \(3.20 \times 10^{-4}\) \\
GNOT       & 12.1& 0.0054 & \(1.24 \times 10^{-3}\) & \(1.05 \times 10^{-3}\) & 0.0049 & \(1.08 \times 10^{-3}\) & \(8.97 \times 10^{-4}\) & 0.0040 & \(5.40 \times 10^{-4}\) & \(4.60 \times 10^{-4}\) & 0.0040 & \(6.00 \times 10^{-4}\) & \(5.00 \times 10^{-4}\) \\
Transolver & \underline{10.8}& 0.0036 & \(8.28 \times 10^{-4}\) & \(7.02 \times 10^{-4}\) & 0.0032 & \(7.04 \times 10^{-4}\) & \(5.92 \times 10^{-4}\) & 0.0028 & \(3.78 \times 10^{-4}\) & \(3.22 \times 10^{-4}\) & 0.0035 & \(5.25 \times 10^{-4}\) & \(4.38 \times 10^{-4}\) \\
\textbf{GSNO (Ours)} 
           & 9.6& \textbf{0.0012} & \(\mathbf{2.76 \times 10^{-4}}\) & \(\mathbf{2.34 \times 10^{-4}}\)
           & \textbf{0.0012} & \(\mathbf{2.64 \times 10^{-4}}\) & \(\mathbf{2.22 \times 10^{-4}}\)
           & \textbf{0.0009} & \(\mathbf{1.21 \times 10^{-4}}\) & \(\mathbf{1.04 \times 10^{-4}}\)
           & \textbf{0.0008} & \(\mathbf{1.20 \times 10^{-4}}\) & \(\mathbf{1.00 \times 10^{-4}}\) \\
\bottomrule
\end{tabular}%
}%
}
\end{table}

\newpage

\subsection{Additional results for Shape-Net 3D Car}

Figure~\ref{fig:sample_car} illustrates a representative prediction from GSNO on the Car3D dataset, highlighting the reconstructed velocity streamlines and pressure distribution. The corresponding quantitative comparison is provided in Table~\ref{tab:car3d_metrics}, which reports per-epoch training cost (s/epoch) together with error measures (Relative \(L_2\), RMSE, MAE) for velocity and pressure across all competing models.

\begin{figure}[h]
    \centering
    \includegraphics[trim={0 10 0 0}, clip, width=0.7\textwidth]{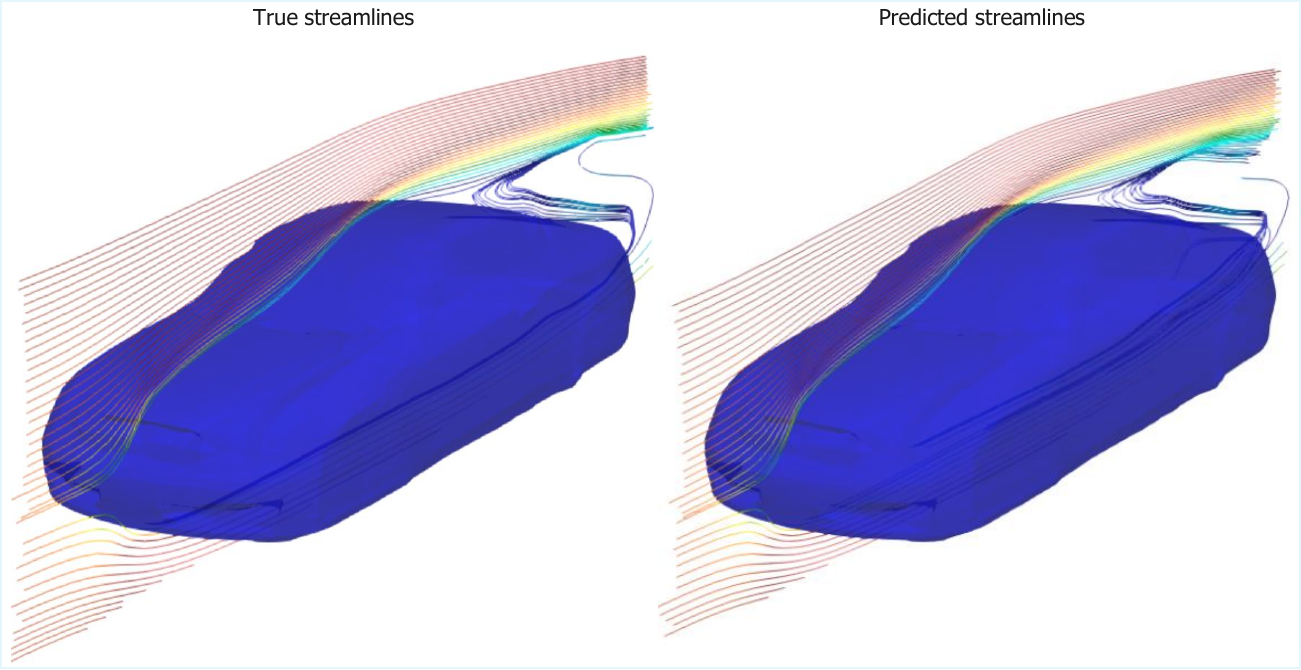}
    \captionsetup{font=footnotesize}
    \caption{Shape-Net 3D Car — flow streamlines. Left: reference simulation; right: GSNO prediction.}
    \label{fig:sample_car}
\end{figure}

\begin{table}[h]
\centering
\captionsetup{font=scriptsize}
\caption{Comparison of neural operator models on \textbf{Shape-Net 3D Car} (\( N_s = 32186 \))}
\label{tab:car3d_metrics}
\scriptsize
\resizebox{0.58\textwidth}{!}{%
{\setlength{\tabcolsep}{2.4pt}\renewcommand{\arraystretch}{1.05}%
\begin{tabular}{@{}l|c|ccc|ccc@{}}
\toprule
\multirow{2}{*}{\textbf{Model}} 
& \multirow{2}{*}{\makecell{\textbf{Train}\\(s/epoch)}}
& \multicolumn{3}{c|}{\textbf{Pressure}} 
& \multicolumn{3}{c}{\textbf{Velocity}} \\
& & \textbf{Rel.\(L_2\)} & \textbf{RMSE} & \textbf{MAE}
  & \textbf{Rel.\(L_2\)} & \textbf{RMSE} & \textbf{MAE} \\
\midrule
CORAL        & 62 & 0.1680 & \(4.20 \times 10^{-2}\) & \(3.57 \times 10^{-2}\) & 0.1750 & \(4.72 \times 10^{-2}\) & \(4.02 \times 10^{-2}\) \\
Geo-FNO      & 58 & 0.1560 & \(3.90 \times 10^{-2}\) & \(3.31 \times 10^{-2}\) & 0.1620 & \(4.37 \times 10^{-2}\) & \(3.72 \times 10^{-2}\) \\
MGKN         & 67 & 0.1350 & \(3.38 \times 10^{-2}\) & \(2.87 \times 10^{-2}\) & 0.1420 & \(3.83 \times 10^{-2}\) & \(3.26 \times 10^{-2}\) \\
DeepONet     & 87 & 0.1400 & \(3.50 \times 10^{-2}\) & \(2.98 \times 10^{-2}\) & 0.1480 & \(4.00 \times 10^{-2}\) & \(3.40 \times 10^{-2}\) \\
AMG          & 181& 0.0878 & \(2.20 \times 10^{-2}\) & \(1.87 \times 10^{-2}\) & 0.0919 & \(2.48 \times 10^{-2}\) & \(2.11 \times 10^{-2}\) \\
SP$^2$GNO    & 48 & 0.1005 & \(2.50 \times 10^{-2}\) & \(2.13 \times 10^{-2}\) & 0.1102 & \(2.95 \times 10^{-2}\) & \(2.50 \times 10^{-2}\) \\
GNOT         & 36 & 0.1199 & \(3.00 \times 10^{-2}\) & \(2.55 \times 10^{-2}\) & 0.1206 & \(3.26 \times 10^{-2}\) & \(2.77 \times 10^{-2}\) \\
Transolver   & \underline{27} & 0.0993 & \(2.48 \times 10^{-2}\) & \(2.11 \times 10^{-2}\) & 0.1208 & \(3.26 \times 10^{-2}\) & \(2.77 \times 10^{-2}\) \\
\textbf{GSNO (Ours)} 
             & 24 & \textbf{0.0712} & \(\mathbf{1.78 \times 10^{-2}}\) & \(\mathbf{1.51 \times 10^{-2}}\)
             & \textbf{0.0759} & \(\mathbf{2.05 \times 10^{-2}}\) & \(\mathbf{1.74 \times 10^{-2}}\) \\
\bottomrule
\end{tabular}%
}%
}
\end{table}

\newpage

\subsection{Additional results for 2D Burgers’ Equation}
Figure~\ref{fig:sample_burgers_uv} presents GSNO predictions for the 2D Burgers’ equation benchmark, showcasing both horizontal and vertical velocity components over time. Figures~\ref{fig:gsno-accuracy-gap_burger} and~\ref{fig:runtime_comparison_burger} provide a comparative analysis of generalization accuracy across resolutions and training efficiency under varying temporal settings.

\begin{table}[h]
\centering
\captionsetup{font=footnotesize}
\caption{Comparison of neural operator models on Burgers’ Equation (\( N_s = 1168 \)).}
\label{tab:burgers_metrics}
\scriptsize
\resizebox{0.98\textwidth}{!}{%
\begin{tabular}{l|ccc|ccc|ccc|ccc}
\toprule
\multirow{2}{*}{\textbf{Model}} 
& \multicolumn{3}{c|}{Temporal Config: 1$\rightarrow$50} 
& \multicolumn{3}{c|}{Temporal Config: 3$\rightarrow$48} 
& \multicolumn{3}{c|}{Temporal Config: 5$\rightarrow$46} 
& \multicolumn{3}{c}{Temporal Config: 10$\rightarrow$41} \\
\cmidrule(lr){2-4} \cmidrule(lr){5-7} \cmidrule(lr){8-10} \cmidrule(lr){11-13}
& Rel $L_2$ & RMSE & MAE 
& Rel $L_2$ & RMSE & MAE 
& Rel $L_2$ & RMSE & MAE 
& Rel $L_2$ & RMSE & MAE \\
\midrule
CORAL      & 0.1542 & \(1.30 \times 10^{-1}\) & \(9.47 \times 10^{-2}\)
& 0.1308 & \(1.05 \times 10^{-1}\) & \(7.97 \times 10^{-2}\)
& 0.1052 & \(7.76 \times 10^{-2}\) & \(6.15 \times 10^{-2}\)
& 0.0966 & \(6.88 \times 10^{-2}\) & \(5.52 \times 10^{-2}\) \\
Geo-FNO    & 0.1968 & \(1.59 \times 10^{-1}\) & \(1.10 \times 10^{-1}\)
& 0.1718 & \(1.32 \times 10^{-1}\) & \(9.58 \times 10^{-2}\)
& 0.1564 & \(1.17 \times 10^{-1}\) & \(8.71 \times 10^{-2}\)
& 0.1410 & \(1.02 \times 10^{-1}\) & \(7.77 \times 10^{-2}\) \\
MGKN       & 0.0876 & \(5.86 \times 10^{-2}\) & \(4.78 \times 10^{-2}\)
& \underline{0.0686} & \(3.68 \times 10^{-2}\) & \(3.10 \times 10^{-2}\)
& 0.0612 & \(3.09 \times 10^{-2}\) & \(2.62 \times 10^{-2}\)
& 0.0562 & \(2.76 \times 10^{-2}\) & \(2.35 \times 10^{-2}\) \\
DeepONet   & 0.1146 & \(8.56 \times 10^{-2}\) & \(6.70 \times 10^{-2}\)
& 0.1004 & \(6.98 \times 10^{-2}\) & \(5.59 \times 10^{-2}\)
& 0.0894 & \(5.79 \times 10^{-2}\) & \(4.72 \times 10^{-2}\)
& 0.0846 & \(5.25 \times 10^{-2}\) & \(4.32 \times 10^{-2}\) \\
AMG        & 0.0812 & \(5.45 \times 10^{-2}\) & \(4.46 \times 10^{-2}\)
& 0.0694 & \(3.74 \times 10^{-2}\) & \(3.15 \times 10^{-2}\)
& 0.0570 & \(2.88 \times 10^{-2}\) & \(2.48 \times 10^{-2}\)
& 0.0540 & \(2.63 \times 10^{-2}\) & \(2.25 \times 10^{-2}\) \\
GNOT       & 0.1301 & \(9.60 \times 10^{-2}\) & \(7.60 \times 10^{-2}\)
& 0.1232 & \(8.90 \times 10^{-2}\) & \(7.10 \times 10^{-2}\)
& 0.0907 & \(6.30 \times 10^{-2}\) & \(5.05 \times 10^{-2}\)
& 0.0855 & \(5.75 \times 10^{-2}\) & \(4.60 \times 10^{-2}\) \\
SP$^2$GNO  & 0.1054 & \(7.20 \times 10^{-2}\) & \(5.85 \times 10^{-2}\)
& 0.0998 & \(6.60 \times 10^{-2}\) & \(5.48 \times 10^{-2}\)
& 0.0736 & \(3.95 \times 10^{-2}\) & \(3.30 \times 10^{-2}\)
& 0.0695 & \(3.50 \times 10^{-2}\) & \(3.02 \times 10^{-2}\) \\
Transolver & \underline{0.0806} & \(5.40 \times 10^{-2}\) & \(4.42 \times 10^{-2}\)
& 0.0763 & \(4.10 \times 10^{-2}\) & \(3.44 \times 10^{-2}\)
& \underline{0.0564} & \(2.85 \times 10^{-2}\) & \(2.45 \times 10^{-2}\)
& \underline{0.0534} & \(2.58 \times 10^{-2}\) & \(2.20 \times 10^{-2}\) \\
\textbf{GSNO} & \textbf{0.0221} & \(\mathbf{1.21 \times 10^{-2}}\) & \(\mathbf{1.05 \times 10^{-2}}\)
& \textbf{0.0213} & \(\mathbf{1.16 \times 10^{-2}}\) & \(\mathbf{1.01 \times 10^{-2}}\)
& \textbf{0.0156} & \(\mathbf{7.60 \times 10^{-3}}\) & \(\mathbf{6.67 \times 10^{-3}}\)
& \textbf{0.0148} & \(\mathbf{7.20 \times 10^{-3}}\) & \(\mathbf{6.32 \times 10^{-3}}\) \\
\bottomrule
\end{tabular}
}
\end{table}

\begin{figure}[h]
    \centering
    \begin{subfigure}[t]{0.98\textwidth}
        \centering
        \includegraphics[width=\linewidth]{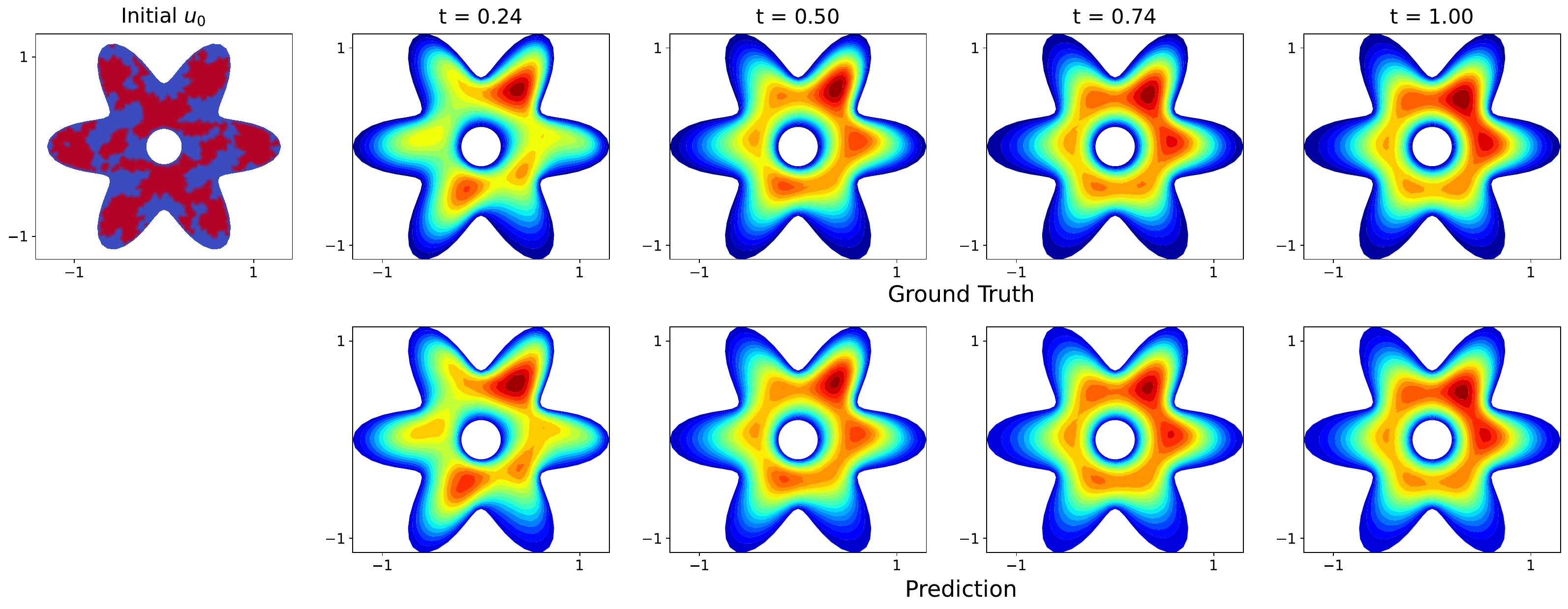}
    \end{subfigure}
    
    \vspace{4pt}
    
    \begin{subfigure}[t]{0.98\textwidth}
        \centering
        \includegraphics[width=\linewidth]{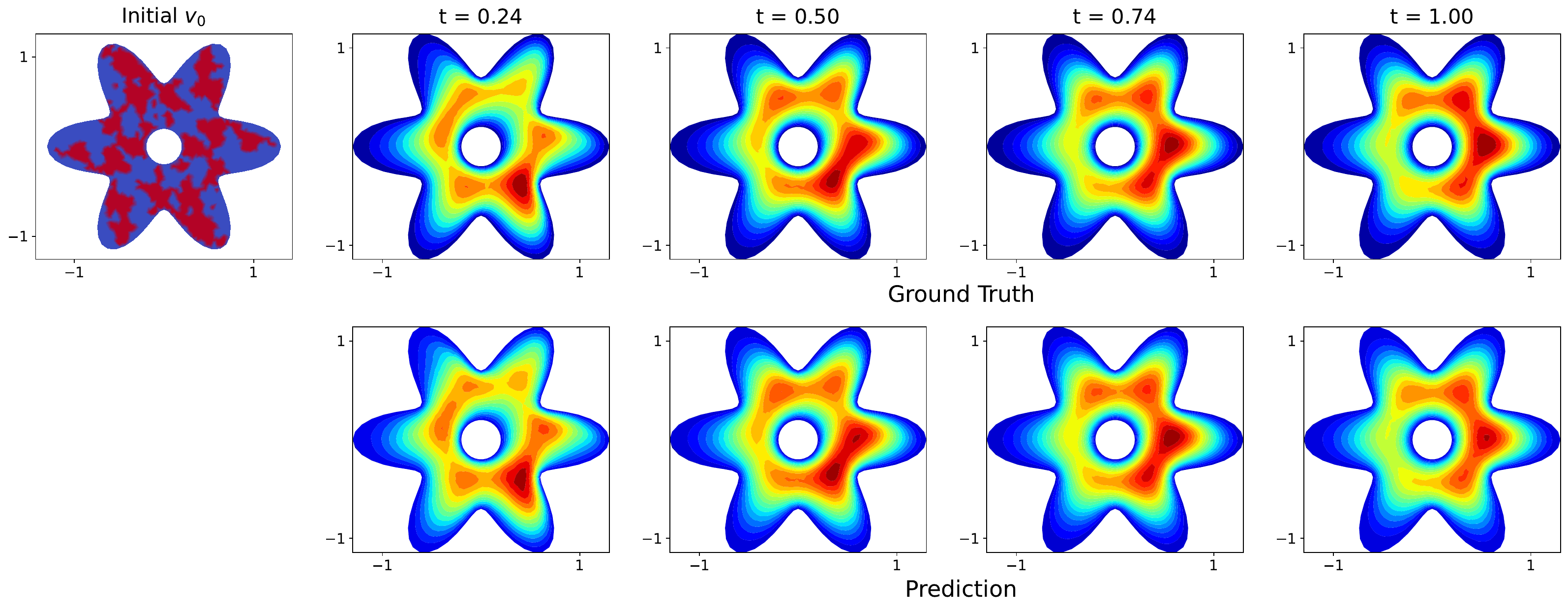}
    \end{subfigure}
    
    \captionsetup{font=footnotesize}
    \caption{2D Burgers’ equation on an irregular flower-shaped domain with a circular hole. \textbf{Top row:} Ground truth evolution of the horizontal velocity component \( u \). \textbf{Bottom row:} GSNO inference predictions of the vertical component \( v \). The model is trained and evaluated on the same resolution mesh with \( N_s = 1168 \) nodes and 51 temporal snapshots. Inference is performed using \( T_{\text{in}} = 5 \) input steps to forecast \( T_{\text{out}} = 46 \) future steps.}
    \label{fig:sample_burgers_uv}
\end{figure}

\begin{figure}[ht!]
    \centering
    \includegraphics[width=0.95\textwidth]{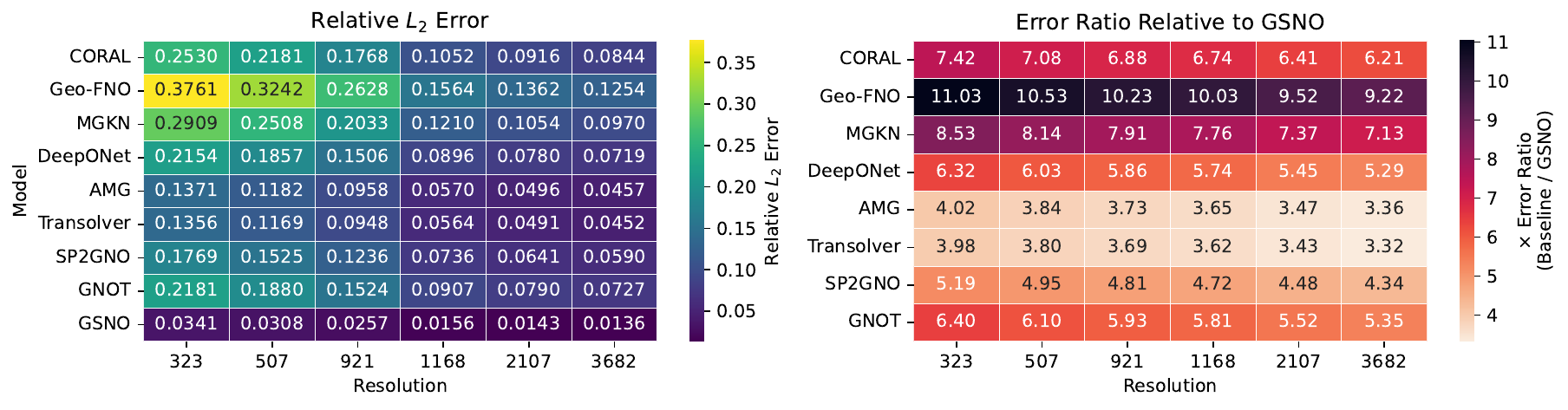}
    \captionsetup{font=footnotesize}
    \caption{
    Resolution-wise generalization results on the 2D Burgers’ equation benchmark.  
    \textbf{Left:} Relative $L_2$ error on the test set across increasing spatial resolutions.  
    \textbf{Right:} Error degradation relative to GSNO, computed as the ratio of each model's error to GSNO’s at the same resolution.  
    All models are trained and evaluated on matching meshes using $T_{\text{in}} = 5$ input snapshots to predict $T_{\text{out}} = 46$ future steps.  
    GSNO consistently outperforms baselines across all scales, achieving up to $10\times$ lower error.
    }
    \label{fig:gsno-accuracy-gap_burger}
\end{figure}

\begin{figure}[h]
    \centering
    \includegraphics[width=0.95\textwidth]{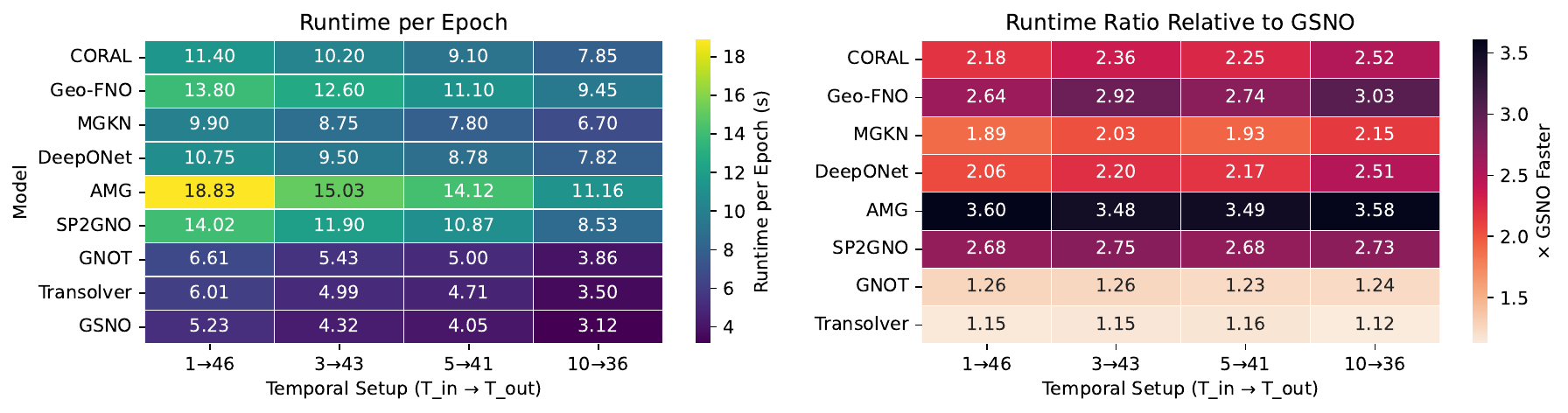}
    \captionsetup{font=footnotesize}
    \caption{
    \textbf{Training efficiency of neural operator models on Burgers’ equation under varying temporal configurations at resolution \( N_s = 1168 \).}  
    \textbf{Left:} Average runtime per epoch (in seconds) across different temporal input-output setups (\( T_{\text{in}} \rightarrow T_{\text{out}} \)).  
    \textbf{Right:} Slowdown relative to GSNO, computed as the ratio of each model's runtime to GSNO's at the same setting.  
    GSNO consistently achieves the lowest per-epoch runtime, highlighting its computational efficiency.
    }
    \label{fig:runtime_comparison_burger}
\end{figure}

\newpage

\subsection{Additional results for 3D unsteady Burgers' equation}

Figure~\ref{fig:runtime_comparison_3d_burgers} summarizes the per-epoch training cost for the 3D unsteady Burgers' equation under different input--output temporal settings. Across all configurations, GSNO requires the lowest runtime, indicating that its spectral space--time formulation remains computationally efficient for this high-dimensional unsteady benchmark.

Table~\ref{tab:swe_metrics} provides the corresponding quantitative comparison using relative \(L_2\), RMSE, and MAE. GSNO achieves the best accuracy in every temporal configuration, with especially strong gains over the closest competing methods as the prediction horizon becomes longer.

\begin{figure}[h]
    \centering
    \includegraphics[width=\textwidth]{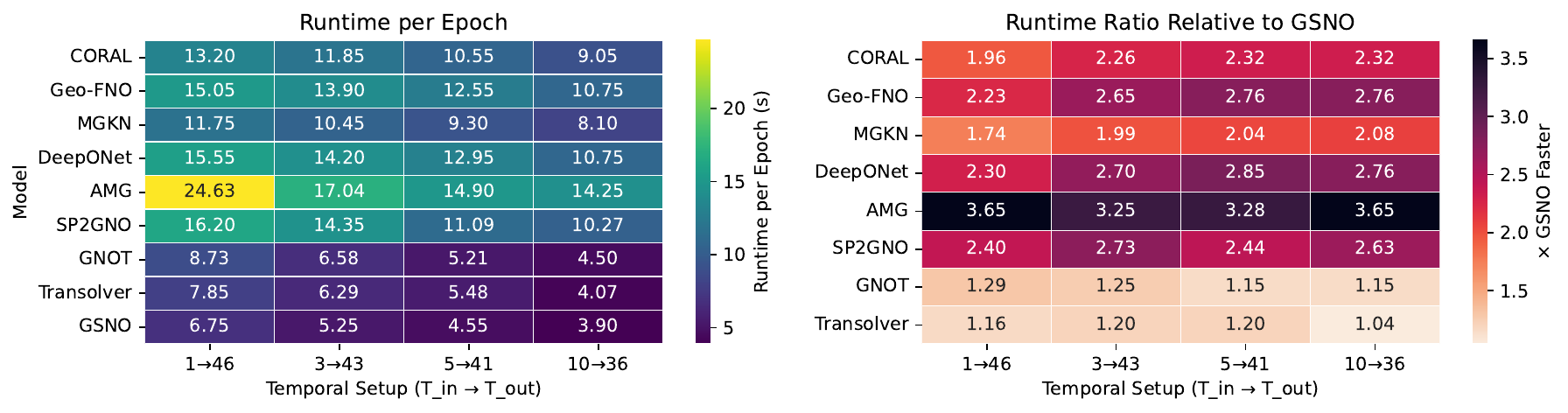}
    \captionsetup{font=footnotesize}
    \caption{
    \textbf{Training efficiency of neural operator models on the 3D unsteady Burgers’ equation at resolution \( N_s = 1867 \).}  
    \textbf{Left:} Average runtime per epoch (in seconds) for different temporal configurations (\( T_{\text{in}} \rightarrow T_{\text{out}} \)).  
    \textbf{Right:} Relative runtime overhead, computed as the ratio of each model’s runtime to GSNO’s.  
    GSNO consistently demonstrates the lowest computational cost across all temporal settings, highlighting its efficiency for high-dimensional time-dependent PDEs.
    }
    \label{fig:runtime_comparison_3d_burgers}
\end{figure}

\begin{table}[h]
\centering
\captionsetup{font=footnotesize}
\caption{Comparison of neural operator models on 3D unsteady Burgers' equation (\( N_s = 1867 \)).}
\label{tab:swe_metrics}
\scriptsize
\resizebox{0.98\textwidth}{!}{%
\begin{tabular}{l|ccc|ccc|ccc|ccc}
\toprule
\multirow{2}{*}{\textbf{Model}} 
& \multicolumn{3}{c|}{Temporal Config: 1$\rightarrow$50} 
& \multicolumn{3}{c|}{Temporal Config: 3$\rightarrow$48} 
& \multicolumn{3}{c|}{Temporal Config: 5$\rightarrow$46} 
& \multicolumn{3}{c}{Temporal Config: 10$\rightarrow$41} \\
\cmidrule(lr){2-4} \cmidrule(lr){5-7} \cmidrule(lr){8-10} \cmidrule(lr){11-13}
& Rel $L_2$ & RMSE & MAE 
& Rel $L_2$ & RMSE & MAE 
& Rel $L_2$ & RMSE & MAE 
& Rel $L_2$ & RMSE & MAE \\
\midrule

CORAL
& 0.1975 & \(1.15 \times 10^{-1}\) & \(9.20 \times 10^{-2}\)
& 0.1688 & \(9.80 \times 10^{-2}\) & \(7.84 \times 10^{-2}\)
& 0.1340 & \(7.77 \times 10^{-2}\) & \(6.22 \times 10^{-2}\)
& 0.1255 & \(7.28 \times 10^{-2}\) & \(5.82 \times 10^{-2}\) \\

Geo-FNO
& 0.2380 & \(1.38 \times 10^{-1}\) & \(1.10 \times 10^{-1}\)
& 0.2070 & \(1.20 \times 10^{-1}\) & \(9.60 \times 10^{-2}\)
& 0.1820 & \(1.06 \times 10^{-1}\) & \(8.48 \times 10^{-2}\)
& 0.1695 & \(9.83 \times 10^{-2}\) & \(7.86 \times 10^{-2}\) \\

MGKN
& \underline{0.1245} & \(7.22 \times 10^{-2}\) & \(5.78 \times 10^{-2}\)
& \underline{0.0965} & \(5.60 \times 10^{-2}\) & \(4.48 \times 10^{-2}\)
& 0.0798 & \(4.63 \times 10^{-2}\) & \(3.70 \times 10^{-2}\)
& 0.0735 & \(4.26 \times 10^{-2}\) & \(3.41 \times 10^{-2}\) \\

DeepONet
& 0.1415 & \(8.21 \times 10^{-2}\) & \(6.57 \times 10^{-2}\)
& 0.1235 & \(7.16 \times 10^{-2}\) & \(5.73 \times 10^{-2}\)
& 0.1080 & \(6.26 \times 10^{-2}\) & \(5.01 \times 10^{-2}\)
& 0.0995 & \(5.77 \times 10^{-2}\) & \(4.62 \times 10^{-2}\) \\

AMG
& 0.1325 & \(7.69 \times 10^{-2}\) & \(6.15 \times 10^{-2}\)
& 0.0975 & \(5.66 \times 10^{-2}\) & \(4.53 \times 10^{-2}\)
& 0.0755 & \(4.38 \times 10^{-2}\) & \(3.50 \times 10^{-2}\)
& 0.0685 & \(3.97 \times 10^{-2}\) & \(3.18 \times 10^{-2}\) \\

GNOT
& 0.2425 & \(1.41 \times 10^{-1}\) & \(1.13 \times 10^{-1}\)
& 0.1710 & \(9.92 \times 10^{-2}\) & \(7.94 \times 10^{-2}\)
& 0.1215 & \(7.05 \times 10^{-2}\) & \(5.64 \times 10^{-2}\)
& 0.1105 & \(6.41 \times 10^{-2}\) & \(5.13 \times 10^{-2}\) \\

SP$^2$GNO
& 0.1965 & \(1.14 \times 10^{-1}\) & \(9.12 \times 10^{-2}\)
& 0.1400 & \(8.12 \times 10^{-2}\) & \(6.50 \times 10^{-2}\)
& 0.1005 & \(5.83 \times 10^{-2}\) & \(4.66 \times 10^{-2}\)
& 0.0915 & \(5.31 \times 10^{-2}\) & \(4.25 \times 10^{-2}\) \\

Transolver
& 0.1500 & \(8.70 \times 10^{-2}\) & \(6.96 \times 10^{-2}\)
& 0.1065 & \(6.18 \times 10^{-2}\) & \(4.94 \times 10^{-2}\)
& \underline{0.0765} & \(4.44 \times 10^{-2}\) & \(3.55 \times 10^{-2}\)
& \underline{0.0695} & \(4.03 \times 10^{-2}\) & \(3.22 \times 10^{-2}\) \\

\textbf{GSNO}
& \textbf{0.0425} & \(\mathbf{2.47 \times 10^{-2}}\) & \(\mathbf{1.98 \times 10^{-2}}\)
& \textbf{0.0305} & \(\mathbf{1.77 \times 10^{-2}}\) & \(\mathbf{1.42 \times 10^{-2}}\)
& \textbf{0.0220} & \(\mathbf{1.28 \times 10^{-2}}\) & \(\mathbf{1.02 \times 10^{-2}}\)
& \textbf{0.0198} & \(\mathbf{1.15 \times 10^{-2}}\) & \(\mathbf{9.20 \times 10^{-3}}\) \\

\bottomrule
\end{tabular}
}
\end{table}

\newpage

\subsection{Additional results for 2D Navier–Stokes Equation}

Figure~\ref{fig:sample_NSE} demonstrates GSNO’s ability to perform zero-shot super-resolution on the 2D Navier–Stokes equation benchmark. Figures~\ref{fig:gsno-accuracy-gap_nse} and~\ref{fig:runtime_comparison_nse} further present generalization accuracy across resolutions and training efficiency under varying temporal input-output settings.

\begin{table}[h]
\centering
\captionsetup{font=footnotesize}
\caption{Comparison of neural operator models on Navier--Stokes Equation (\( N_s = 1244 \)).}
\label{tab:nse_metrics_mixed}
\scriptsize
\resizebox{0.95\textwidth}{!}{%
\begin{tabular}{l|ccc|ccc|ccc|ccc}
\toprule
\multirow{2}{*}{\textbf{Model}} 
& \multicolumn{3}{c|}{Temporal Config: 1$\rightarrow$50} 
& \multicolumn{3}{c|}{Temporal Config: 3$\rightarrow$48} 
& \multicolumn{3}{c|}{Temporal Config: 5$\rightarrow$46} 
& \multicolumn{3}{c}{Temporal Config: 10$\rightarrow$41} \\
\cmidrule(lr){2-4} \cmidrule(lr){5-7} \cmidrule(lr){8-10} \cmidrule(lr){11-13}
& Rel $L_2$ & RMSE & MAE 
& Rel $L_2$ & RMSE & MAE 
& Rel $L_2$ & RMSE & MAE 
& Rel $L_2$ & RMSE & MAE \\
\midrule
CORAL
& 0.1654 & \(8.91 \times 10^{-1}\) & \(6.59 \times 10^{-1}\) 
& 0.1412 & \(7.61 \times 10^{-1}\) & \(5.79 \times 10^{-1}\) 
& 0.1148 & \(6.19 \times 10^{-1}\) & \(4.86 \times 10^{-1}\) 
& 0.1070 & \(5.77 \times 10^{-1}\) & \(4.57 \times 10^{-1}\) \\
Geo-FNO
& 0.2056 & \(1.11 \times 10^{+0}\) & \(7.81 \times 10^{-1}\) 
& 0.1790 & \(9.65 \times 10^{-1}\) & \(7.02 \times 10^{-1}\) 
& 0.1614 & \(8.70 \times 10^{-1}\) & \(6.47 \times 10^{-1}\) 
& 0.1512 & \(8.15 \times 10^{-1}\) & \(6.13 \times 10^{-1}\) \\
MGKN
& \underline{0.0964} & \(5.20 \times 10^{-1}\) & \(4.16 \times 10^{-1}\) 
& \underline{0.0770} & \(4.15 \times 10^{-1}\) & \(3.39 \times 10^{-1}\) 
& 0.0654 & \(3.52 \times 10^{-1}\) & \(2.92 \times 10^{-1}\) 
& 0.0606 & \(3.27 \times 10^{-1}\) & \(2.72 \times 10^{-1}\) \\
DeepONet
& 0.1250 & \(6.74 \times 10^{-1}\) & \(5.23 \times 10^{-1}\) 
& 0.1096 & \(5.91 \times 10^{-1}\) & \(4.66 \times 10^{-1}\) 
& 0.0958 & \(5.16 \times 10^{-1}\) & \(4.14 \times 10^{-1}\) 
& 0.0916 & \(4.94 \times 10^{-1}\) & \(3.98 \times 10^{-1}\) \\
AMG
& 0.1060 & \(5.71 \times 10^{-1}\) & \(4.61 \times 10^{-1}\)
& 0.0780 & \(4.20 \times 10^{-1}\) & \(3.39 \times 10^{-1}\)
& 0.0580 & \(3.13 \times 10^{-1}\) & \(2.52 \times 10^{-1}\)
& 0.0540 & \(2.91 \times 10^{-1}\) & \(2.35 \times 10^{-1}\) \\
GNOT
& 0.1876 & \(1.01 \times 10^{+0}\) & \(8.16 \times 10^{-1}\)
& 0.1326 & \(7.15 \times 10^{-1}\) & \(5.77 \times 10^{-1}\)
& 0.0912 & \(4.91 \times 10^{-1}\) & \(3.97 \times 10^{-1}\)
& 0.0844 & \(4.55 \times 10^{-1}\) & \(3.67 \times 10^{-1}\) \\
SP$^2$GNO
& 0.1533 & \(7.35 \times 10^{-1}\) & \(5.90 \times 10^{-1}\)
& 0.1081 & \(5.80 \times 10^{-1}\) & \(4.65 \times 10^{-1}\)
& 0.0744 & \(3.65 \times 10^{-1}\) & \(2.95 \times 10^{-1}\)
& 0.0689 & \(3.35 \times 10^{-1}\) & \(2.70 \times 10^{-1}\) \\
Transolver
& 0.1189 & \(6.41 \times 10^{-1}\) & \(5.17 \times 10^{-1}\)
& 0.0836 & \(4.51 \times 10^{-1}\) & \(3.64 \times 10^{-1}\)
& \underline{0.0575} & \(3.10 \times 10^{-1}\) & \(2.50 \times 10^{-1}\)
& \underline{0.0534} & \(2.88 \times 10^{-1}\) & \(2.32 \times 10^{-1}\) \\
\textbf{GSNO}
& \textbf{0.0336} & \(\mathbf{1.81 \times 10^{-1}}\) & \(\mathbf{1.55 \times 10^{-1}}\) 
& \textbf{0.0237} & \(\mathbf{1.28 \times 10^{-1}}\) & \(\mathbf{1.11 \times 10^{-1}}\) 
& \textbf{0.0164} & \(\mathbf{8.84 \times 10^{-2}}\) & \(\mathbf{7.71 \times 10^{-2}}\) 
& \textbf{0.0152} & \(\mathbf{8.19 \times 10^{-2}}\) & \(\mathbf{7.15 \times 10^{-2}}\) \\
\bottomrule
\end{tabular}
}
\end{table}

\begin{figure}[h]
    \centering
    \includegraphics[width=0.8\textwidth]{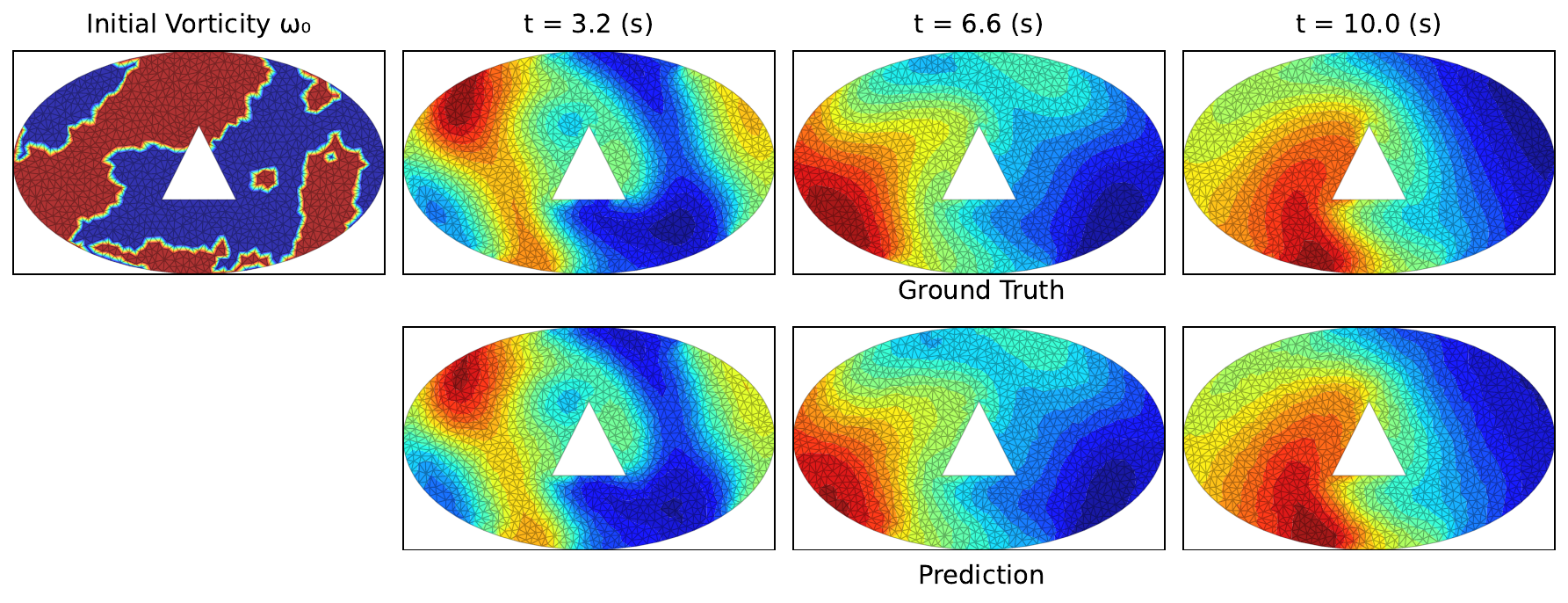}
    \captionsetup{font=footnotesize}
    \caption{Navier–Stokes simulation with viscosity \( \nu = 10^{-3} \), demonstrating the model’s ability to perform zero-shot super-resolution. The GSNO is trained on a coarse point cloud with \( N_s = 972 \) nodes and evaluated on a finer mesh with \( N_s = 1903 \) nodes without retraining. Ground truth results are shown on the top row, and GSNO predictions are shown on the bottom row. (See Section \ref{sec:zero_shot} for further details.)}
    \label{fig:sample_NSE}
\end{figure}

\begin{figure}[h]
    \centering
    \includegraphics[width=\textwidth]{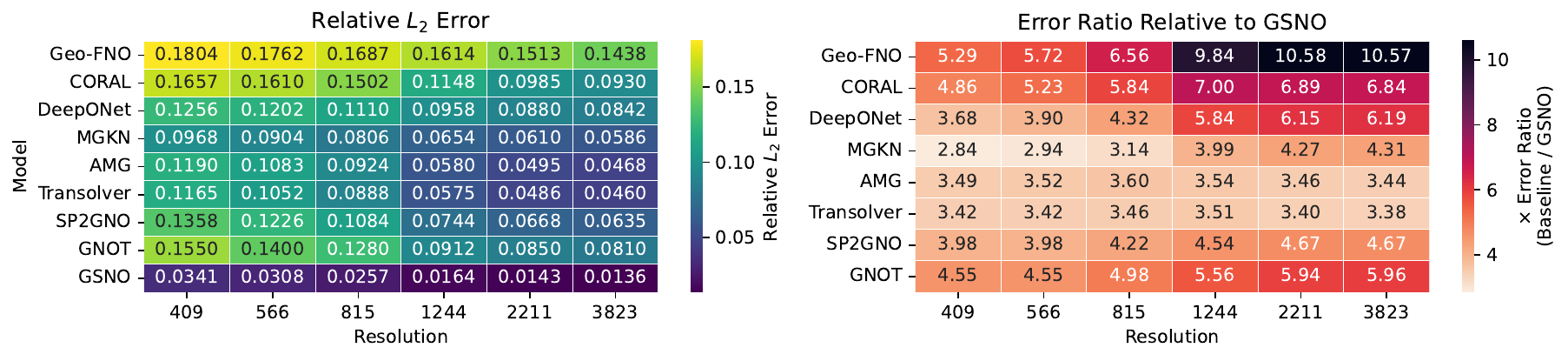}
    \captionsetup{font=footnotesize}
    \caption{
    Resolution-based generalization comparison for the 2D Navier–Stokes equation in vorticity form.  
    \textbf{Left:} Relative $L_2$ error across increasing mesh resolutions for all models.  
    \textbf{Right:} Accuracy gap with respect to GSNO, computed as the ratio of each model’s error to GSNO’s error at each resolution.  
    All models are trained and tested on matching unstructured meshes using \( T_{\text{in}} = 5 \rightarrow T_{\text{out}} = 46 \).  
    GSNO consistently achieves the lowest error across all resolutions, outperforming others by margins of up to $10\times$.
    }
    \label{fig:gsno-accuracy-gap_nse}
\end{figure}

\begin{figure}[!h]
    \centering
    \includegraphics[width=\textwidth]{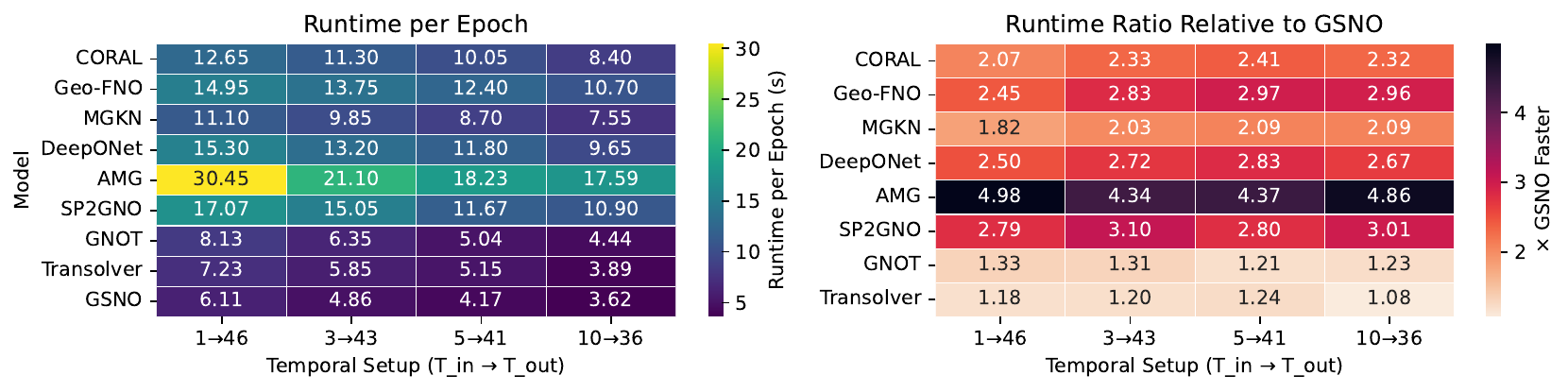}
    \captionsetup{font=footnotesize}
    \caption{
    \textbf{Training efficiency of neural operator models on the 2D Navier–Stokes equation under varying temporal configurations at resolution \( N_s = 1244 \).}  
    \textbf{Left:} Average runtime per epoch (in seconds) for different input-output lengths (\( T_{\text{in}} \rightarrow T_{\text{out}} \)).  
    \textbf{Right:} Runtime overhead relative to GSNO, computed as the ratio of each model's epoch runtime to GSNO's at the same setting.  
    GSNO remains the most computationally efficient, outperforming all baselines in training speed across temporal configurations.
    }
    \label{fig:runtime_comparison_nse}
\end{figure}

\newpage

\subsection{Additional results for Hyper-Elasticity Equation}

Figure~\ref{fig:sample_hyperelasticity} shows representative GSNO predictions for the hyper-elasticity benchmark with different void configurations. These examples illustrate that the model can recover the stress field across geometry-dependent solid domains with varying internal structures.

Table~\ref{tab:elasticity_relL2_rmse_mae} summarizes the quantitative results for the stress prediction task, including per-epoch runtime, relative \(L_2\), RMSE, and MAE. GSNO obtains the lowest error values while also maintaining the fastest training time among the compared methods.

\begin{figure}[!h]
    \centering
    \includegraphics[width=0.8\textwidth]{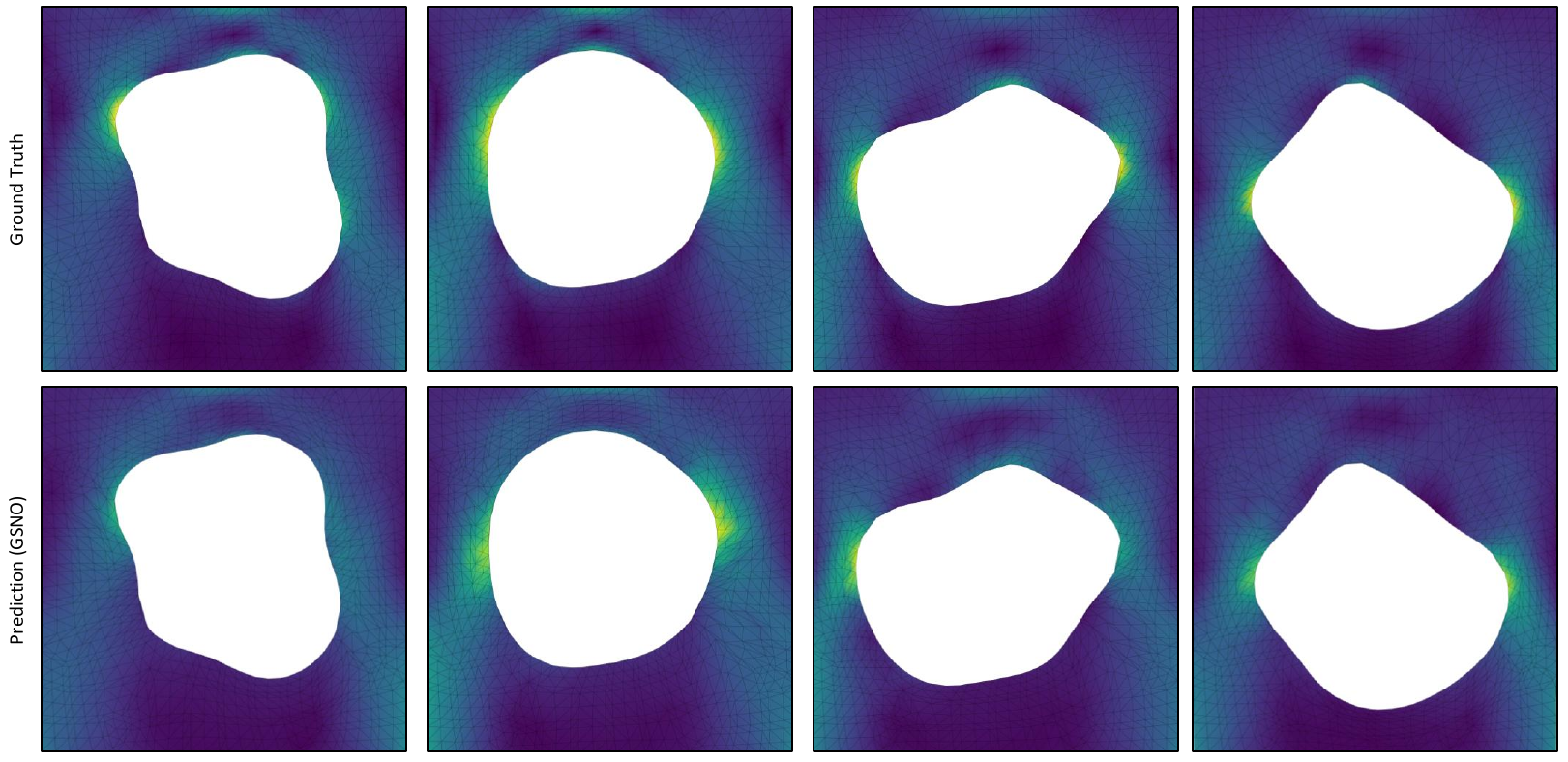}
    \captionsetup{font=footnotesize}
    \caption{Representative GSNO predictions for the hyper-elasticity benchmark with geometry-dependent voids. Each row corresponds to a distinct solid configuration, showing that the model can capture stress responses across different internal void patterns.}
    \label{fig:sample_hyperelasticity}
\end{figure}

\begin{table}[h]
\centering
\captionsetup{font=footnotesize}
\caption{Comparison of neural operator models on \textbf{Hyper-Elasticity} (\(N_s=972\))}
\label{tab:elasticity_relL2_rmse_mae}
\scriptsize
\resizebox{0.42\textwidth}{!}{%
{\setlength{\tabcolsep}{3pt}\renewcommand{\arraystretch}{1.1}%
\begin{tabular}{@{}l|c|ccc@{}}
\toprule
\multirow{2}{*}{\textbf{Model}}
& \multirow{2}{*}{\textbf{CPU (s/epoch)}}
& \multicolumn{3}{c}{\textbf{Stress field}} \\
& & \textbf{Rel.\(L_2\)} & \textbf{RMSE} & \textbf{MAE} \\
\midrule

CORAL      & 5.58 & 0.1091 & \(3.18 \times 10^{-1}\) & \(2.54 \times 10^{-1}\) \\
Geo-FNO    & 6.75 & 0.0937 & \(1.51 \times 10^{-1}\) & \(1.21 \times 10^{-1}\) \\
MGKN       & 6.00 & 0.0795 & \(1.06 \times 10^{-1}\) & \(8.47 \times 10^{-2}\) \\
DeepONet   & 7.17 & 0.1335 & \(3.11 \times 10^{-1}\) & \(2.49 \times 10^{-1}\) \\
AMG        & 14.78 & 0.0770 & \(1.02 \times 10^{-1}\) & \(8.20 \times 10^{-2}\) \\
SP$^2$GNO  & 4.27 & 0.1115 & \(2.81 \times 10^{-1}\) & \(2.25 \times 10^{-1}\) \\
GNOT       & 2.81 & 0.1097 & \(3.06 \times 10^{-1}\) & \(2.44 \times 10^{-1}\) \\
Transolver & 2.54 & 0.0827 & \(1.10 \times 10^{-1}\) & \(8.80 \times 10^{-2}\) \\
\textbf{GSNO} & \textbf{2.15} & \textbf{0.0271} & \(\mathbf{3.61 \times 10^{-2}}\) & \(\mathbf{2.88 \times 10^{-2}}\) \\

\bottomrule
\end{tabular}%
}%
}
\end{table}

\newpage

\subsection{Additional results for Pipe Flow}

Figure~\ref{fig:sample_pipe} presents representative GSNO predictions for the pipe flow benchmark across several curved pipe geometries. The examples show that GSNO can capture the velocity response over different geometric configurations, indicating its ability to handle shape variability within the same problem class.

Table~\ref{tab:pipe_relL2_rmse_mae} reports the quantitative comparison for the pipe flow case in terms of training time, relative \(L_2\), RMSE, and MAE. GSNO achieves the lowest prediction errors while also requiring the smallest per-epoch runtime among the compared models.

\begin{figure}[!h]
    \centering
    \includegraphics[width=0.8\textwidth]{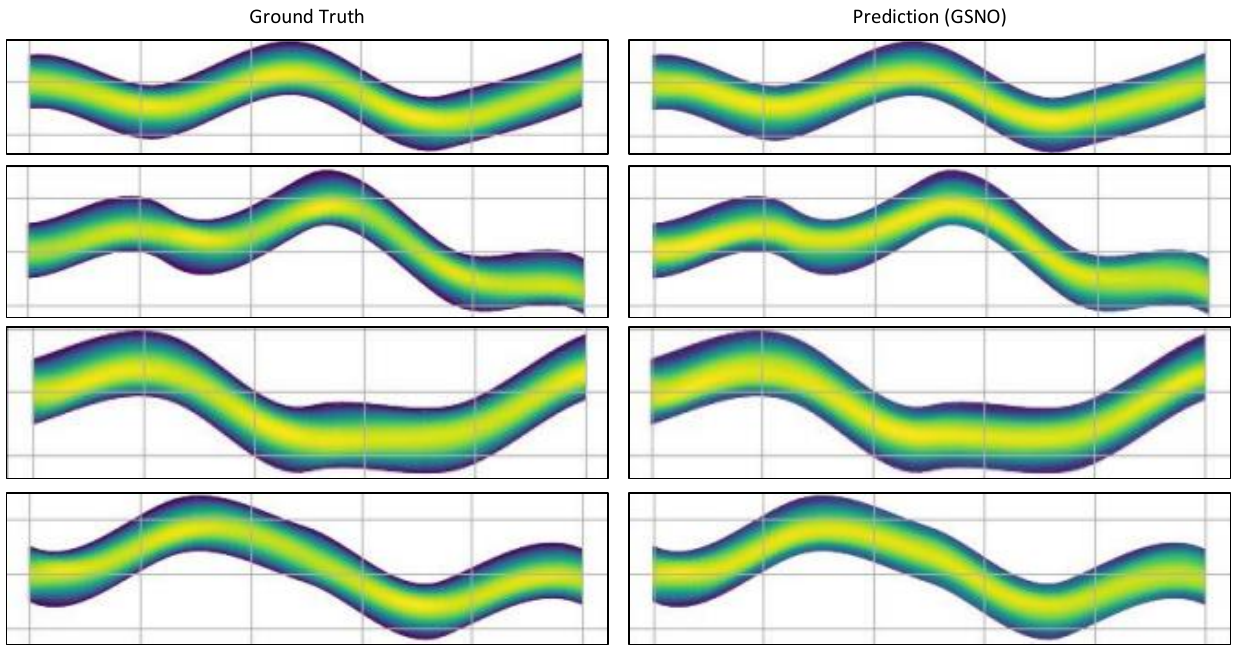}
    \captionsetup{font=footnotesize}
    \caption{Sample predictions of GSNO on pipe flow geometries with varying shapes. Each row shows a different geometry, illustrating the model’s ability to generalize across a variety of curved pipe configurations.}
    \label{fig:sample_pipe}
\end{figure}

\begin{table}
\centering
\captionsetup{font=footnotesize}
\caption{Comparison of neural operator models on \textbf{Pipe Flow} (\(N_s=4225\))}
\label{tab:pipe_relL2_rmse_mae}
\scriptsize
\resizebox{0.4\textwidth}{!}{%
{\setlength{\tabcolsep}{3pt}\renewcommand{\arraystretch}{1.1}%
\begin{tabular}{@{}l|c|ccc@{}}
\toprule
\multirow{2}{*}{\textbf{Model}}
& \multirow{2}{*}{\textbf{CPU (s/epoch)}}
& \multicolumn{3}{c}{\textbf{Velocity}} \\
& & \textbf{Rel.\(L_2\)} & \textbf{RMSE} & \textbf{MAE} \\
\midrule
CORAL      & 18.7 & 0.0461 & \(7.54 \times 10^{-2}\) & \(6.03 \times 10^{-2}\) \\
Geo-FNO    & 22.6 & 0.0501 & \(8.19 \times 10^{-2}\) & \(6.55 \times 10^{-2}\) \\
MGKN       & 20.1 & 0.0300 & \(4.91 \times 10^{-2}\) & \(3.93 \times 10^{-2}\) \\
DeepONet   & 24.0 & 0.0371 & \(6.06 \times 10^{-2}\) & \(4.85 \times 10^{-2}\) \\
AMG        & 49.5 & 0.0293 & \(4.79 \times 10^{-2}\) & \(3.83 \times 10^{-2}\) \\
SP$^2$GNO  & 14.3 & 0.0394 & \(6.43 \times 10^{-2}\) & \(5.14 \times 10^{-2}\) \\
GNOT       & 9.4  & 0.0436 & \(7.13 \times 10^{-2}\) & \(5.70 \times 10^{-2}\) \\
Transolver & 8.5  & 0.0310 & \(5.08 \times 10^{-2}\) & \(4.06 \times 10^{-2}\) \\
\textbf{GSNO} & \textbf{7.2} & \textbf{0.0142} & \(\mathbf{2.32 \times 10^{-2}}\) & \(\mathbf{1.86 \times 10^{-2}}\) \\
\bottomrule
\end{tabular}%
}%
}
\end{table}

\clearpage
\subsection{Additional results for 2D Shallow Water Equations Equation}

Figure~\ref{fig:swe_zero_shot} highlights GSNO’s zero-shot super-resolution capability on the shallow water equations, accurately transferring predictions from a coarse to a finer mesh.

\begin{figure}
  \centering
  \includegraphics[width=0.95\linewidth]{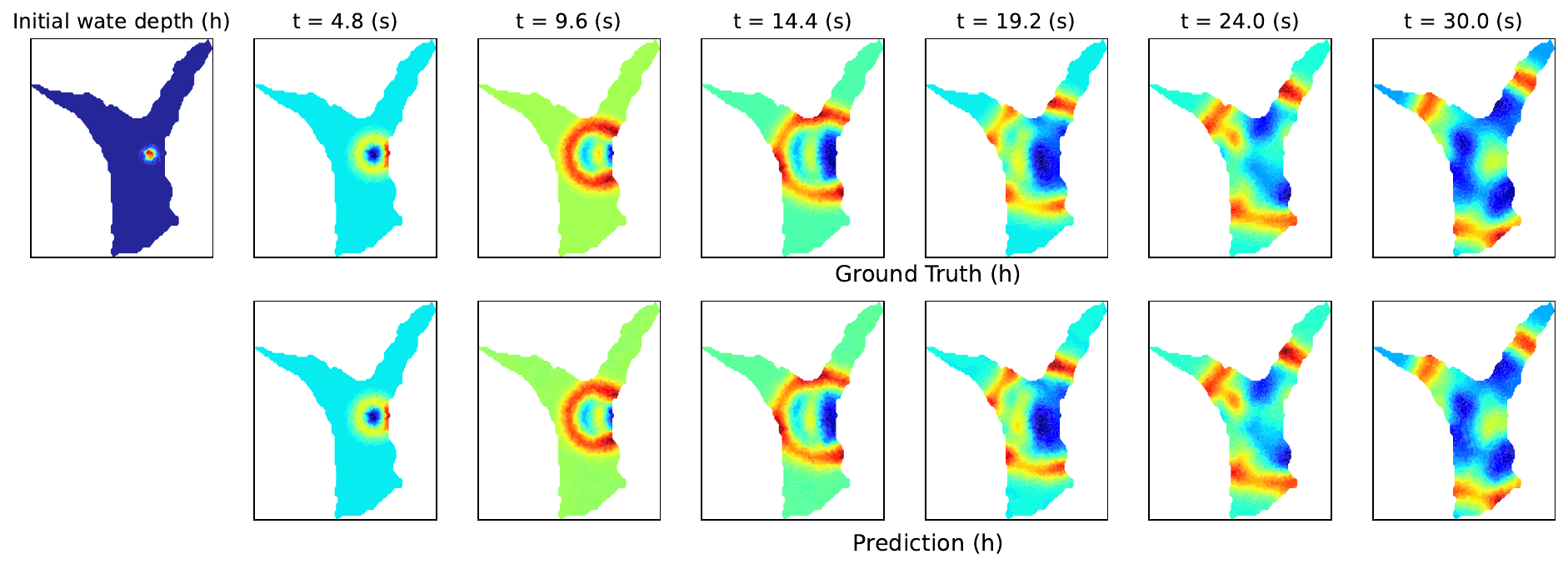}
  \captionsetup{font=footnotesize}
  \caption{Shallow water simulation over a realistic lake-basin geometry, showing the evolution of the water-height field \(h(x,y,t)\) over 30 seconds. GSNO is trained on a coarse mesh with \(N_s=1832\) nodes and evaluated zero-shot on a finer mesh with \(N_s=3663\) nodes without retraining. The top row shows the ground truth water-height fields, while the bottom row shows the corresponding GSNO predictions.}
  \label{fig:swe_zero_shot}
\end{figure}

Figure~\ref{fig:gsno-accuracy-gap_swe} reports the generalization accuracy of GSNO on the 2D Shallow Water Equations benchmark across different mesh resolutions. Figure~\ref{fig:runtime_comparison_swe} further compares the training efficiency of all models under varying temporal input-output settings, highlighting GSNO’s computational advantages.

\begin{table}
\centering
\captionsetup{font=footnotesize}
\caption{Comparison of neural operator models on Shallow Water Equation (\( N_s = 1830 \)).}
\label{tab:swe_metrics}
\scriptsize
\resizebox{0.98\textwidth}{!}{%
\begin{tabular}{l|ccc|ccc|ccc|ccc}
\toprule
\multirow{2}{*}{\textbf{Model}} 
& \multicolumn{3}{c|}{Temporal Config: 1$\rightarrow$50} 
& \multicolumn{3}{c|}{Temporal Config: 3$\rightarrow$48} 
& \multicolumn{3}{c|}{Temporal Config: 5$\rightarrow$46} 
& \multicolumn{3}{c}{Temporal Config: 10$\rightarrow$41} \\
\cmidrule(lr){2-4} \cmidrule(lr){5-7} \cmidrule(lr){8-10} \cmidrule(lr){11-13}
& Rel $L_2$ & RMSE & MAE 
& Rel $L_2$ & RMSE & MAE 
& Rel $L_2$ & RMSE & MAE 
& Rel $L_2$ & RMSE & MAE \\
\midrule
CORAL
& 0.1784 & \(1.78 \times 10^{-1}\) & \(1.52 \times 10^{-1}\)
& 0.1534 & \(1.53 \times 10^{-1}\) & \(1.31 \times 10^{-1}\)
& 0.1238 & \(1.24 \times 10^{-1}\) & \(1.07 \times 10^{-1}\)
& 0.1162 & \(1.16 \times 10^{-1}\) & \(1.00 \times 10^{-1}\) \\
Geo-FNO
& 0.2112 & \(2.11 \times 10^{-1}\) & \(1.79 \times 10^{-1}\)
& 0.1848 & \(1.85 \times 10^{-1}\) & \(1.58 \times 10^{-1}\)
& 0.1650 & \(1.65 \times 10^{-1}\) & \(1.41 \times 10^{-1}\)
& 0.1526 & \(1.53 \times 10^{-1}\) & \(1.31 \times 10^{-1}\) \\
MGKN
& \underline{0.1128} & \(1.13 \times 10^{-1}\) & \(9.79 \times 10^{-2}\)
& \underline{0.0876} & \(8.76 \times 10^{-2}\) & \(7.63 \times 10^{-2}\)
& 0.0724 & \(7.24 \times 10^{-2}\) & \(6.32 \times 10^{-2}\)
& 0.0668 & \(6.68 \times 10^{-2}\) & \(5.84 \times 10^{-2}\) \\
DeepONet
& 0.1284 & \(1.28 \times 10^{-1}\) & \(1.11 \times 10^{-1}\)
& 0.1128 & \(1.13 \times 10^{-1}\) & \(9.79 \times 10^{-2}\)
& 0.0982 & \(9.82 \times 10^{-2}\) & \(8.53 \times 10^{-2}\)
& 0.0904 & \(9.04 \times 10^{-2}\) & \(7.87 \times 10^{-2}\) \\
AMG
& 0.1210 & \(1.21 \times 10^{-1}\) & \(1.04 \times 10^{-1}\)
& 0.0890 & \(8.90 \times 10^{-2}\) & \(7.65 \times 10^{-2}\)
& 0.0692 & \(6.92 \times 10^{-2}\) & \(5.95 \times 10^{-2}\)
& 0.0630 & \(6.30 \times 10^{-2}\) & \(5.42 \times 10^{-2}\) \\
GNOT
& 0.2150 & \(2.15 \times 10^{-1}\) & \(1.85 \times 10^{-1}\)
& 0.1534 & \(1.53 \times 10^{-1}\) & \(1.32 \times 10^{-1}\)
& 0.1093 & \(1.09 \times 10^{-1}\) & \(9.40 \times 10^{-2}\)
& 0.0995 & \(9.95 \times 10^{-2}\) & \(8.56 \times 10^{-2}\) \\
SP$^2$GNO
& 0.1752 & \(1.75 \times 10^{-1}\) & \(1.51 \times 10^{-1}\)
& 0.1250 & \(1.25 \times 10^{-1}\) & \(1.08 \times 10^{-1}\)
& 0.0891 & \(8.91 \times 10^{-2}\) & \(7.66 \times 10^{-2}\)
& 0.0810 & \(8.10 \times 10^{-2}\) & \(6.97 \times 10^{-2}\) \\
Transolver
& 0.1354 & \(1.35 \times 10^{-1}\) & \(1.16 \times 10^{-1}\)
& 0.0965 & \(9.65 \times 10^{-2}\) & \(8.40 \times 10^{-2}\)
& \underline{0.0689} & \(6.89 \times 10^{-2}\) & \(5.93 \times 10^{-2}\)
& \underline{0.0625} & \(6.25 \times 10^{-2}\) & \(5.38 \times 10^{-2}\) \\
\textbf{GSNO}
& \textbf{0.0375} & \(\mathbf{3.75 \times 10^{-2}}\) & \(\mathbf{3.30 \times 10^{-2}}\)
& \textbf{0.0268} & \(\mathbf{2.68 \times 10^{-2}}\) & \(\mathbf{2.36 \times 10^{-2}}\)
& \textbf{0.0193} & \(\mathbf{1.93 \times 10^{-2}}\) & \(\mathbf{1.70 \times 10^{-2}}\)
& \textbf{0.0174} & \(\mathbf{1.74 \times 10^{-2}}\) & \(\mathbf{1.54 \times 10^{-2}}\) \\
\bottomrule
\end{tabular}
}
\end{table}

\begin{figure}
    \centering
    \includegraphics[width=0.95\textwidth]{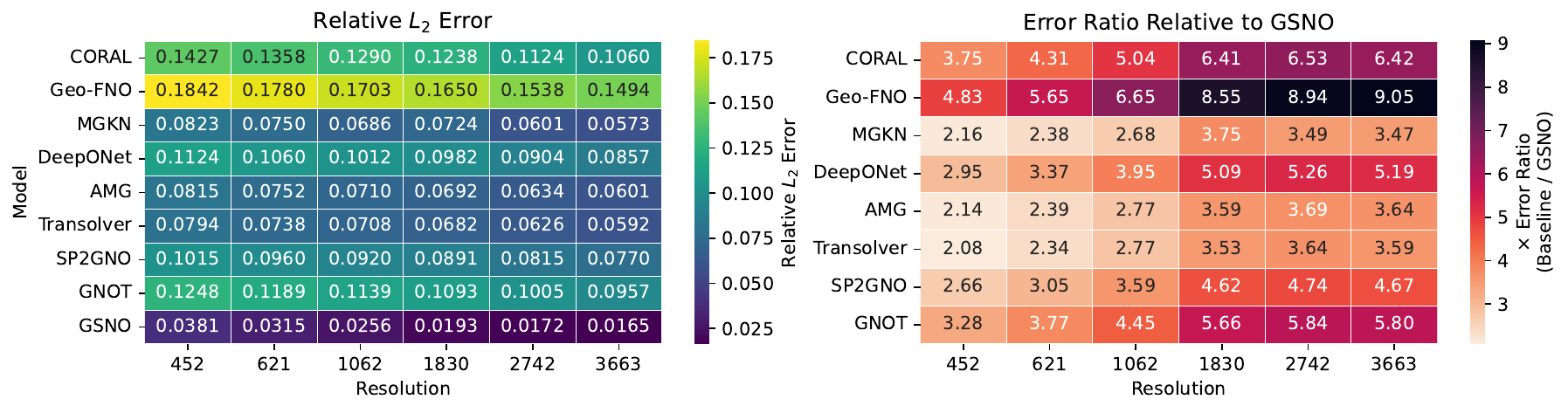}
    \captionsetup{font=footnotesize}
    \caption{
    Resolution-based generalization comparison for the 2D Shallow Water Equation (SWE).  
    \textbf{Left:} Relative $L_2$ error across increasing mesh resolutions for all models.  
    \textbf{Right:} Accuracy gap with respect to GSNO, computed as the ratio of each model’s error to GSNO’s error at each resolution.  
    All models are trained and tested on matching unstructured meshes using \( T_{\text{in}} = 5 \rightarrow T_{\text{out}} = 46 \).  
    GSNO consistently delivers the highest predictive accuracy across spatial scales, with improvements of up to $7\times$ over baseline methods.
    }
    \label{fig:gsno-accuracy-gap_swe}
\end{figure}

\begin{figure}[!t]
    \centering
    \includegraphics[width=0.95\textwidth]{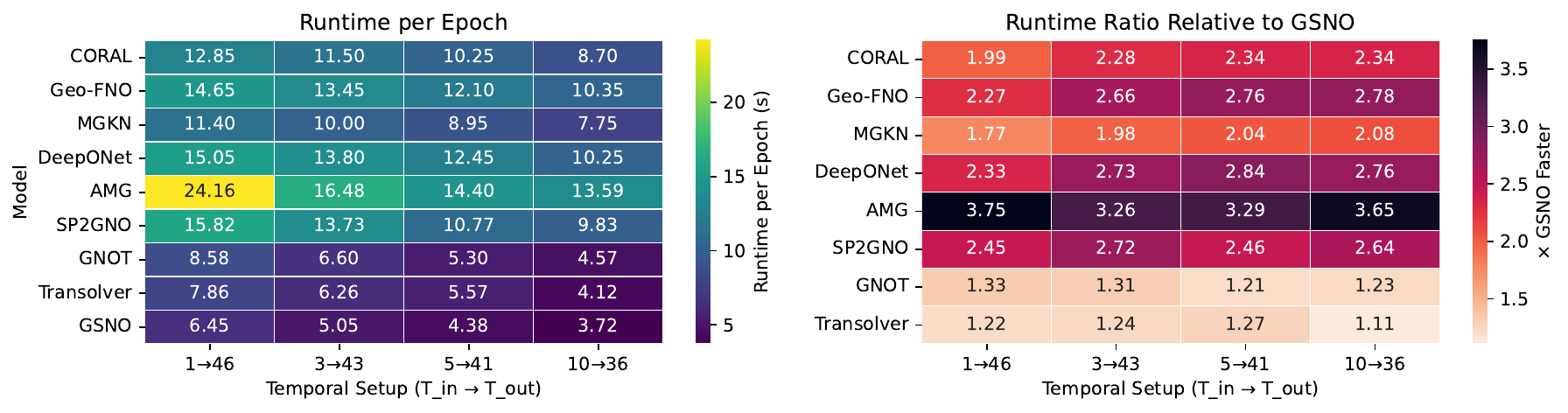}
    \captionsetup{font=footnotesize}
    \caption{
    \textbf{Training efficiency of neural operator models on the 2D Shallow Water Equations at resolution \( N_s = 3663 \).}  
    \textbf{Left:} Average runtime per epoch (in seconds) for various temporal configurations (\( T_{\text{in}} \rightarrow T_{\text{out}} \)).  
    \textbf{Right:} Relative runtime overhead, calculated as the ratio of each model’s epoch runtime to GSNO’s.  
    GSNO achieves the best training efficiency across all settings, offering significant speedups over baseline models in time-dependent simulations.
    }
    \label{fig:runtime_comparison_swe}
\end{figure}

\newpage
\section{GSNO Hyperparameters}
\label{appendix:hyperparams}
\renewcommand{\thefigure}{H.\arabic{figure}}  
\renewcommand{\thetable}{H.\arabic{table}}  
\setcounter{figure}{0}
\setcounter{table}{0}

\subsection{GSNO Architecture, Training Setup, and Sensitivity to the Number of Spatial and Temporal Modes}
\label{app:mode-sensitivity}

Two important hyperparameters in GSNO are the number of spatial modes $k_s$ and the number of temporal modes $k_t$. Together they control the trade-off between accuracy, efficiency, and generalization. Larger values capture finer details but increase computation and can lead to overfitting, while smaller values act as a regularizer but may lose important information. The parameter $k_s$ determines how many graph Laplacian eigenvectors are used to form the spatial spectral basis. We fix $k_t = 8$ and evaluate $k_s \in \{4,6,8,16,32\}$ on the 2D Navier–Stokes case ($N_s=1244$, Temporal Config: $5 \!\rightarrow\! 46$). We report Relative $L^2$ error (lower is better) and CPU time per epoch.

\begin{table}[h]
\centering
\caption{\textbf{Sensitivity of GSNO to spatial modes $k_s$} on 2D Navier--Stokes ($N_s{=}1244$, $k_t{=}8$).}
\label{tab:ks_sensitivity}
\resizebox{0.48\textwidth}{!}{%
\begin{tabular}{lcc}
\toprule
\textbf{Spatial Modes $k_s$} & \textbf{Relative $L^2$ Error} & \textbf{CPU Time / epoch (s)} \\
\midrule
4  & 0.0412 & 2.11 \\
6  & 0.0351 & 3.85 \\
\textbf{8}  & \textbf{0.0164} & \textbf{4.17} \\
16 & 0.0218 & 7.12 \\
32 & 0.0275 & 10.98 \\
\bottomrule
\end{tabular}
}
\end{table}

As shown in Table~\ref{tab:ks_sensitivity}, performance follows a U–shaped curve: very small $k_s$ underfits, very large $k_s$ increases cost and slightly hurts generalization. The sweet spot is $k_s=8$, which minimizes error with moderate runtime. In practice, values between $6$ and $10$ work well. Next, we fix $k_s = 8$ (the optimal setting above) and vary $k_t$ to see its impact. The parameter $k_t$ controls how many Fourier modes are used along the temporal dimension. Larger $k_t$ can improve long-term dynamics but at the expense of slower training.

\begin{table}[h]
\centering
\scriptsize
\caption{\textbf{Sensitivity of GSNO to temporal modes $k_t$} on 2D Navier--Stokes ($N_s{=}1244$, $k_s{=}8$).}
\label{tab:kt_sensitivity}
\resizebox{0.48\textwidth}{!}{%
\begin{tabular}{lcc}
\toprule
\textbf{Temporal Modes $k_t$} & \textbf{Relative $L^2$ Error} & \textbf{CPU Time / epoch (s)} \\
\midrule
4  & 0.0287 & 3.02 \\
6  & 0.0219 & 3.65 \\
\textbf{8}  & \textbf{0.0164} & \textbf{4.17} \\
12 & 0.0182 & 5.46 \\
16 & 0.0235 & 6.88 \\
\bottomrule
\end{tabular}
}
\end{table}

Table~\ref{tab:kt_sensitivity} shows a similar trend: too few temporal modes limit accuracy, while too many slow training and slightly degrade generalization. The best balance occurs at $k_t=8$.

Overall, GSNO achieves the best trade-off when both $k_s$ and $k_t$ are chosen in the mid-range. Too few modes limit expressivity, while too many increase cost and risk overfitting. Based on our experiments, $k_s=8$ and $k_t=8$ provide a strong default setting for 2D Navier–Stokes and related PDE benchmarks. We therefore used this configuration in our main experiments.

Table~\ref{tab:gsno_hyperparameters} summarizes the architecture and training hyperparameters used for GSNO across all benchmark PDEs. The columns "$k_s$" and "$k_t$" represent the number of retained spectral modes in the spatial (graph Laplacian) and temporal (FFT) domains, respectively. The "Width" column denotes the latent feature dimensionality throughout the GSNO blocks. Each block includes a learnable 4D spectral kernel and a residual 1×1 convolution branch, followed by a GELU nonlinearity. Inputs and outputs are min-max normalized per dataset. All GSNO models are trained using the Adam optimizer with a batch size of 32.

\begin{table}[h]
\centering
\captionsetup{font=footnotesize}
\caption{GSNO architecture and training hyperparameters.}
\label{tab:gsno_hyperparameters}
\scriptsize
\resizebox{0.95\textwidth}{!}{%
\begin{tabular}{lccccccccccc}
\toprule
\textbf{PDE Case} & $k_s$ & $k_t$ & Width & \makecell{Lifting \\ MLP} & GSNO Layers & \makecell{Projection \\ MLP} & \makecell{Spatial \\ Branch} & Nonlinearity & \#Params & LR & Epochs \\
\midrule
Darcy Flow        
& 8  & -- & 20 
& 1-layer 
& 4 
& 2-layer: (20 $\times$ 128), (128 $\times$ 1) 
& Conv 1$\times$1 
& GELU 
& 20,817 
& 0.001 
& 1000 \\
 2D Airfoil& 8& -- & 20& 1-layer & 4& 2-layer: (20 $\times$ 128), (128 $\times$ 1) & Conv 1$\times$1 & GELU & 20,817 & 0.001 &1000 \\
 Shape Net 3d Car& 8& -& 40& 1-layer & 4& 2-layer: (20 $\times$ 128), (128 $\times$ 1) & Conv 1$\times$1 & GELU & 42,516& 0.001&1000\\
Burgers’ Equation 
& 8  & 8  & 40 
& 1-layer & 4 
& 2-layer: (20 $\times$ 128), (128 $\times$ 2)& Conv 1$\times$1 & GELU & 431,108& 0.001& 1000 \\
 Navier–Stokes     & 8& 8& 40& 1-layer & 4& 2-layer: (20 $\times$ 128), (128 $\times$ 1) & Conv 1$\times$1 & GELU & 430,857 & 0.001&1000
\\
Shallow water& 8  & 8  & 40 
& 1-layer & 4 
& 2-layer: (20 $\times$ 128), (128 $\times$ 1) & Conv 1$\times$1 
& GELU 
& 430,857 & 0.001 
& 1000 \\
\bottomrule
\end{tabular}
}
\end{table}

\clearpage
\section{Baseline Models Overview, Key Differences, and Hyperparameters}
\label{appendix:baseline_comparison}
\renewcommand{\thefigure}{I.\arabic{figure}}  
\renewcommand{\thetable}{I.\arabic{table}}  
\setcounter{figure}{0}
\setcounter{table}{0}

We compare GSNO against a diverse set of state-of-the-art neural operator models: \textbf{DeepONet}, \textbf{MGKN}, \textbf{CORAL}, \textbf{Geo-FNO}, \textbf{GNOT}, \textbf{Transolver}, and \textbf{AMG}. Each baseline embodies a distinct philosophy for operator learning on irregular or multi-scale domains, ranging from dual-network architectures and kernel-based graph operators to mesh-free latent encodings, Fourier-based domain warping, transformer-driven attention mechanisms, and multi-graph constructions. We briefly summarize their architectures below before presenting a detailed comparison.

\textbf{DeepONet}~\cite{lu2021learning} employs a dual-network structure: a branch network encodes the input function (e.g., coefficients or initial conditions), while a trunk network processes spatial or spatiotemporal coordinates. The outputs of both networks are combined via an inner product to yield the final prediction. This design enables flexible, pointwise evaluation of the solution operator but does not incorporate mesh structure or explicit spectral modeling. As a result, DeepONet's generalization can be sensitive to the distribution and quality of sampled input points, especially on highly irregular domains.

\textbf{MGKN}~\cite{li2020multipole} extends the kernel integral operator framework using learned multipole kernels over graphs built from unstructured meshes. It models spatial interactions by encoding inputs and applying graph convolutions based on Delaunay connectivity. MGKN effectively captures spatial dependencies but does not exploit temporal structure spectrally. Temporal dynamics are typically modeled through standard sequence processing methods, limiting their capability to globally capture long-range temporal dependencies.

\textbf{CORAL}~\cite{serrano2023operator} proposes a mesh-free, coordinate-based neural operator framework. It encodes input data into a latent space using MLPs and reconstructs outputs through coordinate queries. While this design allows CORAL to flexibly generalize across different geometries, it lacks structured spatial priors such as Laplacian smoothness or graph connectivity, which can limit its ability to capture long-range or multi-scale spatial correlations efficiently.

\textbf{Geo-FNO}~\cite{li2023fourier} adapts Fourier Neural Operators to irregular domains by learning a mapping from the physical domain to a latent uniform grid using a transformer encoder. Standard Fourier convolutions are then applied in the latent space. Although Geo-FNO preserves the advantages of global receptive fields inherent to Fourier methods, its performance depends critically on the smoothness and quality of the learned domain warping. In highly complex domains with topological irregularities or sharp features, this warping may introduce distortions, reducing model accuracy and stability.

\textbf{GNOT}~\citep{hao2023gnot} introduces a transformer-based neural operator designed to jointly address three core challenges in operator learning: irregular meshes, multiple input functions, and multi-scale physical dynamics. Its architecture is built around a \emph{heterogeneous normalized attention (HNA)} mechanism, which encodes arbitrary types of inputs (e.g., boundary shapes, parameters, or distributed functions) into a unified representation and applies efficient cross- and self-attention with linear complexity. This enables flexible handling of irregular discretizations and diverse inputs. In addition, GNOT incorporates a \emph{geometric gating mechanism}, inspired by domain decomposition, which adaptively assigns different expert subnetworks to regions of the domain. This soft domain decomposition allows the model to capture multi-scale phenomena more effectively.

\textbf{SP$^2$GNO} \citep{sarkar2025spatio} adopts a hybrid design that couples truncated spectral graph convolutions with a spatial message-passing branch gated by Lipschitz positional embeddings. While this dual-path strategy balances local and global modeling, it introduces architectural complexity and runtime overhead due to dynamic gating and stacked GNN layers. More critically, SP$^2$GNO treats the temporal dimension implicitly through autoregressive rollout of spatial layers, which limits its ability to capture long-range correlations and global frequency structure. In contrast, \textbf{GSNO} follows a simpler and more principled approach: it directly leverages the graph Laplacian eigenbasis for spatial spectral learning and augments it with real Fourier transforms along the temporal dimension, forming a joint space–time spectral kernel without auxiliary gating. This unified treatment eliminates error accumulation from autoregression, avoids over-smoothing, and reduces computation. As a result, GSNO achieves higher efficiency and scalability, with faster runtimes, lower memory footprints, and stronger mesh-invariant generalization, while SP$^2$GNO remains sensitive to graph construction choices. Empirical results confirm that GSNO consistently surpasses SP$^2$GNO in both accuracy and efficiency across steady-state and time-dependent PDE benchmarks.

\textbf{Transolver}~\citep{wu2024transolver} is a transformer-based neural operator specifically designed for solving PDEs across diverse geometries and boundary conditions. Unlike models that rely on fixed grids or handcrafted kernels, Transolver treats operator learning as a sequence-to-sequence problem. It encodes input functions and geometric features through a transformer encoder, then reconstructs solution fields using a decoder equipped with spectral attention. A key design is its ability to incorporate positional and geometric encodings that allow it to directly handle irregular domains without requiring domain warping. By leveraging long-range self-attention, Transolver captures global dependencies in both space and time, which improves its robustness on PDE benchmarks with complex dynamics.

\textbf{AMG}~\citep{li2025harnessing} introduces a \textit{multi-graph neural operator} framework designed to solve PDEs on arbitrary geometries. Its key innovation is the use of three complementary graphs: a \textit{local graph} that captures fine-scale, high-frequency interactions, a \textit{global graph} that encodes broad spatial dependencies, and a \textit{physics graph} that incorporates physical priors into the representation. These graphs are processed through a novel \textbf{GraphFormer} block with dynamic graph attention, which generalizes attention as a learnable integral operator over irregular domains. This design allows AMG to balance local detail and global coherence, while explicitly grounding predictions in physical attributes. Unlike purely spectral or kernel-based models, AMG can adapt to highly complex geometries and dynamically changing meshes.

\begin{table}
\centering
\captionsetup{font=footnotesize}
\caption{Comparison of GSNO with representative neural operator methods for irregular domains. 
\(\checkmark\) indicates explicit support, \(\times\) indicates absence, and \(\sim\) indicates partial or indirect support.}
\label{tab:comparison_prior}
\vspace{0.3em}

\resizebox{\textwidth}{!}{%
\small
\setlength{\tabcolsep}{4pt}
\begin{tabular}{@{}l*{9}{c}@{}}
\toprule
\textbf{Feature}
& \textbf{DeepONet}
& \textbf{MGKN}
& \textbf{CORAL}
& \textbf{Geo-FNO}
& \textbf{GNOT}
& \textbf{Transolver}
& \textbf{AMG}
& \textbf{SP$^2$GNO}
& \textbf{GSNO} \\
\midrule
\multicolumn{10}{@{}l}{\textit{Spectral \& Invariance Properties}} \\
\addlinespace[0.1em]
Space spectral learning
& \(\times\)
& \(\sim\)
& \(\times\)
& \(\sim\)
& \(\times\)
& \(\times\)
& \(\sim\)
& \(\checkmark\)
& \(\checkmark\) \\
Time spectral learning
& \(\times\)
& \(\times\)
& \(\times\)
& \(\times\)
& \(\times\)
& \(\sim\)
& \(\times\)
& \(\times\)
& \(\checkmark\) \\
Mesh-invariant inference
& \(\sim\)
& \(\sim\)
& \(\checkmark\)
& \(\sim\)
& \(\checkmark\)
& \(\checkmark\)
& \(\checkmark\)
& \(\sim\)
& \(\checkmark\) \\
\midrule
\multicolumn{10}{@{}l}{\textit{Architecture Mechanisms}} \\
\addlinespace[0.1em]
Requires GNN stacks
& \(\times\)
& \(\checkmark\)
& \(\times\)
& \(\times\)
& \(\times\)
& \(\times\)
& \(\checkmark\)
& \(\checkmark\)
& \(\times\) \\
Attention mechanism
& \(\times\)
& \(\times\)
& \(\times\)
& \(\times\)
& \(\checkmark\)
& \(\checkmark\)
& \(\checkmark\)
& \(\times\)
& \(\times\) \\
Multi-graph / gating
& \(\times\)
& \(\times\)
& \(\times\)
& \(\times\)
& \(\checkmark\)
& \(\times\)
& \(\checkmark\)
& \(\checkmark\)
& \(\times\) \\
\midrule
\multicolumn{10}{@{}l}{\textit{Global Space--Time Modeling}} \\
\addlinespace[0.1em]
Joint spectral kernel
& \(\times\)
& \(\times\)
& \(\times\)
& \(\times\)
& \(\times\)
& \(\sim\)
& \(\sim\)
& \(\times\)
& \(\checkmark\) \\
\bottomrule
\end{tabular}
}

\vspace{0.4em}

\begin{minipage}{\textwidth}
\footnotesize
\textit{Note.} Partial entries indicate indirect or limited support, such as spectral operations after domain warping, attention-based global mixing, multipole kernels, or graph constructions that depend on learned or approximate connectivity rather than a fixed graph Laplacian spectral basis.
\end{minipage}
\end{table}

\subsection{Key Differences Compared to GSNO.}
\label{appendix:baseline_diff}

Unlike \textbf{Geo-FNO}, GSNO does not rely on learned domain warping to a latent grid. Instead, it operates directly on the physical mesh using Delaunay-based graphs and fixed Laplacian eigenvectors, preserving native geometry without distortion. Compared to \textbf{DeepONet} and \textbf{CORAL}, which lack explicit spectral structure, GSNO projects features into a spatial spectral basis, capturing global correlations across irregular domains. Relative to \textbf{MGKN}, which learns multipole graph kernels but does not address temporal dynamics spectrally, GSNO introduces a real-valued Fourier transform in the temporal dimension, enabling a joint space–time spectral kernel for coherent dynamical modeling. Compared to transformer-based operators, GSNO follows a lighter but more structured approach. Unlike \textbf{GNOT} and \textbf{Transolver}, which rely on heavy multi-head attention, GSNO avoids quadratic attention costs by restricting spectral learning to graph Laplacians and Fourier modes, while still capturing global dependencies. Unlike the \textbf{AMG} multi-graph strategy that aggregates local, global, and physics graphs, GSNO emphasizes a single spectral basis with lightweight $1{\times}1$ convolutional residual paths, reducing computational complexity while retaining generalization. Finally, while the \textbf{SP$^2$GNO}~\citep{sarkar2025spatio} framework combines truncated Laplacian eigenbasis filtering with gated spatial GNN layers, this hybrid design introduces architectural complexity and depends on k-NN graph construction and message passing. More importantly, SP$^2$GNO lacks an explicit temporal spectral module, instead modeling time implicitly through stacked GNN updates, which limits its ability to capture long-range spatiotemporal correlations. In contrast, GSNO integrates both graph-based spatial spectra and Fourier temporal spectra into a unified space–time kernel, achieving superior accuracy and efficiency without relying on recurrent or autoregressive iterations. These design choices allow GSNO to maintain mesh-invariant generalization (via Laplacian recomputation), minimize overhead, and deliver robust accuracy across steady-state and time-dependent PDEs. As our experiments demonstrate, GSNO achieves higher predictive accuracy and efficiency compared to all baselines, while requiring fewer architectural components. The architectural and functional differences between GSNO and the baselines are summarized in Table~\ref{tab:comparison_prior}.

\subsection{Hyperparameters for Baseline Models}
\label{appendix:baseline_hyper}

All baseline models are retrained under identical data splits and training settings as GSNO to ensure a fair and consistent comparison. Detailed hyperparameter configurations for each model are provided in Tables~\ref{tab:hparams_deeponet}–\ref{tab:hparams_geofno}.

\begin{table}[h!]
\centering
\captionsetup{font=footnotesize}
\begin{minipage}[t]{0.48\textwidth}
\centering
\caption{Hyperparameters for \textbf{DeepONet}.}
\label{tab:hparams_deeponet}
\tiny
\begin{tabular}{lc}
\toprule
\textbf{Component} & \textbf{Configuration} \\
\midrule
Trunk Network      & 3-layer MLP, 100 hidden units, ReLU \\
Branch Network     & 2-layer MLP, 100 hidden units, ReLU \\
Input              & Coordinates on a grid \\
Output             & Pointwise function values \\
\bottomrule
\end{tabular}
\end{minipage}%
\hfill
\begin{minipage}[t]{0.48\textwidth}
\centering
\caption{\footnotesize Hyperparameters for \textbf{MGKN}.}
\label{tab:hparams_mgkn}
\tiny
\begin{tabular}{lc}
\toprule
\textbf{Component} & \textbf{Configuration} \\
\midrule
Input Encoder      & 3-layer MLP, 64 units, GELU \\
Decoder            & 2-layer MLP, 64 units, GELU \\
Graph Kernel       & Multipole (Gaussian RBF) \\
RBF Width ($\gamma$) & 1.0 \\
\bottomrule
\end{tabular}
\end{minipage}
\end{table}

\begin{table}[h!]
\centering
\captionsetup{font=footnotesize}
\begin{minipage}[t]{0.48\textwidth}
\centering
\caption{\footnotesize Hyperparameters for \textbf{CORAL}.}
\label{tab:hparams_coral}
\tiny
\begin{tabular}{lc}
\toprule
\textbf{Component} & \textbf{Configuration} \\
\midrule
Encoder              & 4-layer SIREN, width 128, $\omega_0 = 10$ \\
Latent Code          & 128-dimensional vector \\
Decoder              & 3-layer MLP, width 64 \\
Training Strategy    & Meta-learning (outer/inner loops) \\
Input Representation & Coordinate-based (mesh-free) \\
\bottomrule
\end{tabular}
\end{minipage}%
\hfill
\begin{minipage}[t]{0.48\textwidth}
\centering
\captionof{table}{\footnotesize Hyperparameters for \textbf{Geo-FNO}.}
\label{tab:hparams_geofno}
\tiny
\begin{tabular}{lc}
\toprule
\textbf{Component} & \textbf{Configuration} \\
\midrule
Input Encoder        & 3-layer MLP, width 32, sinusoidal encoding \\
Latent Mapping       & Learned warp to regular grid \\
Latent Grid          & Uniform FFT grid (2D for static, 3D for temporal PDEs) \\
Fourier Layers       & 4 layers, 8 retained modes, width 32 \\
Spectral Operation   & Complex-valued FFT on latent grid \\
\bottomrule
\end{tabular}
\end{minipage}
\end{table}

\begin{table}[h!]
\centering
\captionsetup{font=footnotesize}
\label{tab:hparams_sp2gno}
\tiny
\begin{minipage}[t]{0.48\textwidth}
\centering
\caption{\footnotesize SP$^2$GNO (Steady-state)}
\label{tab:hparams_sp2gno_static}
\begin{tabular}{lc}
\toprule
\textbf{Component} & \textbf{Configuration} \\
\midrule
Blocks ($L$)         & 6 \\
Hidden Width ($d$)   & 32 \\
Graph Construction   & k-NN ($k{=}16$) \\
Laplacian Basis      & First $m{=}32$ eigenvectors (LOBPCG) \\
Spectral Kernel      & $K\in\mathbb{R}^{m\times d\times d}$ (learnable) \\
Spatial Branch       & Gated GCN-style conv \\
Positional Encoding  & Lipschitz anchor embeddings \\
\bottomrule
\end{tabular}
\end{minipage}\hfill
\begin{minipage}[t]{0.48\textwidth}
\centering
\caption{\footnotesize SP$^2$GNO (Time-dependent)}
\label{tab:hparams_sp2gno_temporal}
\begin{tabular}{lc}
\toprule
\textbf{Component} & \textbf{Configuration} \\
\midrule
Blocks ($L$)           & 6 \\
Hidden Width ($d$)     & 32 \\
Graph Construction     & k-NN ($k{=}16$), fixed per frame \\
Laplacian Basis        & Reuse first $m{=}32$ eigenvectors \\
Temporal Handling      & Train $1{\to}1$; autoregressive rollout for multi-step \\
Rollout Settings       & Eval: $1{\to}K$ via iterative $1{\to}1$ predictions \\
\bottomrule
\end{tabular}
\end{minipage}
\end{table}

\begin{table}[h!]
\centering
\scriptsize
\captionsetup{font=footnotesize}
\begin{minipage}[t]{0.48\textwidth}
\centering
\caption{Hyperparameters for \textbf{GNOT}.}
\label{tab:hparams_gnot}
\tiny
\begin{tabular}{lc}
\toprule
\textbf{Component} & \textbf{Configuration} \\
\midrule
Attention Layers        & 4 \\
Hidden Size (Attention) & 256 \\
Embedding Dimension     & 256 \\
MLP Depth / Width       & 4 layers, 256 units \\
Attention Heads         & 8 \\
Experts (Geometric Gating) & 3 \\
\bottomrule
\end{tabular}
\end{minipage}%
\hfill
\begin{minipage}[t]{0.48\textwidth}
\centering
\scriptsize
\caption{\footnotesize Hyperparameters for \textbf{Transolver}.}
\label{tab:hparams_transolver}
\tiny
\begin{tabular}{lc}
\toprule
\textbf{Component} & \textbf{Configuration} \\
\midrule
Transformer Layers      & 6 \\
Embedding Dimension     & 256 \\
MLP Depth / Width       & 2 layers, 256 units \\
Attention Heads         & 8 \\
Spectral Attention Modes & 16 \\
Positional Encoding     & Sinusoidal + geometric features \\
\bottomrule
\end{tabular}
\end{minipage}
\end{table}

\begin{table}[h!]
\centering
\captionsetup{font=footnotesize}
\centering
\caption{Hyperparameters for \textbf{AMG}.}
\label{tab:hparams_amg}
\tiny
\begin{tabular}{lc}
\toprule
\textbf{Component} & \textbf{Configuration} \\
\midrule
Graph Types          & Local, Global, Physics \\
Processor Depth      & 3 GraphFormer layers \\
Local Node Number    & 1024 \\
Global Sample Ratio  & 75\% of nodes \\
Physics Nodes        & 32 \\
Attention Heads      & 8 \\
Hidden Size          & 256 \\
\bottomrule
\end{tabular}
\end{table}

\clearpage
\section{Memory footprints}
\label{appendix:foot}
\renewcommand{\thefigure}{J.\arabic{figure}}  
\renewcommand{\thetable}{J.\arabic{table}}  
\setcounter{figure}{0}
\setcounter{table}{0}

\noindent This appendix summarizes the computational footprint of all models. For the NSE and Darcy Flow setups, we report \emph{inference time per batch} and \emph{peak GPU memory} during training and inference at fixed $N_s$ and batch size, as summarized in Tables~\ref{tab:nse_runtime_mem_compact} and \ref{tab:darcy_runtime_mem_compact}.

\begin{table}[h!]
\centering
\captionsetup{font=footnotesize}
\caption{Inference time and memory footprint of models for the Darcy Flow case ($N_s{=}3421$, batch size $32$).}
\label{tab:darcy_runtime_mem_compact}
\tiny
\resizebox{0.75\textwidth}{!}{%
\begin{tabular}{lccc}
\toprule
\textbf{Model} & \textbf{Inference Time (s/batch)} & \textbf{Peak Training GPU Mem} & \textbf{Inference GPU Mem} \\
\midrule
CORAL       & $\sim$0.025 & $\sim$1.5 GB & $\sim$320 MB \\
MGKN        & $\sim$0.023 & $\sim$1.8 GB & $\sim$360 MB \\
Geo-FNO     & $\sim$0.030 & $\sim$2.1 GB & $\sim$370 MB \\
GNOT        & $\sim$0.027 & $\sim$2.0 GB & $\sim$380 MB \\
Transolver  & $\sim$0.025 & $\sim$1.9 GB & $\sim$380 MB \\
AMG         & $\sim$0.104 & $\sim$2.6 GB & $\sim$420 MB \\
SP$^2$GNO   & $\sim$0.023& $\sim$2.1 GB& $\sim$400 MB\\
GSNO        & $\sim$0.022 & $\sim$1.8 GB & $\sim$360 MB \\
\bottomrule
\end{tabular}%
}
\end{table}

\begin{table}[h!]
\centering
\captionsetup{font=footnotesize}
\caption{Inference time and memory footprint of models for the NSE case ($N_s{=}1244$, $5{\to}46$, batch size $32$).}
\label{tab:nse_runtime_mem_compact}
\tiny
\resizebox{0.75\textwidth}{!}{%
\begin{tabular}{lccc}
\toprule
\textbf{Model} & \textbf{Inference Time (s/batch)} & \textbf{Peak Training GPU Mem} & \textbf{Inference GPU Mem} \\
\midrule
CORAL      & $\sim$0.31  & $\sim$5.1 GB & $\sim$1.0 GB \\
MGKN       & $\sim$0.27  & $\sim$7.3 GB & $\sim$1.4 GB \\
Geo-FNO    & $\sim$0.32  & $\sim$6.4 GB & $\sim$1.1 GB \\
GNOT       & $\sim$0.261 & $\sim$6.2 GB & $\sim$1.1 GB \\
Transolver & $\sim$0.239 & $\sim$6.0 GB & $\sim$1.1 GB \\
AMG        & $\sim$0.992 & $\sim$8.5 GB & $\sim$1.2 GB \\
SP$^2$GNO  & $\sim$0.285 & $\sim$6.8 GB & $\sim$1.2 GB \\
GSNO       & $\sim$0.21  & $\sim$5.8 GB & $\sim$1.0 GB \\
\bottomrule
\end{tabular}%
}
\end{table}